\documentclass[10pt,twocolumn,letterpaper]{article}
\usepackage[pagenumbers]{cvpr} 

\usepackage{xcolor}         
\definecolor{mydarkred}{rgb}{0.6,0,0}
\definecolor{mydarkgreen}{rgb}{0,0.6,0}
\usepackage{amsfonts}       
\usepackage{nicefrac}       
\usepackage{microtype}      
\usepackage{amsthm}
\usepackage{bm,bbm}
\usepackage[ruled,vlined]{algorithm2e}

\newtheorem{theorem}{Theorem}

\newtheorem{definition}{Definition}

\usepackage{threeparttable}
\usepackage{upgreek}
\usepackage{colortbl}
\usepackage{threeparttable}
\usepackage{multirow}

\definecolor{cvprblue}{rgb}{0.21,0.49,0.74}
\usepackage[pagebackref,breaklinks,colorlinks,citecolor=cvprblue]{hyperref}

\def\paperID{7257} 
\def\confName{CVPR}
\def\confYear{2024}

\title{Investigating and Mitigating the Side Effects of Noisy Views for Self-Supervised Clustering Algorithms in Practical Multi-View Scenarios}

\author{
Jie Xu$^1$,
Yazhou Ren$^{1,2}$,
Xiaolong Wang$^1$,
Lei Feng$^3$,
Zheng Zhang$^4$,
Gang Niu$^5$,
Xiaofeng Zhu$^{1,2,*}$
\\
{\small $^1$University of Electronic Science and Technology of China, Chengdu, China; $^2$Shenzhen Institute for Advanced Study, University} \\
{\small of Electronic Science and Technology of China, Shenzhen, China; $^3$Singapore University of Technology and Design, Singapore;} \\
{\small $^4$Harbin Institute of Technology, Shenzhen, China; $^5$RIKEN Center for Advanced Intelligence Project, Tokyo, Japan}
}

\begin{document}
\maketitle

\let\thefootnote\relax\footnotetext{$^{*}$Corresponding Author.}

\begin{abstract}
Multi-view clustering (MVC) aims at exploring category structures among multi-view data in self-supervised manners. Multiple views provide more information than single views and thus existing MVC methods can achieve satisfactory performance. However, their performance might seriously degenerate when the views are noisy in practical multi-view scenarios. In this paper, we formally investigate the drawback of noisy views and then propose a theoretically grounded deep MVC method (namely MVCAN) to address this issue. Specifically, we propose a novel MVC objective that enables un-shared parameters and inconsistent clustering predictions across multiple views to reduce the side effects of noisy views. Furthermore, a two-level multi-view iterative optimization is designed to generate robust learning targets for refining individual views' representation learning. Theoretical analysis reveals that MVCAN works by achieving the multi-view consistency, complementarity, and noise robustness. Finally, experiments on extensive public datasets demonstrate that MVCAN outperforms state-of-the-art methods and is robust against the existence of noisy views.
\end{abstract}

\section{Introduction}\label{sec:intro}
Recently, real-world applications generate increasing multi-view data where one sample is described from multiple views, multiple modalities, or multiple groups of features.
To handle such multi-view data, multi-view clustering (MVC) is an effective self-supervised clustering approach and has been applied in many fields (\eg, industry \cite{vazquez2020multigraph}, internet \cite{he2014comment}, and medicine \cite{rappoport2018multi,li2023scbridge}), which can recognize the category structures and patterns without label supervision.
In addition to traditional MVC~\cite{wang2022highly,wen2023highly}, deep learning based MVC is usually built on self-supervised methods like contrastive learning \cite{oord2018representation} and self-training \cite{nigam2000analyzing}, which has been attracting researchers' attention in recent years \cite{zhang2020deep,ren2022deep,tang2022deep,lin2021completer,9839616,dong2023cross} and we conduct a review for related work in Appendix A.

The success of existing MVC methods lies in that they are able to explore the \emph{consistency} and \emph{complementarity} among multi-view data \cite{tang2022deep,liu2018late,zhang2018generalized}, thereby outperforming single-view clustering (SVC) methods \cite{macqueen1967some,ng2002spectral,xie2016unsupervised,peng2022xai}. The consistency indicates that multiple views have the consistent information which is helpful for recognizing the same category \cite{zhan2018multiview,tzortzis2012kernel,cao2015diversity}.
For example, multiple views with the consistent category information can enhance the recognition of the category semantics, thereby eliminating the interference of non-semantic information.
The complementarity means that different views contain the complementary information which is conducive to reciprocally correcting and supplementing each other \cite{tang2022unified,xu2022deep,yang2022robust,liu2023contrastive}.
In other words, the combination of multiple views can help discover category structures that cannot be discovered by individual views.
However, a challenge is that consistency and complementarity of multiple views are still abstract concepts.
To conceptually explore them, previous methods usually leverage different views to supervise each other for learning their common representations, and build consistent clustering predictions for all views' agreement.
For instance, some methods conduct contrastive learning among multiple views for achieving consistency of representations/predictions \cite{jin2023deep,chen2023deep,yan2023gcfagg,lu2023decoupled}.
Some methods integrate multiple views' representations to explore their complementarity and generate a unified cluster partition for optimization as self-training manners \cite{xie2020joint,wen2020dimc,wang2021generative,9839616}.

Despite important advances, experiments reveal that MVC is not necessarily superior to SVC in some practical multi-view scenarios (see Sec.~\ref{SOTA}).
This is because features extracted from some views might be noise, which could be not only useless but even detrimental for clustering.
For example, we consider a situation of observing animals at night, where the view captured by infrared cameras is informative but the view from optical cameras is noisy.
Contrary to informative views, noisy views can play a negative role in recognizing their common category
such that many MVC methods exhibit decreased performance compared to a SVC method that is performed on the optimal single view.
This practical dilemma could affect the effectiveness of MVC and this paper shortly entitles it Noisy-View Drawback (NVD).
We find two reasons that the NVD negatively affects the performance of existing MVC methods in practical scenarios:
I) To obtain fused representations, many methods have to leverage additional neural networks shared by all views \cite{wang2016iterative,zhang2017latent,zhou2019dual,wang2019multi,liu2021one}. However, the clustering objective punished on the noisy view might be dominant that on other informative views, causing the shared parameters in that neural networks to fit the noisy view and thus missing the useful information of other views.
II) For multi-view data, obtaining consistent clustering predictions for all views is a consensus in previous methods \cite{nie2016parameter,wang2018multiview,Zhan8052206,zhou2020end,tang2022deep,xu2022deep}. Nevertheless, it is suboptimal to force the clustering prediction of the noisy view to be the same as that of other views, inversely, this process might make the representation learning and clustering on the informative views degenerate.

In this paper, we consider the NVD and propose a theoretically grounded deep MVC method termed MVCAN: \emph{\underline{M}ulti-\underline{V}iew \underline{C}lustering \underline{A}gainst \underline{N}oisy-view drawback}.
Firstly, based on the aforementioned two reasons, the proposed clustering objective I) requires that the parameters in neural networks are un-shared for individual views and II) optimizes a subproblem that allows inconsistent clustering predictions among different views' soft labels.
Hence, MVCAN designs parameter-decoupled deep models of learning representations and soft labels for different views, aiming to avoid the side effects of noisy views.
Secondly, MVCAN establishes a two-level multi-view iterative optimization for training the parameter-decoupled models.
To be exact, $\mathcal{T}$-level leverages the representations and soft labels to optimize a robust learning target, which makes MVCAN able to explore the useful information among informative views and be robust to noisy views.
$\mathcal{R}$-level automatically matches the learning target with soft labels to optimize the representations of individual views.
Finally, we conduct extensive comparison and ablation experiments to demonstrate the effectiveness of our method.
In summary, the contributions of this work include:
\begin{itemize}
    \item The NVD is pervasive but challenging for MVC, which motivates us to research the robustness towards noisy views.
    To eliminate the two reasons that noisy views hinder clustering effectiveness, we propose a novel clustering objective constrained with two specific conditions.
    \item To effectively train a parameter-decoupled model for each view, we propose a two-level multi-view iterative optimization strategy.
    Extensive experiments on public datasets demonstrate that our MVCAN outperforms state-of-the-art methods and is robust against the existence of noisy views.
    \item In the literature, almost no work theoretically describes the consistency and complementarity of multi-view learning. This paper attempts to theoretically investigate the consistency and complementarity relations among multiple views, and explain the achieved noise robustness.
\end{itemize}

\section{Background and Analysis}\label{Method}

\textbf{Notations}.~We denote $\{\mathbf{X}^v \in \mathbb{R}^{N \times D_v}\}_{v=1}^V$ as a multi-view dataset which contains $N$ samples with $V$ views. $\mathbf{Z}^v \in \mathbb{R}^{N \times d_v}$ and $\mathbf{Y}^v \in \mathbb{R}^{N \times K}$ are the learned representations and soft labels for data in the $v$-th view. $D_v$ and $d_v$ denote the dimensionality of $\mathbf{X}^v$ and $\mathbf{Z}^v$, respectively. $K$ is the cluster number. More notation details are shown in Appendix A.

\subsection{Preliminaries}\label{Analysis}
Deep embedded clustering (DEC \cite{xie2016unsupervised}) is a self-supervised SVC method providing an effective optimization paradigm to promote learning representations and clustering.
Specifically, DEC learns representations $\mathbf{Z}$ from data matrix $\mathbf{X}$ of a single view, and conducts end-to-end clustering by learning soft labels $\mathbf{Y}$ with trainable cluster centroids $\{\bm{\upmu}_j\}_{j=1}^K$ in the representation space of $\mathbf{Z}$.
DEC formulates $\mathbf{Y}$ as follows:
\begin{equation}\label{q}
    y_{ij} = \frac{(1+\lVert \mathbf{z}_i-\bm{\upmu}_j\rVert^2_2)^{-1}}
    {\sum_{j=1}^K(1+\lVert \mathbf{z}_i-\bm{\upmu}_{j}\rVert^2_2)^{-1}} \in \mathbf{Y},
\end{equation}
where $\mathbf{z}_i = \mathcal{E}_{\mathbf{\Phi}}(\mathbf{x}_i) \in \mathbf{Z}$ is the new representation of the $i$-th sample $\mathbf{x}_i \in \mathbf{X}$, obtained by the deep encoder network $\mathcal{E}_{\mathbf{\Phi}}$ with the parameters $\mathbf{\Phi}$.
$(1+\lVert \mathbf{z}_i-\bm{\upmu}_j\rVert^2_2)^{-1}$ can be interpreted as the representation similarity in our Definition \ref{d1}.
We have $\sum_j y_{ij} = 1$ and $y_{ij}$ represents the probabilistic soft label indicating that the sample $\mathbf{x}_i$ comes from the $j$-th cluster. Then, DEC establishes the learning target $\mathbf{T} \in \mathbb{R}^{N \times K}$ to refine $\mathbf{Y}$ and $\mathbf{Z}$ by training the model parameters, where
\begin{equation}\label{p}     
    t_{ij}=\frac{{(y_{ij})^2}/{\sum_{i=1}^N y_{ij}}}
    {\sum_{j=1}^K\left({(y_{ij})^2}/{\sum_{i=1}^N y_{ij}}\right)} \in \mathbf{T}.
\end{equation}
Indeed, Eq.~(\ref{p}) enhances the elements of large values in the soft labels $\mathbf{Y}$ for each sample. As a result, this self-training paradigm establishes the learning target $\mathbf{T}$ to push the soft labels $\mathbf{Y}$ to learn the cluster structures with high confidence.

\subsection{Analysis of Noisy-View Drawback (NVD)}\label{AnalysisNVP}
The aforementioned learning paradigm inspires a lot of developments and is one of the most widely used approaches to conduct deep MVC \cite{wen2020dimc,wang2021generative,xie2020joint,xu2022deep,huang2023self}.
For MVC, previous methods usually learn the representations $\mathbf{Z}^v$ and soft labels $\mathbf{Y}^v$ for individual views, and then leverage the fusion strategies to explore useful information hidden in multiple views, \eg, early fusion \cite{wen2020dimc,wang2021generative} and late fusion \cite{xie2020joint}. They also construct the learning target $\mathbf{T}$ with Eq.~(\ref{p}) to train models.

Although some efforts \cite{ye2018multi,wen2020dimc,trostenMVC,wang2021generative} consider the view diversity and propose weighting strategies in fusion modules,
previous methods usually require shared network parameters and consistent clustering predictions for multiple views, whose models might be not robust when meeting low-quality even noisy views in practical scenarios (will be verified in Sec.~\ref{SOTA}).
To illustrate this, we denote $\{\mathbf{Z}^v\}_{v=1}^V$ as all views' representations and consider an ideal clustering objective:
\begin{equation}\label{c}
\begin{aligned}
\mathop{\min_{\mathbf{\Theta}}}{\sum\nolimits_{v=1}^V \| \mathbf{T} - \mathcal{F}_{\mathbf{\Theta}}(\mathbf{Y}^v|\{\mathbf{Z}^v\}_{v=1}^V) \|_F^2},
\end{aligned}
\end{equation}
where we write $\mathbf{Y}^v = \mathcal{F}_{\mathbf{\Theta}}(\mathbf{Y}^v|\{\mathbf{Z}^v\}_{v=1}^V)$ through the fusion module $\mathcal{F}$, and $\mathbf{\Theta}$ denotes the set of parameters shared by all $V$ views. $\mathbf{T}$ is the unified learning target for training the consistent soft labels $\{\mathbf{Y}^v\}_{v=1}^V$ of all views. With the ground-truth label matrix $\mathbf{L} \in \{0,1\}^{N \times K}$, we further have the following theorem to indicate the relationship between clustering effectiveness and clustering objectives of views:
\begin{theorem}\label{the:theorm1}
Denoting $\mathbf{\check{Y}} = \mathbf{L} \mathbf{A}$, where $\mathbf{A} \in \{0,1\}^{K\times K}$ makes $\mathbf{\check{Y}}$ maximally match the learning target $\mathbf{T}$. Then, the clustering accuracy can be calculated as $ACC = \frac{1}{N} \left(N - \frac{1}{2}\| \mathbf{\check{Y}} - \mathbf{T} \|_F^2\right) = 1 - \frac{1}{2N}\| \mathbf{\check{Y}} - \mathbf{T} \|_F^2$. In Eq.~(\ref{c}), if $\mathbf{\Theta}$ is shared by multiple views and their soft labels $\{\mathbf{Y}^v\}_{v=1}^V$ have consistent learning target $\mathbf{T}$, we have
\begin{equation}\label{acc}
\small
\begin{aligned}
ACC \leq 1 - \frac{1}{2N} \bigg(\mathop{\max_{1\leq m \leq V}} & \| \mathbf{\check{Y}} - \mathcal{F}_{\mathbf{\Theta}}(\mathbf{Y}^m|\{\mathbf{Z}^v\}_{v=1}^V) \|_F^2 \\
- &\| \mathbf{T} - \mathcal{F}_{\mathbf{\Theta}}(\mathbf{Y}^v|\{\mathbf{Z}^v\}_{v=1}^V) \|_F^2 \bigg).
\end{aligned}
\end{equation}
\end{theorem}
\noindent To be specific, we denote $m^* = \arg \max_{1\leq m \leq V}\| \mathbf{\check{Y}} - \mathcal{F}_{\mathbf{\Theta}}(\mathbf{Y}^m|\{\mathbf{Z}^v\}_{v=1}^V) \|_F^2$ , and $\|\mathbf{\check{Y}} - \mathcal{F}_{\mathbf{\Theta}}(\mathbf{Y}^{m^*}|\{\mathbf{Z}^v\}_{v=1}^V) \|_F^2$ could reflect the largest clustering loss $\| \mathbf{T} - \mathcal{F}_{\mathbf{\Theta}}(\mathbf{Y}^{m^*}|\{\mathbf{Z}^v\}_{v=1}^V) \|_F^2$, which corresponds to the view with the worst quality or the most noisy view.
No matter how the set of parameters $\mathbf{\Theta}$ is optimized, for the $m^*$-th view, the unclear cluster structures and inherent noise properties of $\mathbf{Z}^{m^*}$ make it difficult for $\mathbf{Y}^{m^*}$ to fit the learning target $\mathbf{T}$.
Therefore, the noisy view has a large clustering loss that is difficult to minimize, \ie, $\| \mathbf{T} - \mathcal{F}_{\mathbf{\Theta}}(\mathbf{Y}^{m^*}|\{\mathbf{Z}^v\}_{v=1}^V) \|_F^2$, which usually dominates the optimization of other views in Eq.~(\ref{c}). This makes the shared parameters $\mathbf{\Theta}$ tend to fit the noisy view, resulting the model degeneration on other views which have the small clustering losses that are easy to minimize, \eg, $\sum_{v\neq m^*}\| \mathbf{T} - \mathcal{F}_{\mathbf{\Theta}}(\mathbf{Y}^v|\{\mathbf{Z}^v\}_{v=1}^V) \|_F^2$. As a consequence, the noisy view will limit the clustering effectiveness due to the upper bound in Eq.~(\ref{acc}). The detailed proof and example analysis are provided in Appendix B.

\section{Methodology}\label{MMMMMMMM}
To mitigate the side effects of noisy views, we propose Multi-View Clustering Against Noisy-View Drawback (MVCAN), whose frame diagram is given in Appendix A due to space.
\subsection{Clustering Objective Against NVD}
Based on Theorem~\ref{the:theorm1}, we consider two conditions to constrain the multi-view clustering objective for MVCAN. The first condition is that we require the network parameters to be decoupled for all views instead of using shared modules when generating $\{\mathbf{Z}^v, \mathbf{Y}^v\}_{v=1}^V$, and the second condition is that we allow different views to have different clustering predictions instead of consistent ones during training stages.

Accordingly, we modify the optimization objective in Eq.~(\ref{c}) and propose a novel multi-view clustering objective:
\begin{equation}\label{cccc}
\begin{aligned}
&\mathop{\min_{\{\mathbf{\Theta}^v\}_{v=1}^V}}\mathop{\min_{\{\mathbf{A}^v\}_{v=1}^V}}{\sum\nolimits_{v=1}^V \| \mathbf{T} \mathbf{A}^v - \mathcal{F}^v_{\mathbf{\Theta}^v}(\mathbf{Y}^v|\mathbf{Z}^v) \|_F^2}\\
&s.t.~\mathbf{\Theta}^a \cap \mathbf{\Theta}^b = \varnothing, a,b \in \{1,2,\dots,V\}, a \neq b,\\
&~~~~~~~\mathbf{A}^v (\mathbf{A}^v)^T = \mathbf{I}_K, \mathbf{A}^v\in \{0,1\}^{K\times K},\\
\end{aligned}
\end{equation}
where we write $\mathbf{Y}^v = \mathcal{F}^v_{\mathbf{\Theta}^v}(\mathbf{Y}^v|\mathbf{Z}^v)$ whose calculation follows Eq.~(\ref{q}). $\mathbf{Z}^v = \mathcal{E}^v_{\mathbf{\Phi}^v}(\mathbf{X}^v)$ and $\mathcal{E}^v_{\mathbf{\Phi}^v}$ denotes the encoder of individual view. Moreover, we specifically illustrate the two conditions in Eq.~(\ref{cccc}) as follows (their effectiveness will be verified by ablation experiments presented in Sec.~\ref{ABs}):

\textbf{Condition 1}:{
In this framework, the set of parameters $\mathbf{\Theta}^v$ includes $\{\bm{\upmu}_j^v\}_{j=1}^K$ and $\mathbf{\Phi}^v$ of the $v$-th view. We leverage $\mathbf{\Theta}^a \cap \mathbf{\Theta}^b = \varnothing, a,b \in \{1,2,\dots,V\}, a \neq b$ to indicate that $\{\mathbf{\Theta}^v\}_{v=1}^V$ are un-shared for each other, so as to avoid the limitations caused by the NVD as analyzed in Sec.~\ref{AnalysisNVP}.
This condition designs the parameter-decoupled models of learning representations $\{\mathbf{Z}^v\}_{v=1}^V$ and clustering predictions $\{\mathbf{Y}^v\}_{v=1}^V$ for individual views, aiming to eliminate the dominated influence of noisy views on other informative views.
}

\textbf{Condition 2}:{
Before updating the parameters $\{\mathbf{\Theta}^v\}_{v=1}^V$, we solve a subproblem in the multi-view clustering objective, that is, $\mathop{\min_{\{\mathbf{A}^v\}_{v=1}^V}}{\sum_{v=1}^V \| \mathbf{T} \mathbf{A}^v - \mathcal{F}^v_{\mathbf{\Theta}^v}(\mathbf{Y}^v|\mathbf{Z}^v) \|_F^2}$, which leads to $\sum_{v=1}^V \| \mathbf{T} \mathbf{A}^v - \mathcal{F}^v_{\mathbf{\Theta}^v}(\mathbf{Y}^v|\mathbf{Z}^v) \|_F^2 \leq \sum_{v=1}^V \| \mathbf{T} - \mathcal{F}^v_{\mathbf{\Theta}^v}(\mathbf{Y}^v|\mathbf{Z}^v) \|_F^2$. For each view, this subproblem is equivalent to $\mathop{\min_{\mathbf{A}^v}}{\| \mathbf{T} \mathbf{A}^v - \mathbf{Y}^v \|_F^2}$ in which $\mathbf{A}^v\in \{0,1\}^{K\times K}$ achieves the maximum match between the learning target $\mathbf{T}$ and the soft labels $\mathbf{Y}^v$. For each view, $\mathbf{T}$ is adjusted to correspond with $\mathbf{Y}^v$ by $\mathbf{A}^v$ and we can treat this process as to obtain a different learning target $\mathbf{T}^v = \mathbf{T} \mathbf{A}^v$.
This condition makes the clustering loss smaller, as well as considers that it does not make sense to learn consistent clustering predictions for both informative and noisy views.
}

\subsection{Two-Level Multi-View Iterative Optimization}

In brief, to overcome the NVD, MVCAN does not adopt previous strategies that multiple views need shared network parameters and consistent clustering predictions, but how can Eq.~(\ref{cccc}) explore the useful consistent and complementary information from multiple views?
To this end, we propose a two-level multi-view iterative optimization framework for effectively training the parameter-decoupled models.

\textbf{$\mathcal{T}$-level iteration}.~Firstly, we propose a $\mathcal{T}$-level iteration to generate the robust learning target $\mathbf{T}$ for Eq.~(\ref{cccc}).
$\mathcal{T}$-level iteration will not change the parameters $\{\mathbf{\Theta}^v, \mathbf{A}^v\}_{v=1}^V$, making the models still satisfy the parameter decoupling.
For each iteration in the $\mathcal{T}$-level, we design the scaling matrix $\mathbf{W}_{(t)}$ to automatically explore the informative levels of views for obtaining the scaled representation $\mathbf{Z}_{(t)}$, and then produce the robust soft labels $\mathbf{Y}_{(t)}$ for obtaining $\mathbf{T}$ (the theoretical analysis in Sec.~\ref{ccn} will demonstrate that $\mathbf{Y}_{(t)}$ achieve the consistency and complementarity across multiple views, as well as the noise robustness for the noisy views).

Concretely, in the $t$-th iteration of $\mathcal{T}$-level, MVCAN infers the scaled representations $\mathbf{Z}_{(t)} \in \mathbb{R}^{N \times \sum_v{d_v}}$ from all views, by the multiplication between the already scaled/normalized representations $\begin{bmatrix}
\mathbf{Z}^1 & \mathbf{Z}^2 & \dots & \mathbf{Z}^V
\end{bmatrix} \in \mathbb{R}^{N \times \sum_v{d_v}}$
and the scaling matrix $\mathbf{W}_{(t)} \in \mathbb{R}^{\sum_v{d_v} \times \sum_v{d_v}}$:
\begin{equation}\label{z}
\small
\begin{aligned}
\mathbf{Z}_{(t)} & = \mathcal{H}(\mathbf{Z}_{(t)}|\mathbf{W}_{(t)}, \mathbf{Z}^1, \mathbf{Z}^1, \dots, \mathbf{Z}^V) \\
&= 
\begin{bmatrix}
\mathbf{Z}^1 & \mathbf{Z}^2 & \dots & \mathbf{Z}^V
\end{bmatrix} \mathbf{W}_{(t)} \\
&=
\begin{bmatrix}
\mathbf{Z}^1 & \mathbf{Z}^2 & \dots & \mathbf{Z}^V
\end{bmatrix}
\begin{bmatrix}
w^1_{(t)}\mathbf{I}^{1} &                   &        &  \\
                  & w^2_{(t)}\mathbf{I}^{2} &        & \\
                  &                   & \ddots &  \\
                  &                   &        & w^V_{(t)}\mathbf{I}^{V} \\
\end{bmatrix}
,
\end{aligned}
\end{equation}
where $\mathbf{W}_{(t)}$ is a block diagonal matrix ($\mathbf{W}_{(1)} = \mathbf{I}$), of which each block is the multiplication between the unit matrix $\mathbf{I}^{v} \in \{0,1\}^{d_v \times d_v}$ and the scaling factor $w^v_{(t)} \in \mathbb{R}$ for the individual view.
Based on the scaled representations $\mathbf{Z}_{(t)}$, MVCAN generates the robust soft labels $\mathbf{Y}_{(t)} \in \mathbb{R}^{N \times K}$ in the $t$-th iteration.
To be specific, $\mathbf{Y}_{(t)}$ should reflect the cluster structures among $\mathbf{Z}_{(t)}$, and thus we leverage a variant of Eq.~(\ref{q}) to compute $\mathbf{Y}_{(t)}$. Specifically, $\mathbf{z}_{i(t)} \in \mathbf{Z}_{(t)}$ and we formulate $\mathbf{Y}_{(t)} = \mathcal{F}'(\mathbf{Y}_{(t)}|\mathbf{Z}_{(t)})$ as follows:
\begin{equation}\label{qq}
    y_{ij(t)} = \frac{(1+\lVert \mathbf{z}_{i(t)}-\mathbf{c}_{j(t)}\rVert_2^2)^{-1}}
    {\sum_{j=1}^K(1+\lVert \mathbf{z}_{i(t)}-\mathbf{c}_{j(t)}\rVert_2^2)^{-1}} \in \mathbf{Y}_{(t)},
\end{equation}
where $\{\mathbf{c}_{j(t)} \in \mathbb{R}^{\sum_v d_v}\}_{j=1}^K$ represent the cluster centroids of $\mathbf{Z}_{(t)}$ in the $t$-th iteration.
Note that $\{\mathbf{c}_{j(t)}\}_{j=1}^K$ are computed by $K$-means \cite{macqueen1967some} from the scratch in each iteration, it will not change the parameters $\{\mathbf{\Theta}^v, \mathbf{A}^v\}_{v=1}^V$.
Furthermore, denoting $I$ and $H$ as mutual information and entropy, respectively, we base on the normalized mutual information between the robust soft labels $\mathbf{Y}_{(t)}$ and the soft labels $\mathbf{Y}^v$ of individual view, and denote the iterative strategy of the scaling matrix as $\mathbf{W}_{(t+1)} = \mathcal{G}(\mathbf{W}_{(t+1)}|\mathbf{Y}_{(t)}, \mathbf{Y}^1, \mathbf{Y}^1, \dots, \mathbf{Y}^V)$, in which we compute $w^v_{(t+1)}$ for each view by
\begin{equation}\label{www}
\begin{aligned}
    w^v_{(t+1)} = \exp{\left(\frac{2 I(\mathbf{Y}^v;\mathbf{Y}_{(t)})}{H(\mathbf{Y}^v) + H(\mathbf{Y}_{(t)})}\right)} \in \mathbf{W}_{(t+1)}.
\end{aligned}
\end{equation}
To effectively calculate Eq.~(\ref{www}), we first transform $\mathbf{Y}^v$ and $\mathbf{Y}_{(t)}$ into one-dimensional label vectors $\hat{\mathbf{y}}^v$ and $\hat{\mathbf{y}}_{(t)}$, respectively, where $\hat{y}^v_i = \arg\max_{j} y^v_{ij}$, $\hat{y}_{i(t)} = \arg\max_{j} y_{ij(t)}$, and then calculate the normalized mutual information between $\hat{\mathbf{y}}^v$ and $\hat{\mathbf{y}}_{(t)}$.
Since the computations of $\mathbf{Y}_{(t)}$ and $\{\mathbf{Y}^v\}_{v=1}^V$ are all un-/self-supervised, in effect, MVCAN can automatically recognize the informative levels of different views based on the mutual information among the soft labels, and then generate different scaling factors in $\mathbf{W}_{(t+1)}$ to constrain the representations of all views for next iterations.

After finishing the iteration of the robust soft labels $\mathbf{Y}_{(t)}$, we utilize Eq.~(\ref{p}) to obtain the robust learning target, written as $\mathbf{T} = \mathcal{T}(\mathbf{T}|\mathbf{Y}_{(t)})$. Hence, the robust learning target $\mathbf{T}$ is based on the already learned representations and soft labels, \ie, $\{\mathbf{Z}^v, \mathbf{Y}^v\}_{v=1}^V$. The $\mathcal{T}$-level iteration process outputs $\mathbf{T}$ which is further leveraged to refine $\{\mathbf{Z}^v, \mathbf{Y}^v\}_{v=1}^V$ for all views by the multi-view clustering objective in Eq.~(\ref{cccc}).

\textbf{$\mathcal{R}$-level iteration}.~$\mathcal{R}$-level iteration focuses on training the parameters $\{\mathbf{\Theta}^v, \mathbf{A}^v\}_{v=1}^V$ for individual views by optimizing Eq.~(\ref{cccc}).
Considering the Condition 2 of Eq.~(\ref{cccc}), we first obtain ${\mathbf{A}^v}^* = \mathop{\min_{\mathbf{A}^v}}{\| \mathbf{T} \mathbf{A}^v - \mathbf{Y}^v \|_F^2}$ with Hungarian algorithm.
For each view, ${\mathbf{A}^v}^*$ produces a different learning target $\mathbf{T}^v = \mathbf{T} {\mathbf{A}^v}^*$ and then Eq.~(\ref{cccc}) can be transformed into the following clustering objective (denoted by $\mathcal{L}_{c}^v$):
\begin{equation}
\begin{aligned}
\mathcal{L}_{c}^v: \mathop{\min_{\mathbf{\Theta}^v}}{\|\mathbf{T}^v - \mathcal{F}^v_{\mathbf{\Theta}^v}(\mathbf{Y}^v|\mathbf{Z}^v) \|_F^2}.
\end{aligned}
\end{equation}
Additionally, we follow previous deep MVC methods~\cite{xie2020joint,wen2020dimc,wang2021generative,tang2022deepi,9839616} and adopt deep autoencoders (a popular self-supervised representation learning method) to learn the new representations of multi-view data.
Letting $\mathcal{E}_{\mathbf{\Phi}^v}^v$ and $\mathcal{D}^v_{\mathbf{\Psi}^v}$ respectively denote the encoder and decoder, our method requires that the network parameters $\mathbf{\Phi}^v$ and $\mathbf{\Psi}^v$ of each view are un-shared for other views according to the Condition 1 of Eq.~(\ref{cccc}). Therefore, for the $v$-th view, the reconstruction $\Hat{\mathbf{X}}^v = \mathcal{D}^v_{\mathbf{\Psi}^v}(\mathbf{Z}^v)$ is only related to $\mathbf{Z}^v = \mathcal{E}^v_{\mathbf{\Phi}^v}(\mathbf{X}^v)$, and the representation learning objective (denoted by $\mathcal{L}_{r}^v$) is:
\begin{equation}\label{rec}
\begin{aligned}
\mathcal{L}_{r}^v: \mathop{\min_{\{\mathbf{\Psi}^v, \mathbf{\Phi}^v\}}}{\| \mathbf{X}^v - \mathcal{D}^v_{\mathbf{\Psi}^v}(\mathcal{E}_{\mathbf{\Phi}^v}^v(\mathbf{X}^v)) \|_F^2}.
\end{aligned}
\end{equation}
In $\mathcal{R}$-level iteration, the loss function to train the parameter-decoupled model of each view includes following two parts:
\begin{equation}\label{loss}
\begin{aligned}
\mathcal{L}^v = \mathcal{L}_{r}^v + \lambda \mathcal{L}_{c}^v,
\end{aligned}
\end{equation}
where $\lambda$ achieves the trade-off between $\mathcal{L}_{r}^v$ and $\mathcal{L}_{c}^v$.
Meanwhile, we have $\{\mathbf{\Theta}^a,\mathbf{\Psi}^a, \mathbf{\Phi}^a\} \cap \{\mathbf{\Theta}^b,\mathbf{\Psi}^b, \mathbf{\Phi}^b\} = \varnothing, a,b \in \{1,2,\dots,V\}, a \neq b$ which overcomes the mutual interference among different views during training their network parameters.
The $\mathcal{R}$-level iteration process refines the representations and soft labels $\{\mathbf{Z}^v, \mathbf{Y}^v\}_{v=1}^V$ which are further leveraged to obtain better learning target $\mathbf{T}$.
At last, $\mathbf{Y}_{(t)}$ outputs clustering results for all multi-view data and Algorithm~\ref{A1} concludes the training steps of MVCAN (the effectiveness of two losses and iterations will be verified in Sec.~\ref{ABs}).
\begin{algorithm}[!t]
    \caption{Training steps of MVCAN}\label{A1}
	\KwIn{Dataset $\{\mathbf{X}^v\}_{v=1}^V$, Epochs $E$, $T_1$, $T_2$, $K$, $\lambda$}
        Initialize $\{\mathbf{\Phi}^v, \mathbf{\Psi}^v\}_{v=1}^V$ by Eq.~(\ref{rec}) and initialize $\{\{\bm{\upmu}_j^v\}_{j=1}^K\}_{v=1}^{V}$ with $K$-means, $\mathbf{W}_{(1)} = \mathbf{I}$ \\
	\For{$e \in \{1, 2, \dots, E/T_2\}$}{
	    // $\mathcal{T}$\emph{-level infers $\mathbf{T}$ from all views' $\{\mathbf{Z}^v, \mathbf{Y}^v\}_{v=1}^V$.} \\
    	\For{$t \in \{1, 2, \dots, T_1\}$}{
    		Update $\mathbf{Z}_{(t)}$ by Eq.~(\ref{z})\\
    		Update $\mathbf{Y}_{(t)}$ by Eq.~(\ref{qq})\\
    		Update $\mathbf{W}_{(t+1)}$ by Eq.~(\ref{www})
    	}
    	Update $\mathbf{T} = \mathcal{T}(\mathbf{T}|\mathbf{Y}_{(t)})$ as Eq.~(\ref{p})\\
    	// $\mathcal{R}$\emph{-level learns $\{\mathbf{Z}^v, \mathbf{Y}^v\}$ for each view with $\mathbf{T}$.} \\
    	\For{$v \in \{1, 2, \dots, V\}$}{
                Compute $\mathbf{A}^v$ by $\mathop{\min_{\mathbf{A}^v}}{\| \mathbf{T} \mathbf{A}^v - \mathbf{Y}^v \|_F^2}$ in Eq.~(\ref{cccc}) with Hungarian algorithm \\
                Update $\mathbf{\Phi}^v$, $\mathbf{\Psi}^v$, and $\{\bm{\upmu}_j^v\}_{j=1}^K$ on $\mathbf{X}^v$ for $T_2$ epochs by Eq.~(\ref{loss}) with mini-batch Adam \\
	    }
	}
	\textbf{Output:} The cluster assignment of the $i$-th sample $\arg \max_j y_{ij(t)}$ where $y_{ij(t)} \in \mathbf{Y}_{(t)}$, all views' model parameters $\{\mathbf{\Phi}^v, \mathbf{\Psi}^v, \{\bm{\upmu}_j^v\}_{j=1}^K, \mathbf{A}^v\}_{v=1}^V$
\end{algorithm}

\begin{table*}[!ht]
\caption{Clustering performance gains ($\%$) of MVC methods compared with SVC method (DEC-BestV) on four normal multi-view datasets.}\label{table01}
\centering
\renewcommand\tabcolsep{15.0pt}
\resizebox{\textwidth}{!}{
\begin{threeparttable}
    \begin{tabular}{lllllllll}
    \toprule
    \multirow{2}{*}{Method} &\multicolumn{2}{c}{BDGP} &\multicolumn{2}{c}{DIGIT} &\multicolumn{2}{c}{COIL} &\multicolumn{2}{c}{Amazon} \cr
    \cmidrule(r){2-3} \cmidrule(r){4-5} \cmidrule(r){6-7} \cmidrule(r){8-9}
    &~~~~ACC &~~~~NMI &~~~~ACC &~~~~NMI &~~~~ACC &~~~~NMI &~~~~ACC &~~~~NMI \\
    \hline
    DEC-BestV \cite{xie2016unsupervised}     & 92.6      & 81.9      & 80.9      & 78.9      & 76.6      & 81.5      & 47.0      & 32.5      \\
    
    \rowcolor{gray!10}
    DEC-WorstV \cite{xie2016unsupervised}    & 45.7\textcolor{mydarkgreen}{\footnotesize-46.9}      & 29.1\textcolor{mydarkgreen}{\footnotesize-52.8}      & 54.8\textcolor{mydarkgreen}{\footnotesize-26.1}      & 64.1\textcolor{mydarkgreen}{\footnotesize-14.8}     & 73.5\textcolor{mydarkgreen}{\footnotesize-3.1}      & 77.4\textcolor{mydarkgreen}{\footnotesize-4.1}     & 37.2\textcolor{mydarkgreen}{\footnotesize-9.8}      & 27.9\textcolor{mydarkgreen}{\footnotesize-4.6}   \\
    
    DMJC \cite{xie2020joint}                 & 67.8\textcolor{mydarkgreen}{\footnotesize-24.8} & 46.5\textcolor{mydarkgreen}{\footnotesize-35.4} & 97.6\textcolor{mydarkred}{\footnotesize+16.7} & 96.2\textcolor{mydarkred}{\footnotesize+17.3} & 91.3\textcolor{mydarkred}{\footnotesize+14.7} & 93.8\textcolor{mydarkred}{\footnotesize+12.3} & 63.3\textcolor{mydarkred}{\footnotesize+16.3} & 65.3\textcolor{mydarkred}{\footnotesize+32.8} \\

    \rowcolor{gray!10}
    DIMC-net \cite{wen2020dimc}              & 97.5\textcolor{mydarkred}{\footnotesize+4.9}  & 91.1\textcolor{mydarkred}{\footnotesize+9.2}  & 90.4\textcolor{mydarkred}{\footnotesize+9.5}  & 87.3\textcolor{mydarkred}{\footnotesize+8.4}  & 98.5\textcolor{mydarkred}{\footnotesize+21.9} & 97.5\textcolor{mydarkred}{\footnotesize+16.0} & 62.5\textcolor{mydarkred}{\footnotesize+15.5} & 66.9\textcolor{mydarkred}{\footnotesize+34.4} \\

    GP-MVC \cite{wang2021generative}         & 97.6\textcolor{mydarkred}{\footnotesize+5.0}  & 93.4\textcolor{mydarkred}{\footnotesize+11.5} & 58.6\textcolor{mydarkgreen}{\footnotesize-22.3} & 69.8\textcolor{mydarkgreen}{\footnotesize-9.1}  & 86.1\textcolor{mydarkred}{\footnotesize+9.5}  & 77.5\textcolor{mydarkgreen}{\footnotesize-4.0}  & 53.9\textcolor{mydarkred}{\footnotesize+6.9}  & 57.1\textcolor{mydarkred}{\footnotesize+24.6} \\

    \rowcolor{gray!10}
    CoMVC \cite{trostenMVC}         & 80.7\textcolor{mydarkgreen}{\footnotesize-11.9}  & 67.4\textcolor{mydarkgreen}{\footnotesize-14.5} & 98.5\textcolor{mydarkred}{\footnotesize+17.6} & 97.4\textcolor{mydarkred}{\footnotesize+18.5}  & 98.1\textcolor{mydarkred}{\footnotesize+21.5}  & 97.8\textcolor{mydarkred}{\footnotesize+16.3}  & 68.1\textcolor{mydarkred}{\footnotesize+21.1}  & 60.6\textcolor{mydarkred}{\footnotesize+28.1} \\
    
    DIMVC \cite{xu2022deep}                  & 98.1\textcolor{mydarkred}{\footnotesize+5.5}  & 93.8\textcolor{mydarkred}{\footnotesize+11.9} & 97.6\textcolor{mydarkred}{\footnotesize+16.7} & 96.0\textcolor{mydarkred}{\footnotesize+17.1} & 93.4\textcolor{mydarkred}{\footnotesize+16.8} & 93.5\textcolor{mydarkred}{\footnotesize+12.0} & 77.1\textcolor{mydarkred}{\footnotesize+30.1} & 81.3\textcolor{mydarkred}{\footnotesize+48.8} \\

    \rowcolor{gray!10}
    DSMVC \cite{tang2022deep}                & 52.9\textcolor{mydarkgreen}{\footnotesize-39.7} & 38.3\textcolor{mydarkgreen}{\footnotesize-43.6} & 82.0\textcolor{mydarkred}{\footnotesize+1.1}  & 81.4\textcolor{mydarkred}{\footnotesize+2.5}  & 90.8\textcolor{mydarkred}{\footnotesize+14.2} & 96.5\textcolor{mydarkred}{\footnotesize+15.0} & 37.6\textcolor{mydarkgreen}{\footnotesize-9.4}  & 29.2\textcolor{mydarkgreen}{\footnotesize-3.3}  \\

    DSIMVC \cite{tang2022deepi}              & 98.0\textcolor{mydarkred}{\footnotesize+5.4}  & 94.0\textcolor{mydarkred}{\footnotesize+12.1} & 99.0\textcolor{mydarkred}{\footnotesize+18.1} & 97.1\textcolor{mydarkred}{\footnotesize+18.2} & 99.7\textcolor{mydarkred}{\footnotesize+23.1} & 99.0\textcolor{mydarkred}{\footnotesize+17.5} & 64.6\textcolor{mydarkred}{\footnotesize+17.6} & 57.8\textcolor{mydarkred}{\footnotesize+25.3} \\

    \rowcolor{gray!10}
    CPSPAN \cite{jin2023deep}                & 91.5\textcolor{mydarkgreen}{\footnotesize-1.1}  & 77.2\textcolor{mydarkgreen}{\footnotesize-4.7}  & 84.8\textcolor{mydarkred}{\footnotesize+3.9}  & 82.1\textcolor{mydarkred}{\footnotesize+3.2}  & 80.4\textcolor{mydarkred}{\footnotesize+3.8}  & 85.1\textcolor{mydarkred}{\footnotesize+3.6}  & 71.2\textcolor{mydarkred}{\footnotesize+24.2} & 60.8\textcolor{mydarkred}{\footnotesize+28.3} \\

    SDMVC \cite{9839616}                     & 98.5\textcolor{mydarkred}{\footnotesize+5.9}  & 95.0\textcolor{mydarkred}{\footnotesize+13.1} & 99.8\textcolor{mydarkred}{\footnotesize+18.9} & 99.5\textcolor{mydarkred}{\footnotesize+20.6} & 97.0\textcolor{mydarkred}{\footnotesize+20.4} & 95.6\textcolor{mydarkred}{\footnotesize+14.1} & 57.9\textcolor{mydarkred}{\footnotesize+10.9} & 66.5\textcolor{mydarkred}{\footnotesize+34.0} \\
    
    \rowcolor{gray!10}
    MVCAN [ours]                             & 98.4\textcolor{mydarkred}{\footnotesize+5.8}  & 95.3\textcolor{mydarkred}{\footnotesize+13.4} & 99.5\textcolor{mydarkred}{\footnotesize+18.6} & 98.8\textcolor{mydarkred}{\footnotesize+19.9} & 99.6\textcolor{mydarkred}{\footnotesize+23.0} & 99.1\textcolor{mydarkred}{\footnotesize+17.6} & 82.6\textcolor{mydarkred}{\footnotesize+35.6} & 86.7\textcolor{mydarkred}{\footnotesize+54.2} \\

    \bottomrule
    \end{tabular}
\end{threeparttable}
}
\end{table*}

\begin{table*}[!ht]
\caption{Clustering performance gains ($\%$) of MVC methods compared with SVC method (DEC-BestV) on four real-world multi-view datasets. ``n/a'' denotes the unavailable clustering result due to high computational costs.}\label{table03}
\centering
\renewcommand\tabcolsep{15.0pt}
\resizebox{\textwidth}{!}{
\begin{threeparttable}
    \begin{tabular}{lllllllll}
    \toprule
    \multirow{2}{*}{Method} &\multicolumn{2}{c}{DHA} &\multicolumn{2}{c}{RGB-D} &\multicolumn{2}{c}{Caltech} &\multicolumn{2}{c}{YoutubeVideo} \cr
    \cmidrule(r){2-3} \cmidrule(r){4-5} \cmidrule(r){6-7} \cmidrule(r){8-9}
    &~~~~ACC &~~~~NMI &~~~~ACC &~~~~NMI &~~~~ACC &~~~~NMI &~~~~ACC &~~~~NMI \\
    \hline
    DEC-BestV \cite{xie2016unsupervised}  & 72.6 & 79.3 & 43.6 & 40.1 & 88.2 & 81.6 & 20.9 & 20.4 \\
    
    \rowcolor{gray!10}
    DEC-WorstV \cite{xie2016unsupervised} & 30.4\textcolor{mydarkgreen}{\footnotesize-42.2} & 43.5\textcolor{mydarkgreen}{\footnotesize-35.8} & 15.0\textcolor{mydarkgreen}{\footnotesize-28.6} &~~5.1\textcolor{mydarkgreen}{\footnotesize-35.0} & 35.4\textcolor{mydarkgreen}{\footnotesize-52.8} & 19.6\textcolor{mydarkgreen}{\footnotesize-62.0} & 26.6\textcolor{mydarkred}{\footnotesize+5.7} &~~0.0\textcolor{mydarkgreen}{\footnotesize-20.4} \\
    
    DMJC \cite{xie2020joint}              & 64.4\textcolor{mydarkgreen}{\footnotesize-8.2} & 73.9\textcolor{mydarkgreen}{\footnotesize-5.4} & 31.7\textcolor{mydarkgreen}{\footnotesize-11.9} & 28.5\textcolor{mydarkgreen}{\footnotesize-11.6} & 83.1\textcolor{mydarkgreen}{\footnotesize-5.1} & 80.3\textcolor{mydarkgreen}{\footnotesize-1.3} & 15.1\textcolor{mydarkgreen}{\footnotesize-5.8} & 15.3\textcolor{mydarkgreen}{\footnotesize-5.1} \\
    
    \rowcolor{gray!10}
    DIMC-net \cite{wen2020dimc}          & 60.3\textcolor{mydarkgreen}{\footnotesize-12.3} & 73.5\textcolor{mydarkgreen}{\footnotesize-5.8} & 35.6\textcolor{mydarkgreen}{\footnotesize-8.0} & 32.4\textcolor{mydarkgreen}{\footnotesize-7.7} & 75.0\textcolor{mydarkgreen}{\footnotesize-13.2} & 68.5\textcolor{mydarkgreen}{\footnotesize-13.1} &~~n/a &~~n/a \\
    
    GP-MVC \cite{wang2021generative}      & 73.1\textcolor{mydarkred}{\footnotesize+0.5} & 81.5\textcolor{mydarkred}{\footnotesize+2.2} & 38.5\textcolor{mydarkgreen}{\footnotesize-5.1} & 32.6\textcolor{mydarkgreen}{\footnotesize-7.5} & 80.3\textcolor{mydarkgreen}{\footnotesize-7.9} & 77.6\textcolor{mydarkgreen}{\footnotesize-4.0} & 12.4\textcolor{mydarkgreen}{\footnotesize-8.5} & 10.3\textcolor{mydarkgreen}{\footnotesize-10.1} \\
    
    \rowcolor{gray!10}
    CoMVC \cite{trostenMVC}         & 67.4\textcolor{mydarkgreen}{\footnotesize-5.2}  & 79.2\textcolor{mydarkgreen}{\footnotesize-0.1} & 42.0\textcolor{mydarkgreen}{\footnotesize-1.6} & 41.3\textcolor{mydarkred}{\footnotesize+1.2}  & 72.5\textcolor{mydarkgreen}{\footnotesize-15.7}  & 68.8\textcolor{mydarkgreen}{\footnotesize-12.8}  & 18.1\textcolor{mydarkgreen}{\footnotesize-2.8}  & 17.9\textcolor{mydarkgreen}{\footnotesize-2.5} \\
    
    DIMVC \cite{xu2022deep}               & 79.5\textcolor{mydarkred}{\footnotesize+6.9} & 84.7\textcolor{mydarkred}{\footnotesize+5.4} & 46.9\textcolor{mydarkred}{\footnotesize+3.3} & 41.4\textcolor{mydarkred}{\footnotesize+1.3} & 87.2\textcolor{mydarkgreen}{\footnotesize-1.0} & 80.7\textcolor{mydarkgreen}{\footnotesize-0.9} & 15.4\textcolor{mydarkgreen}{\footnotesize-5.5} & 12.5\textcolor{mydarkgreen}{\footnotesize-7.9} \\
    
    \rowcolor{gray!10}
    DSMVC \cite{tang2022deep}            & 77.4\textcolor{mydarkred}{\footnotesize+4.8} & 83.6\textcolor{mydarkred}{\footnotesize+4.3} & 43.3\textcolor{mydarkgreen}{\footnotesize-0.3} & 40.6\textcolor{mydarkred}{\footnotesize+0.5} & 90.5\textcolor{mydarkred}{\footnotesize+2.3} & 84.7\textcolor{mydarkred}{\footnotesize+3.1} & 17.8\textcolor{mydarkgreen}{\footnotesize-3.1} & 18.0\textcolor{mydarkgreen}{\footnotesize-2.4} \\
    
    DSIMVC \cite{tang2022deepi}           & 64.0\textcolor{mydarkgreen}{\footnotesize-8.6} & 77.3\textcolor{mydarkgreen}{\footnotesize-2.0} & 45.8\textcolor{mydarkred}{\footnotesize+2.2} & 41.0\textcolor{mydarkred}{\footnotesize+0.9} & 76.7\textcolor{mydarkgreen}{\footnotesize-11.5} & 67.5\textcolor{mydarkgreen}{\footnotesize-14.1} & 19.0\textcolor{mydarkgreen}{\footnotesize-1.9} & 18.8\textcolor{mydarkgreen}{\footnotesize-1.6} \\
    
    \rowcolor{gray!10}
    CPSPAN \cite{jin2023deep}             & 67.1\textcolor{mydarkgreen}{\footnotesize-5.5} & 80.0\textcolor{mydarkred}{\footnotesize+0.7} & 42.4\textcolor{mydarkgreen}{\footnotesize-1.2} & 38.3\textcolor{mydarkgreen}{\footnotesize-1.8} & 84.8\textcolor{mydarkgreen}{\footnotesize-3.4} & 73.9\textcolor{mydarkgreen}{\footnotesize-7.7} & 23.0\textcolor{mydarkred}{\footnotesize+2.1} & 22.0\textcolor{mydarkred}{\footnotesize+1.6} \\
    
    SDMVC \cite{9839616}                  & 80.2\textcolor{mydarkred}{\footnotesize+7.6} & 85.4\textcolor{mydarkred}{\footnotesize+6.1} & 44.1\textcolor{mydarkred}{\footnotesize+0.5} & 40.7\textcolor{mydarkred}{\footnotesize+0.6} & 85.3\textcolor{mydarkgreen}{\footnotesize-2.9} & 79.1\textcolor{mydarkgreen}{\footnotesize-2.5} & 18.6\textcolor{mydarkgreen}{\footnotesize-2.3} & 18.0\textcolor{mydarkgreen}{\footnotesize-2.4} \\
    
    \rowcolor{gray!10}
    MVCAN [ours]                          & 84.8\textcolor{mydarkred}{\footnotesize+12.2} & 87.5\textcolor{mydarkred}{\footnotesize+8.2} & 48.0\textcolor{mydarkred}{\footnotesize+4.4} & 41.7\textcolor{mydarkred}{\footnotesize+1.6} & 93.6\textcolor{mydarkred}{\footnotesize+5.4} & 88.7\textcolor{mydarkred}{\footnotesize+7.1} & 24.2\textcolor{mydarkred}{\footnotesize+3.3} & 24.3\textcolor{mydarkred}{\footnotesize+3.9} \\
    \bottomrule
    \end{tabular}
\end{threeparttable}
}
\end{table*}

\subsection{Theoretical Analysis of Multi-View Consistency \& Complementarity \& Noise Robustness}\label{ccn}
Moreover, we attempt to theoretically illustrate why MVCAN works with the following definitions and theorems:
\begin{definition}\label{d1}
Denoting $\mathcal{D}(\mathbf{a},\mathbf{b})=\left\| \mathbf{a} - \mathbf{b} \right\|_2^2$ as squared Euclidean distance between the representations $\mathbf{a}$ and $\mathbf{b}$,
\begin{equation}
    \mathcal{S}(\mathbf{a},\mathbf{b}):= \frac{1}{1+\mathcal{D}(\mathbf{a},\mathbf{b})} \in (0, 1]
\end{equation}
is defined as the representation similarity between $\mathbf{a}$ and $\mathbf{b}$.
Formally, $y_{ij}^v \in (0, 1]$ holds given Eq.~(\ref{q}).
\end{definition}
\begin{definition}\label{d2} ($\varepsilon$, $\mathbf{z}$, $\bm{\upmu}$ - Noisy-view)
For $\forall \mathbf{z}_i^v \in \mathbf{Z}^v$, it belongs to the noisy view if $\exists \bm{\upmu}_a^v, \bm{\upmu}_b^v$, and $\varepsilon > 0$ such that $|\mathcal{D}(\mathbf{z}_i^v,\bm{\upmu}_a^v) - \mathcal{D}(\mathbf{z}_i^v,\bm{\upmu}_b^v) |< \varepsilon$, $\mathcal{S}(\mathbf{z}_i^v,\bm{\upmu}_a^v) \approx \mathcal{S}(\mathbf{z}_i^v,\bm{\upmu}_b^v)$, and $y_{ia}^v \approx y_{ib}^v$, where $\varepsilon$ is a sufficiently small value.
Otherwise, $\mathbf{z}_i^v$ is the informative view.
\end{definition}
Then, the following theorems suggest that $\mathbf{Y}_{(t)}$ achieves our concluded consistency, complementarity, and noise robustness with regard to $\{\mathbf{Y}^v\}_{v=1}^V$ in the framework of MVCAN. All proofs of theorems are provided in Appendix B.
\begin{theorem}\label{the:kkk}
Denoting $\mathcal{L}_{K}$ as the $K$-means objective, $\mathcal{L}_{K}(\mathbf{Z}_{(t)})$ is equivalent to 
punishing different scaling factors on $\{\mathcal{L}_{K}(\mathbf{Z}^v)\}_{v=1}^V$ under the consistency constraint of multiple views' cluster centroids.
\end{theorem}
Theorem \ref{the:kkk} analyses the effect of the scaling matrix $\mathbf{W}_{(t)}$ to constrain the optimization of individual views in the scaled representation $\mathbf{Z}_{(t)}$, which reduces the side effects of noisy views when $\mathcal{T}$-level iteration discovers the cluster structures.
\begin{theorem}\label{the:t2}
(\textbf{Consistency})
If a sample representation is informative in multiple views and has the same cluster assignments in these views, its cluster assignment in $\mathbf{Y}_{(t)}$ is the same as that in these views.
\end{theorem}
Theorem \ref{the:t2} indicates that $y_{ij(t)} \in \mathbf{Y}_{(t)}$ follows $\{y_{ij}^v \in \mathbf{Y}^v\}_{v=1}^V$ when they have consistent clusters, which reflects the property of consistency among multiple views.
\begin{theorem}\label{the:t3}
(\textbf{Complementarity})
If a sample representation is informative in multiple views where it has different cluster assignments, we have two cases according to the differences of similarity among the clusters.\\
Case 1: if the differences of similarity among the clusters are equal, its cluster assignment in $\mathbf{Y}_{(t)}$ is the same as that in the informative view with the largest scaling factor.\\
Case 2: if the differences of similarity among the clusters are not equal, its cluster assignment in $\mathbf{Y}_{(t)}$ is more likely to be the same as that in the informative view with the largest scaling factor.
\end{theorem}
Theorem \ref{the:t3} indicates that $y_{ij(t)} \in \mathbf{Y}_{(t)}$ follows $y_{ij}^v \in \mathbf{Y}^v$ with a large scaling factor when different views have inconsistent clusters, which leverages the view with high confidence to correct other inconsistent views.
\begin{theorem}\label{the:t4}
(\textbf{Complementarity $\&$ Noise robustness})\\
Case 1: if a sample representation is informative in some views and is noisy in other views, its cluster assignment in $\mathbf{Y}_{(t)}$ is the same as that in the informative views.\\
Case 2: if a sample representation is noisy in all views, its cluster assignment in $\mathbf{Y}_{(t)}$ is the same as the common cluster assignments existing in these views.
\end{theorem}
Theorem \ref{the:t4} illustrates the noise robustness of our method that makes the robust soft labels $\mathbf{Y}_{(t)}$ mitigate the side effects of noisy views.
For example, $\mathbf{z}^v_{i} \in \mathbf{Z}^v$ is noisy in individual view but the corresponding scaled representation $\mathbf{z}_{i(t)} \in \mathbf{Z}_{(t)}$ is informative, so the influence from noisy views on the soft labels $y_{ij(t)} \in \mathbf{Y}_{(t)}$ are reduced.
Additionally, Theorems \ref{the:t3} and \ref{the:t4} can be together interpreted as the complementarity among multiple views, \ie, the combination of multiple views is conducive to outperforming single views and discovering comprehensive cluster patterns (which cannot be explored in single-view data) across multi-view data.

\section{Experiments}\label{et}
\subsection{Settings}
We briefly introduce the experimental setup and show more implementation details in Appendix. Our code is provided in \url{https://github.com/SubmissionsIn/MVCAN}.

\textbf{Datasets}.~We conduct experiments on eight public datasets and four noise-simulated ones, and their details are listed in Appendix C.
First, four normal multi-view datasets (easy for clustering) include BDGP~\cite{cai2012joint}, DIGIT~\cite{peng2019comic}, COIL~\cite{nene1996columbia}, and Amazon~\cite{saenko2010adapting}.
Second, we construct four noise-simulated datasets on the four datasets to test the noise robustness of methods in extreme scenarios, where we randomly sample noise to build an additional view and obtain NoisyBDGP/DIGIT/COIL/Amazon for the four individual datasets.
Third, we conduct experiments on four real-world multi-view datasets (hard for clustering) including DHA \cite{lin2012human}, RGB-D \cite{zhou2020end}, Caltech \cite{fei2004learning}, and YoutubeVideo \cite{madani2012using}.

\textbf{Comparison methods}.~We compare our MVCAN with the following 10 self-supervised clustering algorithms.
To be specific, DEC~\cite{xie2016unsupervised} is a popular deep SVC method and we leverage this baseline to investigate the side effects of NVD on MVC methods.
DMJC \cite{xie2020joint}, DIMC-net \cite{wen2020dimc}, GP-MVC~\cite{wang2021generative}, DIMVC \cite{xu2022deep}, and SDMVC \cite{9839616} are DEC-based deep MVC methods which usually establish consistent soft labels for achieving clustering consistency.
DMJC \cite{xie2020joint}, DIMC-net \cite{wen2020dimc}, GP-MVC \cite{wang2021generative}, and DSMVC \cite{tang2022deep} mainly incorporate weighting strategies to obtain fused representations.
CoMVC~\cite{trostenMVC}, DSIMVC \cite{tang2022deepi}, and CPSPAN \cite{jin2023deep} are contrastive learning based deep MVC methods which leverage contrastive learning to learn common representations.
\begin{table*}[!ht]
\caption{Clustering performance gains ($\%$) of MVC methods compared with SVC method (DEC-BestV) on four noise-simulated datasets.}\label{table02}
\centering
\renewcommand\tabcolsep{15.0pt}
\resizebox{\textwidth}{!}{
\begin{threeparttable}
    \begin{tabular}{lllllllll}
    \toprule
    \multirow{2}{*}{Method} &\multicolumn{2}{c}{NoisyBDGP} &\multicolumn{2}{c}{NoisyDIGIT} &\multicolumn{2}{c}{NoisyCOIL} &\multicolumn{2}{c}{NoisyAmazon} \cr
    \cmidrule(r){2-3} \cmidrule(r){4-5} \cmidrule(r){6-7} \cmidrule(r){8-9}
    &~~~~ACC &~~~~NMI &~~~~ACC &~~~~NMI &~~~~ACC &~~~~NMI &~~~~ACC &~~~~NMI \\
    \hline
    DEC-BestV \cite{xie2016unsupervised}     & 92.6      & 81.9      & 80.9      & 78.9      & 76.6      & 81.5      & 47.0      & 32.5      \\
    
    \rowcolor{gray!10}
    DEC-WorstV \cite{xie2016unsupervised}    & 22.2\textcolor{mydarkgreen}{\footnotesize-70.4}      &~~0.2\textcolor{mydarkgreen}{\footnotesize-81.7}      & 12.4\textcolor{mydarkgreen}{\footnotesize-68.5}      &~~0.4\textcolor{mydarkgreen}{\footnotesize-78.5}      & 16.4\textcolor{mydarkgreen}{\footnotesize-60.2}      &~~2.8\textcolor{mydarkgreen}{\footnotesize-78.7}   & 12.0\textcolor{mydarkgreen}{\footnotesize-35.0}   &~~0.4\textcolor{mydarkgreen}{\footnotesize-32.1}  \\

    DMJC \cite{xie2020joint}         & 63.7\textcolor{mydarkgreen}{\footnotesize-28.9} & 59.4\textcolor{mydarkgreen}{\footnotesize-22.5} & 80.7\textcolor{mydarkgreen}{\footnotesize-0.2}  & 82.8\textcolor{mydarkred}{\footnotesize+3.9}  & 85.6\textcolor{mydarkred}{\footnotesize+9.0}  & 92.1\textcolor{mydarkred}{\footnotesize+10.6} & 54.3\textcolor{mydarkred}{\footnotesize+7.3}  & 46.6\textcolor{mydarkred}{\footnotesize+14.1} \\

    \rowcolor{gray!10}
    DIMC-net \cite{wen2020dimc}      & 78.9\textcolor{mydarkgreen}{\footnotesize-13.7} & 68.4\textcolor{mydarkgreen}{\footnotesize-13.5} & 71.6\textcolor{mydarkgreen}{\footnotesize-9.3}  & 76.5\textcolor{mydarkgreen}{\footnotesize-2.4}  & 87.5\textcolor{mydarkred}{\footnotesize+10.9} & 91.8\textcolor{mydarkred}{\footnotesize+10.3} & 43.6\textcolor{mydarkgreen}{\footnotesize-3.4}  & 37.3\textcolor{mydarkred}{\footnotesize+4.8}  \\

    GP-MVC \cite{wang2021generative} & 80.7\textcolor{mydarkgreen}{\footnotesize-11.9} & 78.4\textcolor{mydarkgreen}{\footnotesize-3.5}  & 49.1\textcolor{mydarkgreen}{\footnotesize-31.8} & 63.5\textcolor{mydarkgreen}{\footnotesize-15.4} & 69.4\textcolor{mydarkgreen}{\footnotesize-7.2}  & 72.9\textcolor{mydarkgreen}{\footnotesize-8.6}  & 40.4\textcolor{mydarkgreen}{\footnotesize-6.6}  & 39.8\textcolor{mydarkred}{\footnotesize+7.3}  \\

    \rowcolor{gray!10}
    CoMVC \cite{trostenMVC}         & 63.8\textcolor{mydarkgreen}{\footnotesize-28.8}  & 51.6\textcolor{mydarkgreen}{\footnotesize-30.3} & 86.9\textcolor{mydarkred}{\footnotesize+6.0} & 84.6\textcolor{mydarkred}{\footnotesize+5.7}  & 90.6\textcolor{mydarkred}{\footnotesize+14.0}  & 93.6\textcolor{mydarkred}{\footnotesize+12.1}  & 61.8\textcolor{mydarkred}{\footnotesize+14.8}  & 52.6\textcolor{mydarkred}{\footnotesize+20.1} \\
    
    DIMVC \cite{xu2022deep}          & 94.9\textcolor{mydarkred}{\footnotesize+2.3}  & 87.6\textcolor{mydarkred}{\footnotesize+5.7}  & 88.7\textcolor{mydarkred}{\footnotesize+7.8}  & 93.7\textcolor{mydarkred}{\footnotesize+14.8} & 89.0\textcolor{mydarkred}{\footnotesize+12.4} & 91.7\textcolor{mydarkred}{\footnotesize+10.2} & 63.6\textcolor{mydarkred}{\footnotesize+16.6} & 66.7\textcolor{mydarkred}{\footnotesize+34.2} \\

    \rowcolor{gray!10}
    DSMVC \cite{tang2022deep}        & 57.1\textcolor{mydarkgreen}{\footnotesize-35.5} & 41.8\textcolor{mydarkgreen}{\footnotesize-40.1} & 73.7\textcolor{mydarkgreen}{\footnotesize-7.2}  & 72.2\textcolor{mydarkgreen}{\footnotesize-6.7}  & 81.8\textcolor{mydarkred}{\footnotesize+5.2}  & 84.1\textcolor{mydarkred}{\footnotesize+2.6}  & 36.6\textcolor{mydarkgreen}{\footnotesize-10.4} & 25.9\textcolor{mydarkgreen}{\footnotesize-6.6}  \\
    
    DSIMVC \cite{tang2022deepi}      & 95.1\textcolor{mydarkred}{\footnotesize+2.5}  & 85.2\textcolor{mydarkred}{\footnotesize+3.3}  & 90.4\textcolor{mydarkred}{\footnotesize+9.5}  & 90.5\textcolor{mydarkred}{\footnotesize+11.6} & 98.8\textcolor{mydarkred}{\footnotesize+22.2} & 97.8\textcolor{mydarkred}{\footnotesize+16.3} & 54.7\textcolor{mydarkred}{\footnotesize+7.7}  & 54.1\textcolor{mydarkred}{\footnotesize+21.6} \\

    \rowcolor{gray!10}
    CPSPAN \cite{jin2023deep}        & 73.2\textcolor{mydarkgreen}{\footnotesize-19.4} & 53.5\textcolor{mydarkgreen}{\footnotesize-28.4} & 11.8\textcolor{mydarkgreen}{\footnotesize-69.1} & ~~0.3\textcolor{mydarkgreen}{\footnotesize-78.6}  & 15.8\textcolor{mydarkgreen}{\footnotesize-60.8} & ~~3.3\textcolor{mydarkgreen}{\footnotesize-78.2}  & 12.4\textcolor{mydarkgreen}{\footnotesize-34.6} & ~~0.4\textcolor{mydarkgreen}{\footnotesize-32.1}  \\
    
    SDMVC \cite{9839616}             & 89.6\textcolor{mydarkgreen}{\footnotesize-3.0}  & 83.6\textcolor{mydarkred}{\footnotesize+1.7}  & 75.8\textcolor{mydarkgreen}{\footnotesize-5.1}  & 72.2\textcolor{mydarkgreen}{\footnotesize-6.7}  & 81.0\textcolor{mydarkred}{\footnotesize+4.4}  & 89.2\textcolor{mydarkred}{\footnotesize+7.7}  & 55.4\textcolor{mydarkred}{\footnotesize+8.4}  & 61.0\textcolor{mydarkred}{\footnotesize+28.5} \\

    \rowcolor{gray!10}
    MVCAN [ours]                     & 98.0\textcolor{mydarkred}{\footnotesize+5.4}  & 95.1\textcolor{mydarkred}{\footnotesize+13.2} & 99.0\textcolor{mydarkred}{\footnotesize+18.1} & 98.4\textcolor{mydarkred}{\footnotesize+19.5} & 99.2\textcolor{mydarkred}{\footnotesize+22.6} & 98.8\textcolor{mydarkred}{\footnotesize+17.3} & 72.8\textcolor{mydarkred}{\footnotesize+25.8} & 73.2\textcolor{mydarkred}{\footnotesize+40.7} \\
    \bottomrule
    \end{tabular}
\end{threeparttable}
}
\end{table*}

\subsection{Comparison Results and Analysis}\label{SOTA}
Tables~\ref{table01}, \ref{table03}, and \ref{table02} list clustering effectiveness of comparison methods on all datasets. The performance is evaluated by clustering accuracy (ACC) and normalized mutual information (NMI), and the average values of 10 runs are reported.
DEC-BestV and DEC-WorstV denote the results of the SVC method DEC on the best and the worst views, respectively.

Firstly, we compare DEC-BestV with DEC-WorstV and can easily find that the clustering results of DEC-WorstV is not ideal for many samples, that is, many samples that are correctly clustered by DEC-BestV are incorrectly clustered by DEC-WorstV (especially for real-world multi-view datasets in Table~\ref{table03}).
This suggests that the view qualities of multi-view datasets are different, where the views with unclear cluster structures could be considered as noisy views for clustering.
Secondly, most of MVC methods achieve performance gains on normal datasets (\textcolor{mydarkred}{red} results in Table~\ref{table01}) but have performance degeneration on real-world datasets (\textcolor{mydarkgreen}{green} results in Table~\ref{table03}) when taking DEC-BestV as the baseline. The side effects of noisy views adversely affect many MVC methods and thus we observe that some multi-view methods are not robuster than the single-view method in terms of clustering effectiveness.
Despite some of these MVC methods leverage weighting strategies to balance different views, the noisy-view drawback still prevent them from learning effective cluster structures in some practical scenarios.
Thirdly, our method MVCAN obtains much better performance than DEC-BestV across all datasets and generally achieves the best or comparable performance among all MVC methods.
For example in Table~\ref{table03}, MVCAN improves the best comparison methods by 4\%, 1\%, and 3\% ACC values on DHA, RGB-D, and Caltech, respectively.
The results indicate that MVCAN is able to explore the useful consistent and complementary information among informative views, as well as achieve the noise robustness to noisy views.

Since the performance of MVC could be interfered with noisy views in datasets, it is encouraged to test the robustness of algorithms with extreme noise interference \cite{ye2018multi,trostenMVC}, which can guide the algorithm design of MVC for practical scenarios.
To this end, we conduct comparison experiments on noise-simulated multi-view datasets as shown in Table~\ref{table02}.
Compared with Table~\ref{table01}, Table~\ref{table02} suggests that most of MVC methods have degenerated results but our MVCAN still achieves comparable performance.
Specifically, MVCAN surpasses the best comparison methods by 7\%, 5\%, 1\%, and 6\% NMI values on the four noise-simulated datasets.
This further demonstrates the effectiveness of our method.

\subsection{Ablation Study}\label{ABs}
In this subsection, we investigate the effectiveness of each part of our method in detail from the following aspects.
\begin{table}[!ht]\caption{Importance of two conditions in clustering objective.}\label{tab:table06}
\small
\centering
\resizebox{\linewidth}{!}{
\begin{threeparttable}
    \begin{tabular}{l|cc|cc|cc|cc|cc}
    \toprule
    &\multicolumn{2}{c|}{Conditions} &\multicolumn{2}{c|}{BDGP} &\multicolumn{2}{c|}{DIGIT} &\multicolumn{2}{c|}{NoisyBDGP} &\multicolumn{2}{c}{NoisyDIGIT} \cr
    \hline
    &$\mathbf{\Theta}^v$ &$\mathbf{A}^v$  & ACC & NMI & ACC & NMI & ACC & NMI & ACC & NMI \\
    \hline
    (i)&$\checkmark$&\quad                      &97.3 &92.1 &98.6 &98.0 &97.7 &94.6 &90.2 &93.3 \\
    (ii)&\quad&$\checkmark$                     &66.0 &47.7 &84.0 &73.5 &60.2 &33.9 &61.1 &59.9 \\
    (iii)&$\checkmark$&$\checkmark$             &98.4 &95.3 &99.5 &98.8 &98.0 &95.1 &99.0 &98.4 \\
    \bottomrule
    \end{tabular}
\end{threeparttable}
}
\end{table}
\begin{table}[!ht]\caption{Importance of two loss components in optimization.}\label{tab:table04}
\small
\centering
\resizebox{\linewidth}{!}{
\begin{threeparttable}
    \begin{tabular}{l|cc|cc|cc|cc|cc}
    \toprule
    &\multicolumn{2}{c|}{Components} &\multicolumn{2}{c|}{BDGP} &\multicolumn{2}{c|}{DIGIT} &\multicolumn{2}{c|}{NoisyBDGP} &\multicolumn{2}{c}{NoisyDIGIT} \cr
    \hline
    &$\mathcal{L}_{r}^v$ &$\mathcal{L}_{c}^v$  & ACC & NMI & ACC & NMI & ACC & NMI & ACC & NMI \\
    \hline
    (a)&\quad&\quad                      &64.3 &52.2 &76.8 &72.3 &49.9 &31.5 &74.1 &70.5 \\
    (b)&$\checkmark$&\quad               &94.8 &84.2 &78.7 &74.7 &94.4 &83.9 &76.9 &74.8 \\
    (c)&\quad&$\checkmark$               &79.8 &70.9 &87.0 &94.3 &72.6 &57.4 &59.1 &70.2 \\
    (d)&$\checkmark$&$\checkmark$        &98.4 &95.3 &99.5 &98.8 &98.0 &95.1 &99.0 &98.4 \\
    \bottomrule
    \end{tabular}
\end{threeparttable}
}
\end{table}
\begin{figure}[!ht]
\centering
  \begin{subfigure}{0.49\linewidth}
    \includegraphics[width=\linewidth]{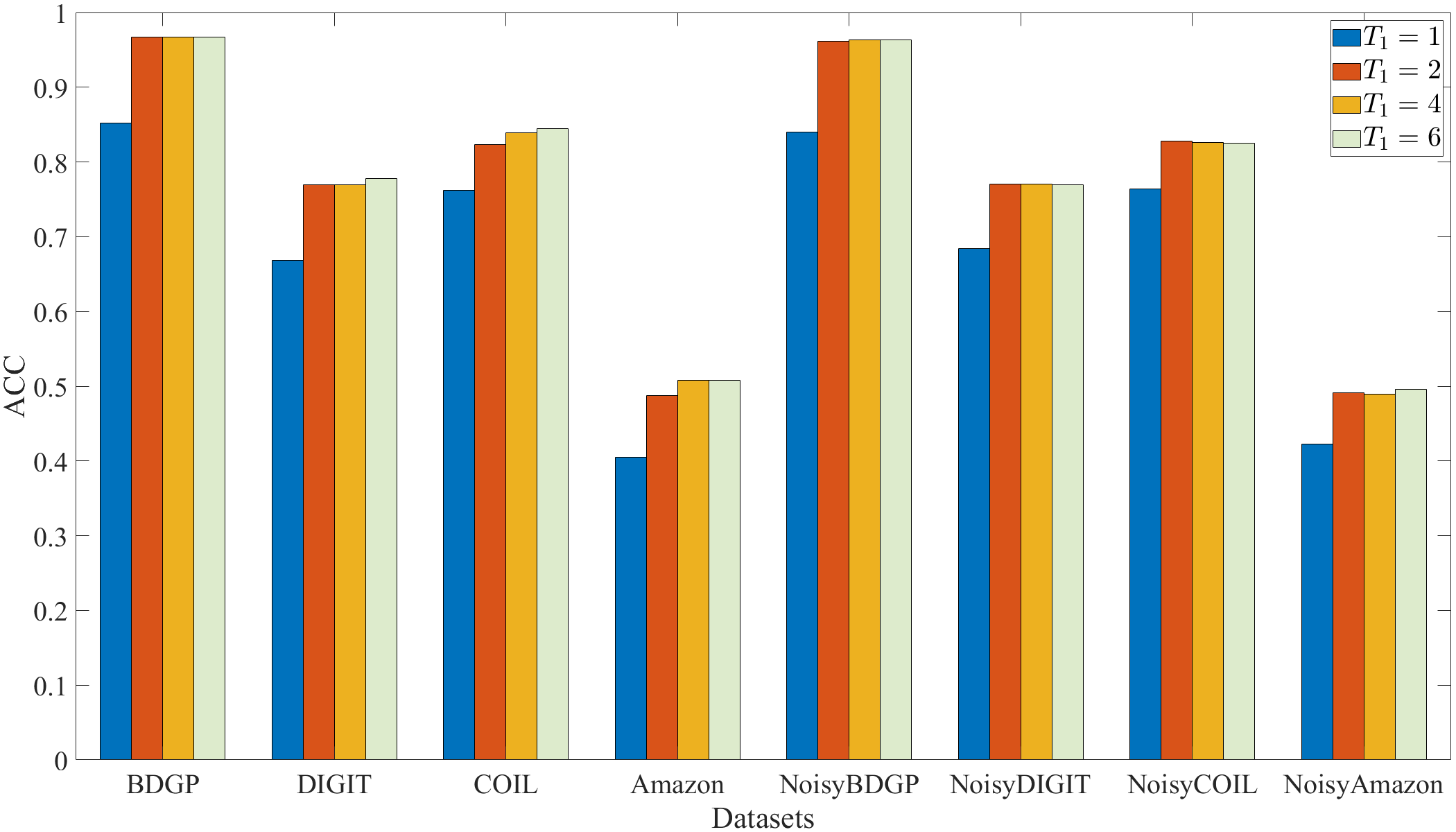}
    \caption{ACC $vs.$ $T_1$}
  \end{subfigure}
  \begin{subfigure}{0.49\linewidth}
    \includegraphics[width=\linewidth]{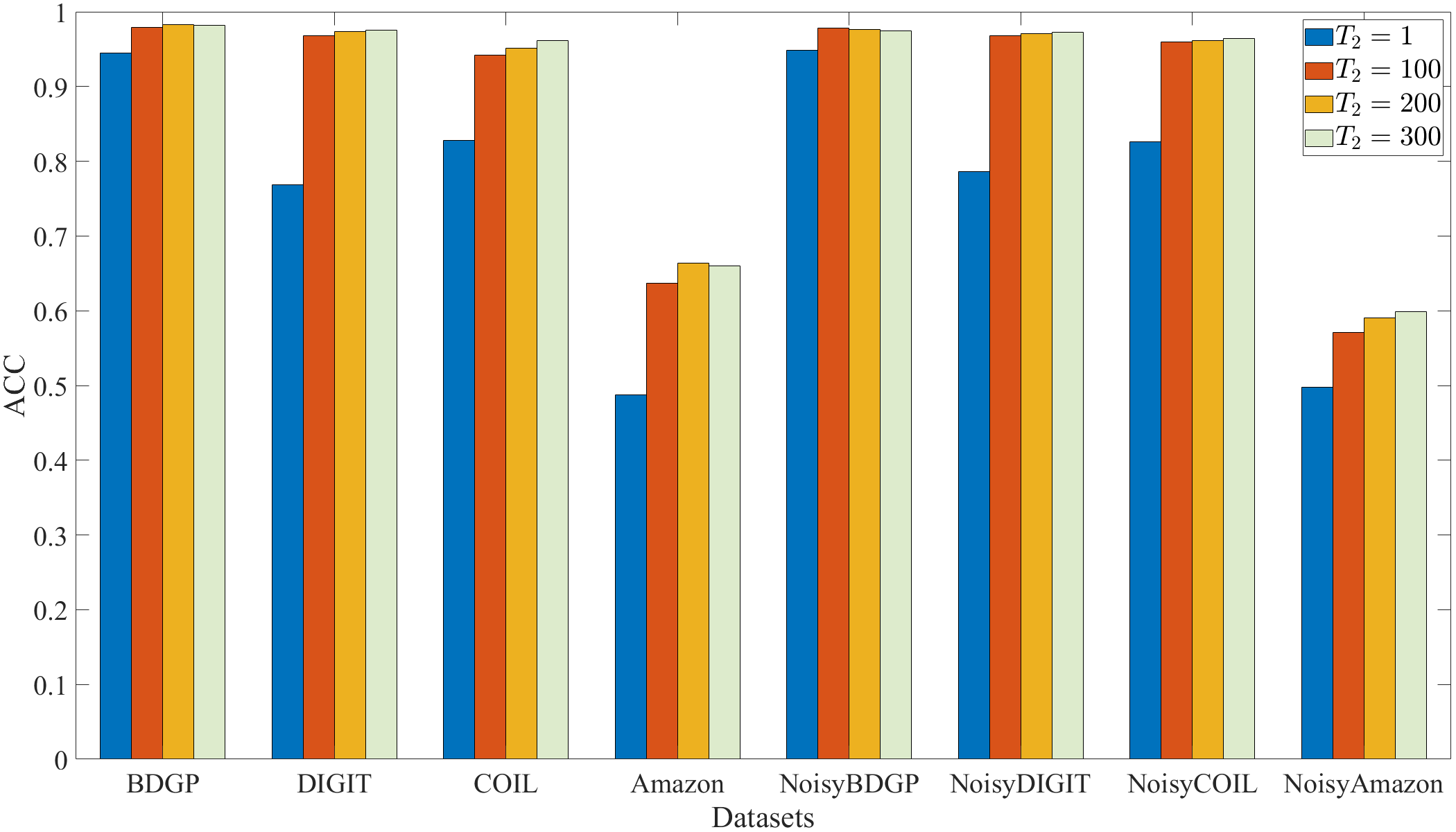}
    \caption{ACC $vs.$ $T_2$}
  \end{subfigure}
\caption{Different training iterations of $\mathcal{T}$-level (a) and $\mathcal{R}$-level (b) in the proposed two-level multi-view iterative optimization.}\label{BS}
\end{figure}

\textbf{Two conditions in clustering objective}.~We investigate the importance of the two conditions in Eq.~(\ref{cccc}). As shown in Table~\ref{tab:table06}, (i) $\mathbf{\Theta}^v$ denotes the first condition of un-shared parameters for all views and (ii) $\mathbf{A}^v$ indicates the second condition that multiple views are not required to be consistent. One could find that the results shown in (iii) achieve the best performance, which verifies the effectiveness of our MVCAN to mitigate the side effects caused by the noisy-view drawback. Concretely, the un-shared $\{\mathbf{\Theta}^v\}_{v=1}^V$ of all views eliminate their unfavourable interference. Moreover, $\{\mathbf{A}^v\}_{v=1}^V$ absolve the noisy views of conformity with the other views when minimizing the clustering objective.

\textbf{Two loss components in optimization}.~Table~\ref{tab:table04} lists the results of MVCAN with different loss components, where (a) denotes the clustering results of $K$-means on the direct concatenation of multi-view data. Compared with (a), both (b) and (c) can obtain improvements due to the representation learning objective achieved by $\mathcal{L}_{r}^v$ and the clustering objective achieved by $\mathcal{L}_{c}^v$, respectively. (d) obtains the best performance which indicates that the representation learning objective and the clustering objective have the effect of mutual promotion in our MVCAN, verified their importance.

\textbf{Two-level multi-view iterative optimization}.~Figure~\ref{BS} shows the performance by changing $T_1$ and $T_2$ in the first iteration of $\mathcal{T}$-level iteration and $\mathcal{R}$-level iteration. Based on the results, we have the following observations. When $T_1=1$ (\ie, the framework is without $\mathcal{T}$-level iteration), MVCAN is unable to infer the scaling factors for different views to generate the more effective robust learning target $\mathbf{T}$. Similarly, when $T_2=1$ (\ie, the framework is without $\mathcal{R}$-level iteration), MVCAN cannot learn the more effective representations with the learning target. When $T_1$ and $T_2$ increase, the performance also improves, which shows the effectiveness of our two-level multi-view iterative optimization. For all tested datasets, we set $T_1=2$ and $T_2=100$.

\subsection{Model Analysis}
This part showcases loss convergence and hyper-parameter analysis to further understand our proposed method.
\begin{figure}[!ht]
\centering
  \begin{subfigure}{0.49\linewidth}
    \includegraphics[width=\linewidth]{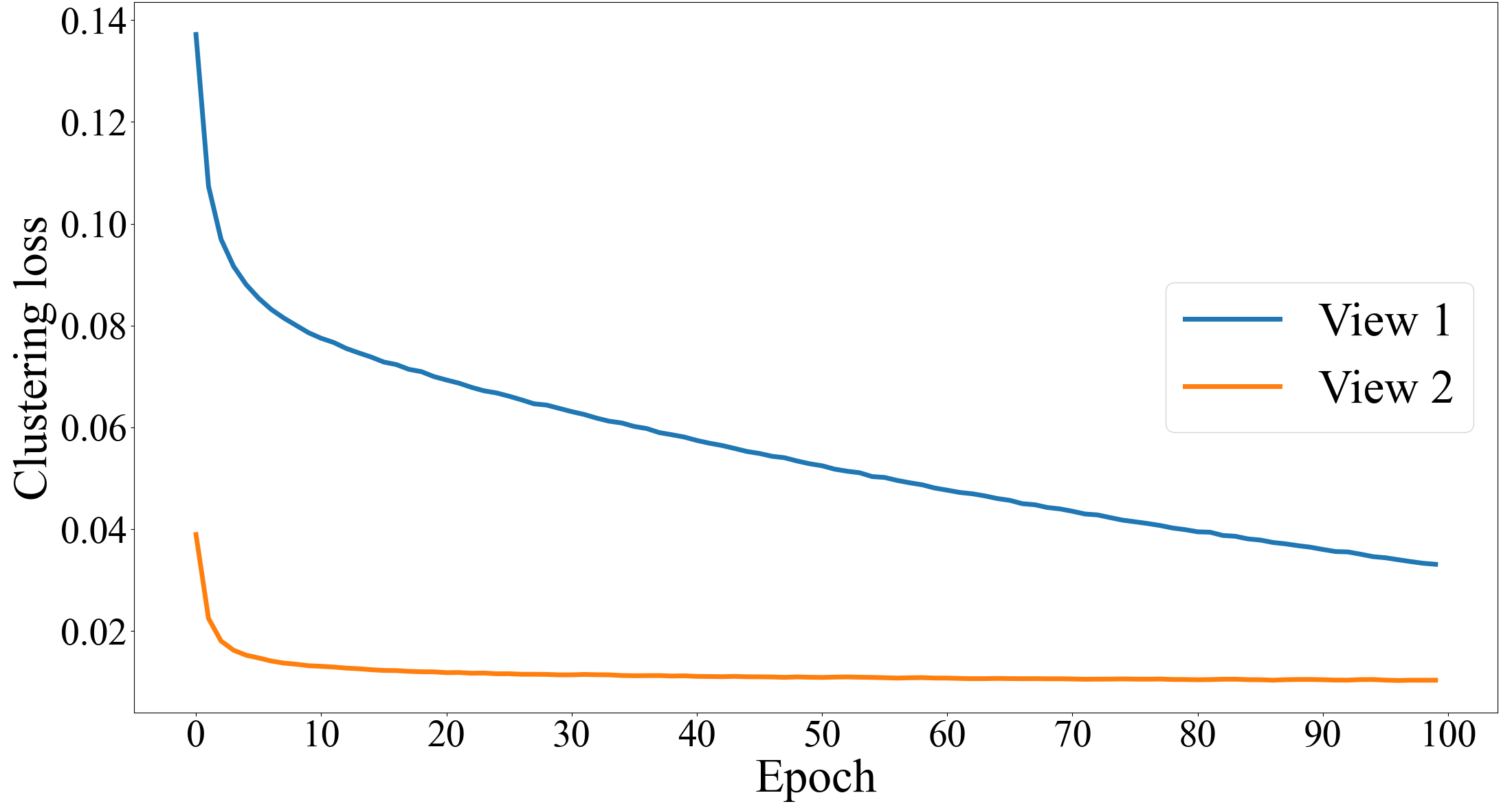}
    \caption{BDGP}
  \end{subfigure}
  \begin{subfigure}{0.49\linewidth}
    \includegraphics[width=\linewidth]{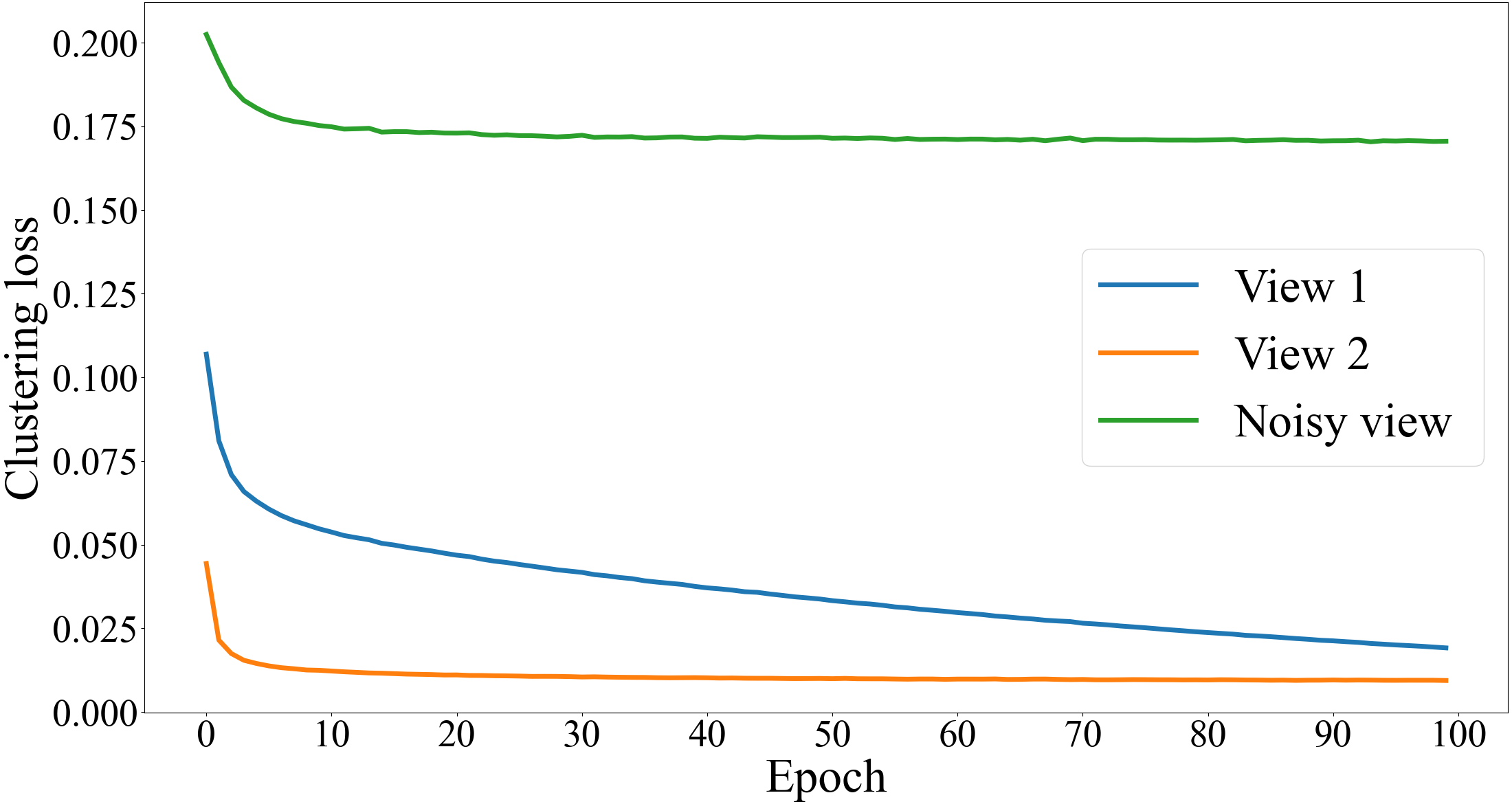}
    \caption{NoisyBDGP}
  \end{subfigure}
\caption{Loss $vs.$ Epoch on BDGP and NoisyBDGP.}\label{lossplot}
\end{figure}
\begin{figure}[!ht]
\centering
  \begin{subfigure}{0.49\linewidth}
    \includegraphics[width=\linewidth]{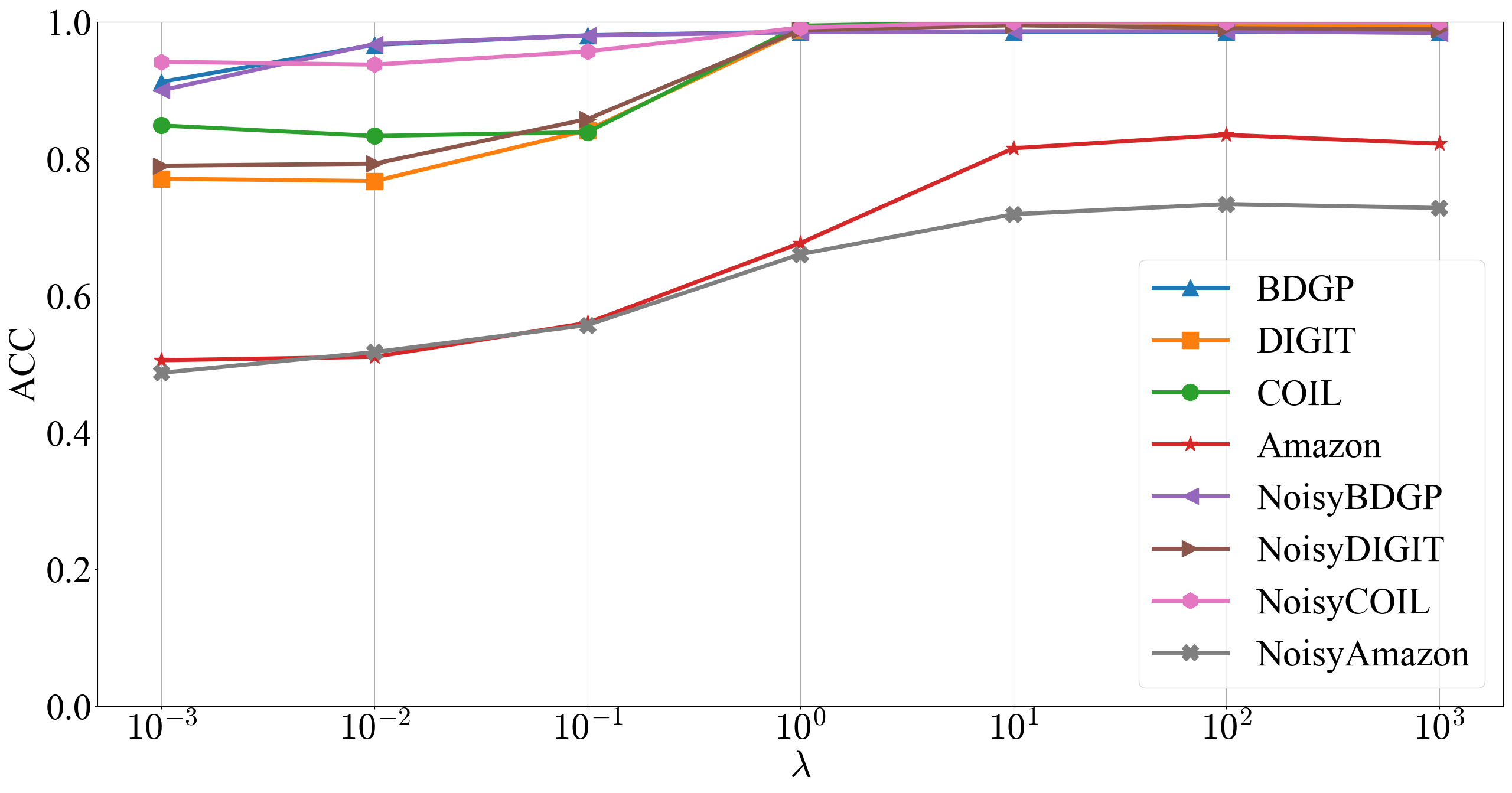}
    \caption{ACC $vs.$ $\lambda$}
  \end{subfigure}
  \begin{subfigure}{0.49\linewidth}
    \includegraphics[width=\linewidth]{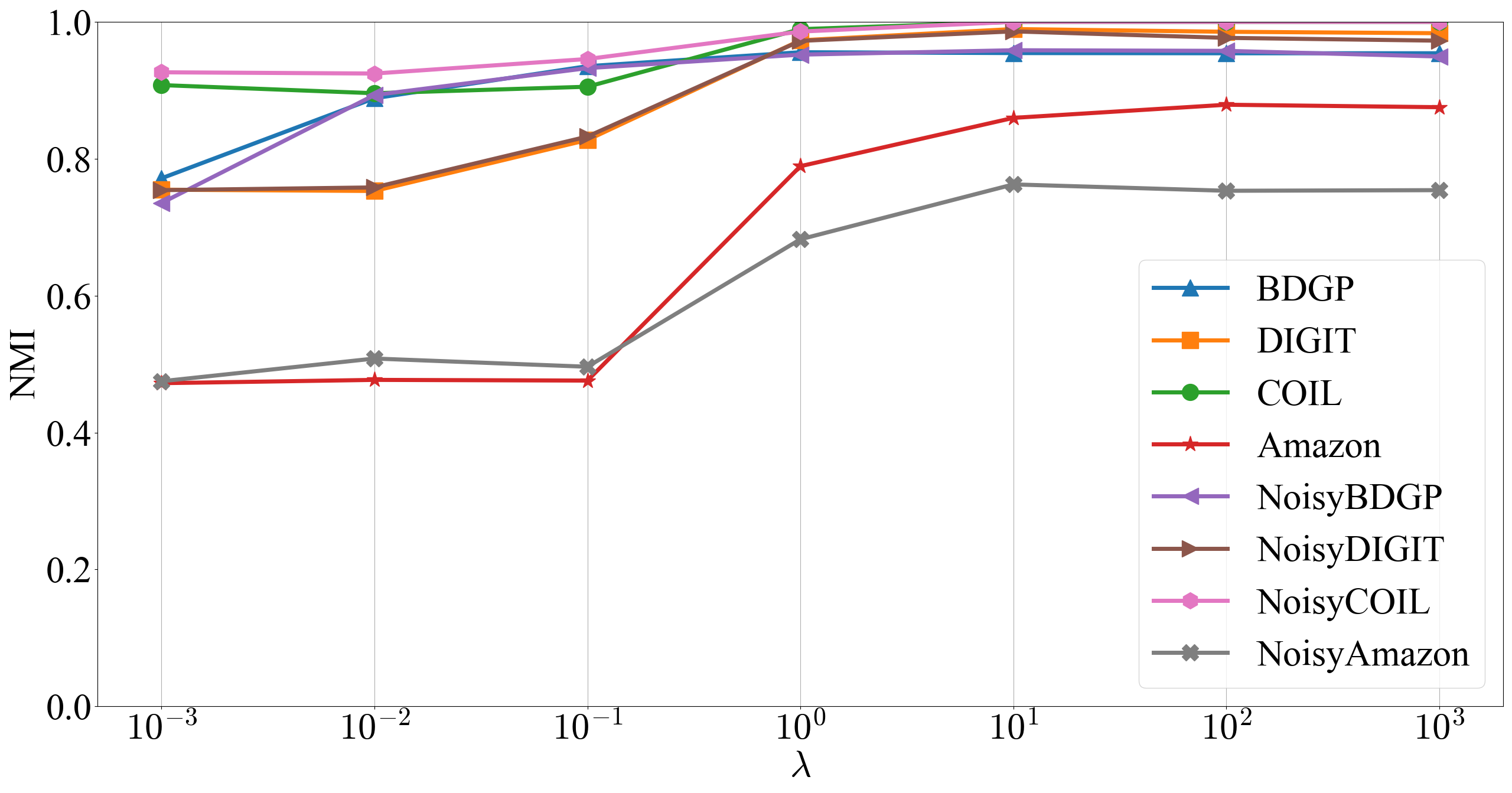}
    \caption{NMI $vs.$ $\lambda$}
  \end{subfigure}
\caption{ACC and NMI $vs.$ $\lambda$ on different datasets.}\label{trade}
\end{figure}

\textbf{Loss convergence analysis}.~Figure~\ref{lossplot} plots the clustering loss curve during training and we could observe that the model has good convergence properties. Moreover, it is worth noting that the loss values of the noisy views are larger than that of other views, which is consistent with our analysis in Sec.~\ref{AnalysisNVP}. Specifically, the features of the noisy views are not informative and have unclear cluster structures, which makes the clustering loss of noisy views difficult to be minimized. Therefore, we propose to constrain un-shared parameters and inconsistent clustering predictions for multiple views in the multi-view clustering objective of MVCAN, to alleviate the adverse impact of noisy views on the optimization process of other informative views.

\textbf{Hyper-parameter analysis}.~The hyper-parameter of MVCAN includes the trade-off $\lambda$ in Eq.~(\ref{loss}), and Figure~\ref{trade} shows the clustering effectiveness by traversing $\lambda$. The results indicate that $\lambda$ is insensitive in the range of $[10^{1}, 10^{3}]$.
Additionally, the cluster number $K$ in the model is changeable.
As shown in Figure~\ref{tsne0}, on DIGIT and NoisyDIGIT, we utilize $t$-SNE \cite{maaten2008visualizing} to visualize the scaled representations learned with different cluster numbers.
For these two datasets, we mark the representations with ground-truth labels and the truth $K$ is 10.
We could observe that MVCAN can learn clear cluster structures on the datasets with noise interference as that on normal ones, indicating the robustness of our method for noisy views.
When $K$ is small (\eg, $K=5$), we can observe that the representations of digits with similar shapes are gathered together, \eg, ``4-7-9'' in Figure~\ref{tsne0}(a).
When $K$ is large (\eg, $K=15$), we observe that the representations of the same digits are separated into two clusters, \eg, ``5'' in Figure~\ref{tsne0}(c) (colored in yellow).
Consequently, MVCAN could learn the coarse-grained or fine-grained cluster structures by changing the prior knowledge of $K$.

\begin{figure}[!t]
\centering
  \begin{subfigure}{0.32\linewidth}
    \includegraphics[width=\linewidth]{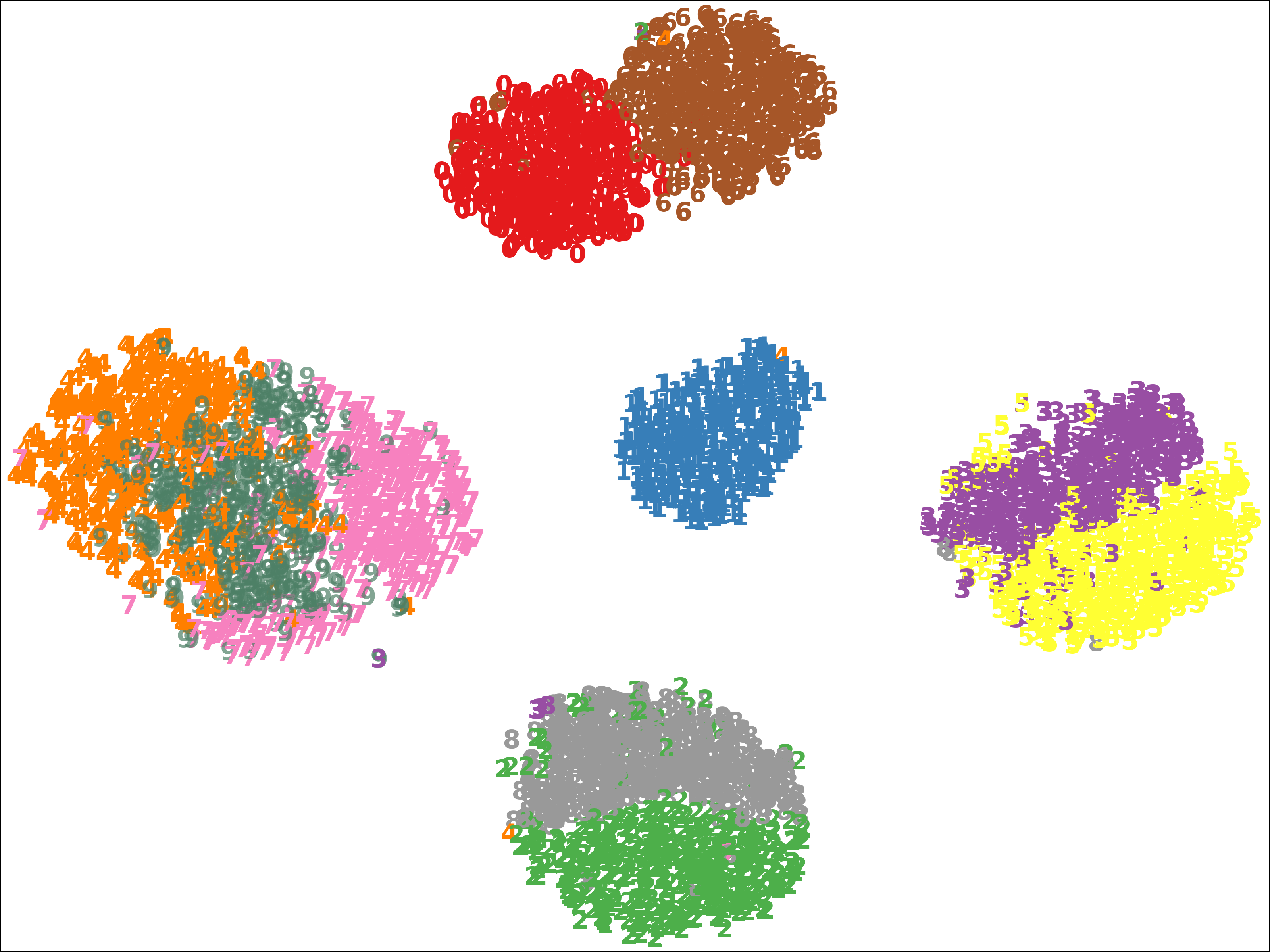}
    \caption{$K=5$}
  \end{subfigure}
  \begin{subfigure}{0.32\linewidth}
    \includegraphics[width=\linewidth]{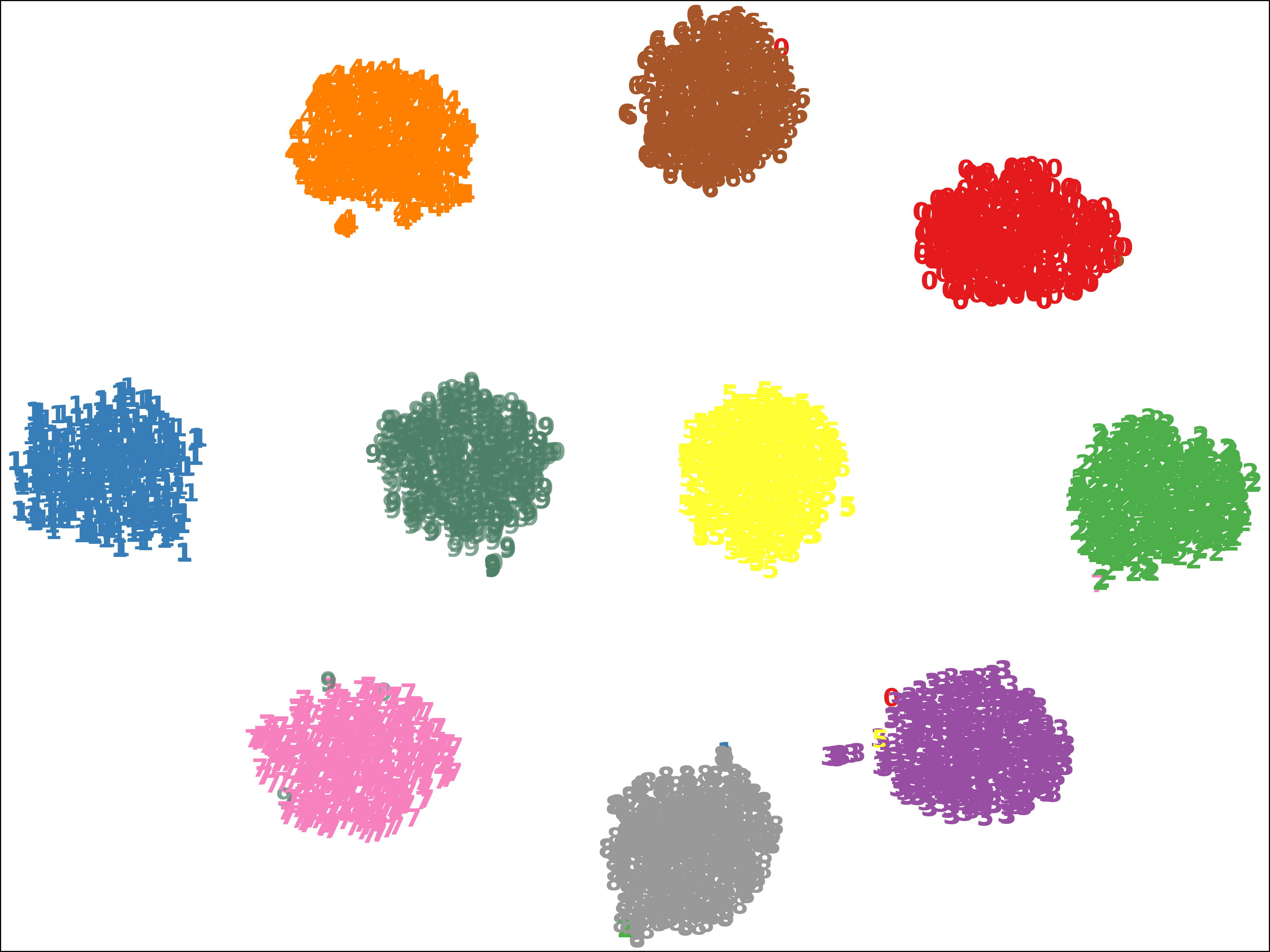}
    \caption{$K=10$}
  \end{subfigure}
  \begin{subfigure}{0.32\linewidth}
    \includegraphics[width=\linewidth]{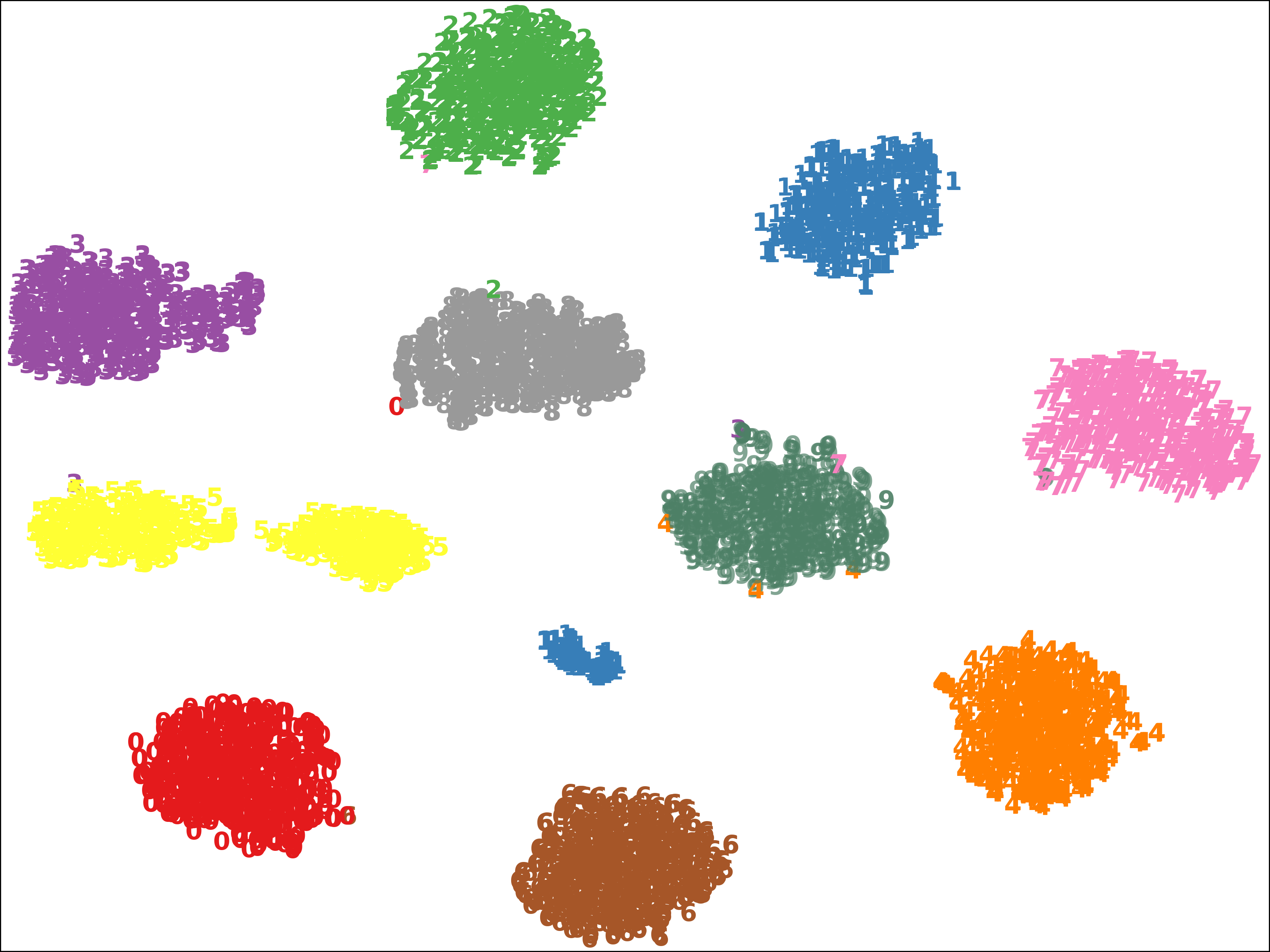}
    \caption{$K=15$}
  \end{subfigure}
  \begin{subfigure}{0.32\linewidth}
    \includegraphics[width=\linewidth]{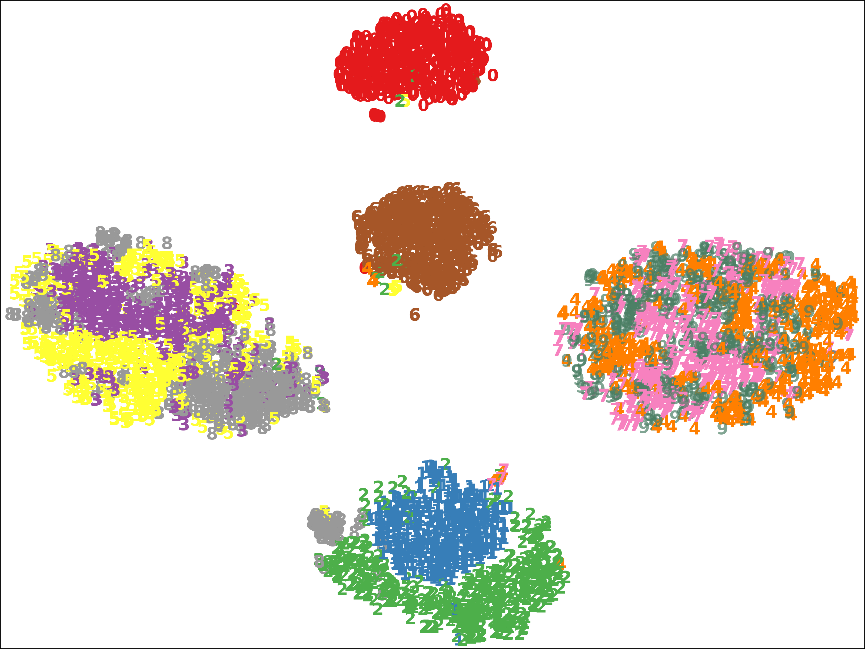}
    \caption{$K=5$}
  \end{subfigure}
  \begin{subfigure}{0.32\linewidth}
    \includegraphics[width=\linewidth]{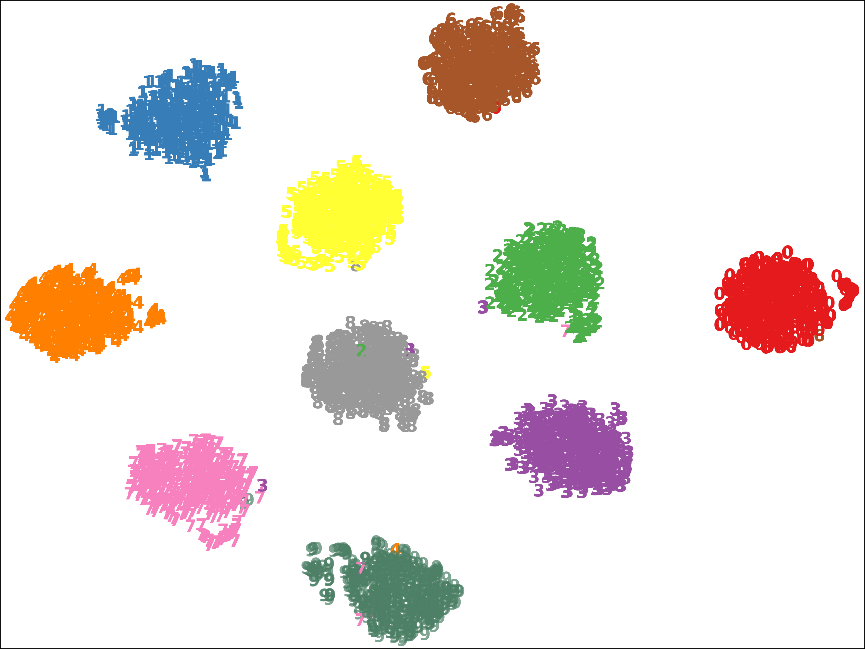}
    \caption{$K=10$}
  \end{subfigure}
  \begin{subfigure}{0.32\linewidth}
    \includegraphics[width=\linewidth]{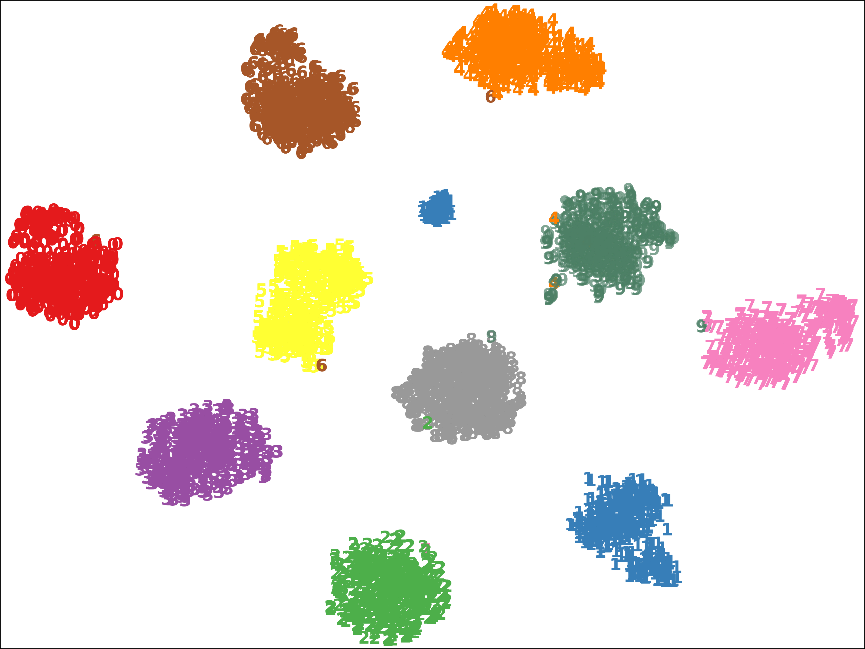}
    \caption{$K=15$}
  \end{subfigure}
\caption{Visualization of the representations learned with different prior cluster numbers on DIGIT (a-c) and NoisyDIGIT (d-f).}\label{tsne0}
\end{figure}

\section{Conclusion}\label{Conclusion}

This paper investigates the pervasive but challenging problem in multi-view clustering, \ie, Noisy-View Drawback (NVD).
To mitigate this issue, we proposed a novel deep multi-view clustering method dubbed MVCAN.
Comprehensive theoretical and empirical results verified the superior performance of MVCAN, together with the effectiveness of our proposed two conditions in clustering objective and of our two-level multi-view iteration in optimization.

We expect our work to produce beneficial impacts for self-supervised multi-view learning where the information qualities obtained from different views are difficult to be guaranteed and thus they bring noisy information. For example, if the views from some sensors/modalities are faulty or unreliable in unsupervised ensemble environments, it might be promising to take the NVD into account to design algorithms as did in MVCAN.
In addition, future work still needs to be devoted to reducing the sensitivity of parameter initialization in deep model and the class number.

\section*{Acknowledgment}
This work was supported in part by the National Key Research \& Development Program of China under Grant 2022YFA1004100,
in part by the Medico-Engineering Cooperation Funds from University of Electronic Science and Technology of China under Grant ZYGX2022YGRH009 and Grant ZYGX2022YGRH014.

{
    \small
    \bibliographystyle{ieeenat_fullname}
    \bibliography{main}

\begin{thebibliography}{73}
\providecommand{\natexlab}[1]{#1}
\providecommand{\url}[1]{\texttt{#1}}
\expandafter\ifx\csname urlstyle\endcsname\relax
  \providecommand{\doi}[1]{doi: #1}\else
  \providecommand{\doi}{doi: \begingroup \urlstyle{rm}\Url}\fi

\bibitem[Abavisani and Patel(2018)]{abavisani2018deep}
Mahdi Abavisani and Vishal~M Patel.
\newblock Deep multimodal subspace clustering networks.
\newblock \emph{IEEE Journal of Selected Topics in Signal Processing}, 12\penalty0 (6):\penalty0 1601--1614, 2018.

\bibitem[Baldi(2012)]{baldi2012autoencoders}
Pierre Baldi.
\newblock Autoencoders, unsupervised learning, and deep architectures.
\newblock In \emph{Proceedings of ICML Workshop on Unsupervised and Transfer Learning}, pages 37--49, 2012.

\bibitem[Cai et~al.(2012)Cai, Wang, Huang, and Ding]{cai2012joint}
Xiao Cai, Hua Wang, Heng Huang, and Chris Ding.
\newblock Joint stage recognition and anatomical annotation of drosophila gene expression patterns.
\newblock \emph{Bioinformatics}, 28\penalty0 (12):\penalty0 i16--i24, 2012.

\bibitem[Cao et~al.(2015)Cao, Zhang, Fu, Liu, and Zhang]{cao2015diversity}
Xiaochun Cao, Changqing Zhang, Huazhu Fu, Si Liu, and Hua Zhang.
\newblock Diversity-induced multi-view subspace clustering.
\newblock In \emph{CVPR}, pages 586--594, 2015.

\bibitem[Chen et~al.(2023)Chen, Mao, Woo, and Peng]{chen2023deep}
Jie Chen, Hua Mao, Wai~Lok Woo, and Xi Peng.
\newblock Deep multiview clustering by contrasting cluster assignments.
\newblock In \emph{ICCV}, pages 16752--16761, 2023.

\bibitem[Dong et~al.(2023)Dong, Wang, Jin, Liu, and Zhu]{dong2023cross}
Zhibin Dong, Siwei Wang, Jiaqi Jin, Xinwang Liu, and En Zhu.
\newblock Cross-view topology based consistent and complementary information for deep multi-view clustering.
\newblock In \emph{ICCV}, pages 19440--19451, 2023.

\bibitem[Fan et~al.(2020)Fan, Wang, Shi, Lu, Lin, and Wang]{fan2020one2multi}
Shaohua Fan, Xiao Wang, Chuan Shi, Emiao Lu, Ken Lin, and Bai Wang.
\newblock One2multi graph autoencoder for multi-view graph clustering.
\newblock In \emph{WWW}, pages 3070--3076, 2020.

\bibitem[Fei-Fei et~al.(2004)Fei-Fei, Fergus, and Perona]{fei2004learning}
Li Fei-Fei, Rob Fergus, and Pietro Perona.
\newblock Learning generative visual models from few training examples: An incremental bayesian approach tested on 101 object categories.
\newblock In \emph{CVPR}, pages 178--178, 2004.

\bibitem[Glorot et~al.(2011)Glorot, Bordes, and Bengio]{glorot2011deep}
Xavier Glorot, Antoine Bordes, and Yoshua Bengio.
\newblock Deep sparse rectifier neural networks.
\newblock In \emph{AISTATS}, pages 315--323, 2011.

\bibitem[Goodfellow et~al.(2014)Goodfellow, Pouget-Abadie, Mirza, Xu, Warde-Farley, Ozair, Courville, and Bengio]{goodfellow2014generative}
Ian Goodfellow, Jean Pouget-Abadie, Mehdi Mirza, Bing Xu, David Warde-Farley, Sherjil Ozair, Aaron Courville, and Yoshua Bengio.
\newblock Generative adversarial nets.
\newblock In \emph{NeurIPS}, pages 2672--2680, 2014.

\bibitem[Guo et~al.(2017)Guo, Gao, Liu, and Yin]{guo2017improved}
Xifeng Guo, Long Gao, Xinwang Liu, and Jianping Yin.
\newblock Improved deep embedded clustering with local structure preservation.
\newblock In \emph{IJCAI}, pages 1753--1759, 2017.

\bibitem[He et~al.(2014)He, Kan, Xie, and Chen]{he2014comment}
Xiangnan He, Min-Yen Kan, Peichu Xie, and Xiao Chen.
\newblock Comment-based multi-view clustering of web 2.0 items.
\newblock In \emph{WWW}, pages 771--782, 2014.

\bibitem[Huang et~al.(2020)Huang, Hu, Zhou, Lv, and Peng]{Huang0ZL020}
Zhenyu Huang, Peng Hu, Joey~Tianyi Zhou, Jiancheng Lv, and Xi Peng.
\newblock Partially view-aligned clustering.
\newblock In \emph{NeurIPS}, pages 2892--2902, 2020.

\bibitem[Huang et~al.(2023)Huang, Ren, Pu, Huang, Xu, and He]{huang2023self}
Zongmo Huang, Yazhou Ren, Xiaorong Pu, Shudong Huang, Zenglin Xu, and Lifang He.
\newblock Self-supervised graph attention networks for deep weighted multi-view clustering.
\newblock In \emph{AAAI}, pages 7936--7943, 2023.

\bibitem[Jin et~al.(2023)Jin, Wang, Dong, Liu, and Zhu]{jin2023deep}
Jiaqi Jin, Siwei Wang, Zhibin Dong, Xinwang Liu, and En Zhu.
\newblock Deep incomplete multi-view clustering with cross-view partial sample and prototype alignment.
\newblock In \emph{CVPR}, pages 11600--11609, 2023.

\bibitem[Kingma and Ba(2014)]{kingma2014adam}
Diederik~P Kingma and Jimmy Ba.
\newblock Adam: {A} method for stochastic optimization.
\newblock \emph{arXiv preprint arXiv:1412.6980}, 2014.

\bibitem[Li et~al.(2019{\natexlab{a}})Li, Zhang, Fu, Peng, Zhou, and Hu]{li2019reciprocal}
Ruihuang Li, Changqing Zhang, Huazhu Fu, Xi Peng, Tianyi Zhou, and Qinghua Hu.
\newblock Reciprocal multi-layer subspace learning for multi-view clustering.
\newblock In \emph{ICCV}, pages 8172--8180, 2019{\natexlab{a}}.

\bibitem[Li et~al.(2023)Li, Zhang, Yang, Peng, Yu, Liu, Lv, Chen, and Peng]{li2023scbridge}
Yunfan Li, Dan Zhang, Mouxing Yang, Dezhong Peng, Jun Yu, Yu Liu, Jiancheng Lv, Lu Chen, and Xi Peng.
\newblock scbridge embraces cell heterogeneity in single-cell rna-seq and atac-seq data integration.
\newblock \emph{Nature Communications}, 14\penalty0 (1):\penalty0 6045, 2023.

\bibitem[Li et~al.(2019{\natexlab{b}})Li, Wang, Tao, Gao, and Yang]{li2019deep}
Zhaoyang Li, Qianqian Wang, Zhiqiang Tao, Quanxue Gao, and Zhaohua Yang.
\newblock Deep adversarial multi-view clustering network.
\newblock In \emph{IJCAI}, pages 2952--2958, 2019{\natexlab{b}}.

\bibitem[Lin et~al.(2021)Lin, Gou, Liu, Li, Lv, and Peng]{lin2021completer}
Yijie Lin, Yuanbiao Gou, Zitao Liu, Boyun Li, Jiancheng Lv, and Xi Peng.
\newblock {COMPLETER}: Incomplete multi-view clustering via contrastive prediction.
\newblock In \emph{CVPR}, pages 11174--11183, 2021.

\bibitem[Lin et~al.(2012)Lin, Hu, Cheng, Hsieh, and Chen]{lin2012human}
Yan-Ching Lin, Min-Chun Hu, Wen-Huang Cheng, Yung-Huan Hsieh, and Hong-Ming Chen.
\newblock Human action recognition and retrieval using sole depth information.
\newblock In \emph{ACM MM}, pages 1053--1056, 2012.

\bibitem[Liu et~al.(2023)Liu, Liu, Yang, Liao, and Xia]{liu2023contrastive}
Jiyuan Liu, Xinwang Liu, Yuexiang Yang, Qing Liao, and Yuanqing Xia.
\newblock Contrastive multi-view kernel learning.
\newblock \emph{TPAMI}, pages 1--15, 2023.

\bibitem[Liu et~al.(2018)Liu, Zhu, Li, Wang, Tang, Yin, Shen, Wang, and Gao]{liu2018late}
Xinwang Liu, Xinzhong Zhu, Miaomiao Li, Lei Wang, Chang Tang, Jianping Yin, Dinggang Shen, Huaimin Wang, and Wen Gao.
\newblock Late fusion incomplete multi-view clustering.
\newblock \emph{TPAMI}, 41\penalty0 (10):\penalty0 2410--2423, 2018.

\bibitem[Liu et~al.(2021)Liu, Liu, Liao, Wang, Zhang, Tu, Tang, Liu, and Zhu]{liu2021one}
Xinwang Liu, Li Liu, Qing Liao, Siwei Wang, Yi Zhang, Wenxuan Tu, Chang Tang, Jiyuan Liu, and En Zhu.
\newblock One pass late fusion multi-view clustering.
\newblock In \emph{ICML}, pages 6850--6859, 2021.

\bibitem[Lu et~al.(2023)Lu, Lin, Yang, Peng, Hu, and Peng]{lu2023decoupled}
Yiding Lu, Yijie Lin, Mouxing Yang, Dezhong Peng, Peng Hu, and Xi Peng.
\newblock Decoupled contrastive multi-view clustering with high-order random walks.
\newblock \emph{arXiv preprint arXiv:2308.11164}, 2023.

\bibitem[Maaten and Hinton(2008)]{maaten2008visualizing}
Laurens van~der Maaten and Geoffrey Hinton.
\newblock Visualizing data using $t$-{SNE}.
\newblock \emph{JMLR}, 9:\penalty0 2579--2605, 2008.

\bibitem[MacQueen(1967)]{macqueen1967some}
James MacQueen.
\newblock Some methods for classification and analysis of multivariate observations.
\newblock In \emph{Proceedings of the Berkeley Symposium on Mathematical Statistics and Probability}, pages 281--297, 1967.

\bibitem[Madani et~al.(2013)Madani, Georg, and Ross]{madani2012using}
Omid Madani, Manfred Georg, and David Ross.
\newblock On using nearly-independent feature families for high precision and confidence.
\newblock \emph{Machine Learning}, 92:\penalty0 457--477, 2013.

\bibitem[Nene et~al.(1996)Nene, Nayar, Murase, et~al.]{nene1996columbia}
Sameer~A Nene, Shree~K Nayar, Hiroshi Murase, et~al.
\newblock Columbia object image library (coil-100).
\newblock \emph{http://www1. cs. columbia. edu/CAVE/software/softlib/coil-100. php}, 1996.

\bibitem[Ng et~al.(2001)Ng, Jordan, and Weiss]{ng2002spectral}
Andrew~Y. Ng, Michael~I. Jordan, and Yair Weiss.
\newblock On spectral clustering: Analysis and an algorithm.
\newblock In \emph{NeurIPS}, pages 849--856, 2001.

\bibitem[Nie et~al.(2016)Nie, Li, and Li]{nie2016parameter}
Feiping Nie, Jing Li, and Xuelong Li.
\newblock Parameter-free auto-weighted multiple graph learning: a framework for multiview clustering and semi-supervised classification.
\newblock In \emph{IJCAI}, pages 1881--1887, 2016.

\bibitem[Nie et~al.(2018)Nie, Tian, and Li]{nie2018multiview}
Feiping Nie, Lai Tian, and Xuelong Li.
\newblock Multiview clustering via adaptively weighted procrustes.
\newblock In \emph{KDD}, pages 2022--2030, 2018.

\bibitem[Nigam and Ghani(2000)]{nigam2000analyzing}
Kamal Nigam and Rayid Ghani.
\newblock Analyzing the effectiveness and applicability of co-training.
\newblock In \emph{CIKM}, pages 86--93, 2000.

\bibitem[Oord et~al.(2018)Oord, Li, and Vinyals]{oord2018representation}
Aaron van~den Oord, Yazhe Li, and Oriol Vinyals.
\newblock Representation learning with contrastive predictive coding.
\newblock \emph{arXiv preprint arXiv:1807.03748}, 2018.

\bibitem[Paszke et~al.(2019)Paszke, Gross, Massa, Lerer, Bradbury, Chanan, Killeen, Lin, Gimelshein, Antiga, et~al.]{paszke2019pytorch}
Adam Paszke, Sam Gross, Francisco Massa, Adam Lerer, James Bradbury, Gregory Chanan, Trevor Killeen, Zeming Lin, Natalia Gimelshein, Luca Antiga, et~al.
\newblock Py{T}orch: An imperative style, high-performance deep learning library.
\newblock In \emph{NeurIPS}, pages 8024--8035, 2019.

\bibitem[Peng et~al.(2019)Peng, Huang, Lv, Zhu, and Zhou]{peng2019comic}
Xi Peng, Zhenyu Huang, Jiancheng Lv, Hongyuan Zhu, and Joey~Tianyi Zhou.
\newblock {COMIC}: Multi-view clustering without parameter selection.
\newblock In \emph{ICML}, pages 5092--5101, 2019.

\bibitem[Peng et~al.(2022)Peng, Li, Tsang, Zhu, Lv, and Zhou]{peng2022xai}
Xi Peng, Yunfan Li, Ivor~W Tsang, Hongyuan Zhu, Jiancheng Lv, and Joey~Tianyi Zhou.
\newblock Xai beyond classification: Interpretable neural clustering.
\newblock \emph{JMLR}, 23\penalty0 (1):\penalty0 227--254, 2022.

\bibitem[Rappoport and Shamir(2018)]{rappoport2018multi}
Nimrod Rappoport and Ron Shamir.
\newblock Multi-omic and multi-view clustering algorithms: review and cancer benchmark.
\newblock \emph{Nucleic acids research}, 46\penalty0 (20):\penalty0 10546--10562, 2018.

\bibitem[Ren et~al.(2022)Ren, Pu, Yang, Xu, Li, Pu, Yu, and He]{ren2022deep}
Yazhou Ren, Jingyu Pu, Zhimeng Yang, Jie Xu, Guofeng Li, Xiaorong Pu, Philip~S Yu, and Lifang He.
\newblock Deep clustering: A comprehensive survey.
\newblock \emph{arXiv preprint arXiv:2210.04142}, 2022.

\bibitem[Saenko et~al.(2010)Saenko, Kulis, Fritz, and Darrell]{saenko2010adapting}
Kate Saenko, Brian Kulis, Mario Fritz, and Trevor Darrell.
\newblock Adapting visual category models to new domains.
\newblock In \emph{ECCV}, pages 213--226, 2010.

\bibitem[Tang et~al.(2022)Tang, Li, Wang, Liu, Zhang, and Zhu]{tang2022unified}
Chang Tang, Zhenglai Li, Jun Wang, Xinwang Liu, Wei Zhang, and En Zhu.
\newblock Unified one-step multi-view spectral clustering.
\newblock \emph{TKDE}, 35\penalty0 (6):\penalty0 6449--6460, 2022.

\bibitem[Tang and Liu(2022{\natexlab{a}})]{tang2022deep}
Huayi Tang and Yong Liu.
\newblock Deep safe multi-view clustering: Reducing the risk of clustering performance degradation caused by view increase.
\newblock In \emph{CVPR}, pages 202--211, 2022{\natexlab{a}}.

\bibitem[Tang and Liu(2022{\natexlab{b}})]{tang2022deepi}
Huayi Tang and Yong Liu.
\newblock Deep safe incomplete multi-view clustering: Theorem and algorithm.
\newblock In \emph{ICML}, pages 21090--21110, 2022{\natexlab{b}}.

\bibitem[Tian et~al.(2020)Tian, Krishnan, and Isola]{tian2019contrastive}
Yonglong Tian, Dilip Krishnan, and Phillip Isola.
\newblock Contrastive multiview coding.
\newblock In \emph{ECCV}, pages 776--794, 2020.

\bibitem[Trosten et~al.(2021)Trosten, Løkse, Jenssen, and Kampffmeyer]{trostenMVC}
Daniel~J. Trosten, Sigurd Løkse, Robert Jenssen, and Michael Kampffmeyer.
\newblock Reconsidering representation alignment for multi-view clustering.
\newblock In \emph{CVPR}, pages 1255--1265, 2021.

\bibitem[Tzortzis and Likas(2012)]{tzortzis2012kernel}
Grigorios Tzortzis and Aristidis Likas.
\newblock Kernel-based weighted multi-view clustering.
\newblock In \emph{ICDM}, pages 675--684, 2012.

\bibitem[V{\'a}zquez and P{\'e}rez-Neira(2020)]{vazquez2020multigraph}
Miguel~{\'A}ngel V{\'a}zquez and Ana~I P{\'e}rez-Neira.
\newblock Multigraph spectral clustering for joint content delivery and scheduling in beam-free satellite communications.
\newblock In \emph{ICASSP}, pages 8802--8806, 2020.

\bibitem[Wang et~al.(2013)Wang, Nie, and Huang]{wang2013multi}
Hua Wang, Feiping Nie, and Heng Huang.
\newblock Multi-view clustering and feature learning via structured sparsity.
\newblock In \emph{ICML}, pages 352--360, 2013.

\bibitem[Wang et~al.(2021)Wang, Ding, Tao, Gao, and Fu]{wang2021generative}
Qianqian Wang, Zhengming Ding, Zhiqiang Tao, Quanxue Gao, and Yun Fu.
\newblock Generative partial multi-view clustering with adaptive fusion and cycle consistency.
\newblock \emph{TIP}, 30:\penalty0 1771--1783, 2021.

\bibitem[Wang et~al.(2019)Wang, Liu, Zhu, Tang, Liu, Hu, Xia, and Yin]{wang2019multi}
Siwei Wang, Xinwang Liu, En Zhu, Chang Tang, Jiyuan Liu, Jingtao Hu, Jingyuan Xia, and Jianping Yin.
\newblock Multi-view clustering via late fusion alignment maximization.
\newblock In \emph{IJCAI}, pages 3778--3784, 2019.

\bibitem[Wang et~al.(2022)Wang, Liu, Liu, Tu, Zhu, Liu, Zhou, and Zhu]{wang2022highly}
Siwei Wang, Xinwang Liu, Li Liu, Wenxuan Tu, Xinzhong Zhu, Jiyuan Liu, Sihang Zhou, and En Zhu.
\newblock Highly-efficient incomplete large-scale multi-view clustering with consensus bipartite graph.
\newblock In \emph{CVPR}, pages 9776--9785, 2022.

\bibitem[Wang et~al.(2016)Wang, Wenjie, Wu, Lin, Fang, and Pan]{wang2016iterative}
Yang Wang, Zhang Wenjie, Lin Wu, Xuemin Lin, Meng Fang, and Shirui Pan.
\newblock Iterative views agreement: an iterative low-rank based structured optimization method to multi-view spectral clustering.
\newblock In \emph{IJCAI}, pages 2153--2159, 2016.

\bibitem[Wang et~al.(2018)Wang, Wu, Lin, and Gao]{wang2018multiview}
Yang Wang, Lin Wu, Xuemin Lin, and Junbin Gao.
\newblock Multiview spectral clustering via structured low-rank matrix factorization.
\newblock \emph{TNNLS}, 29\penalty0 (10):\penalty0 4833--4843, 2018.

\bibitem[Wen et~al.(2018)Wen, Xu, and Liu]{wen2018incomplete}
Jie Wen, Yong Xu, and Hong Liu.
\newblock Incomplete multiview spectral clustering with adaptive graph learning.
\newblock \emph{TCYB}, 50\penalty0 (4):\penalty0 1418--1429, 2018.

\bibitem[Wen et~al.(2020)Wen, Zhang, Zhang, Wu, Fei, Xu, and Zhang]{wen2020dimc}
Jie Wen, Zheng Zhang, Zhao Zhang, Zhihao Wu, Lunke Fei, Yong Xu, and Bob Zhang.
\newblock {DIMC}-net: Deep incomplete multi-view clustering network.
\newblock In \emph{ACM MM}, pages 3753--3761, 2020.

\bibitem[Wen et~al.(2023)Wen, Liu, Xu, Wu, Huang, Fei, and Xu]{wen2023highly}
Jie Wen, Chengliang Liu, Gehui Xu, Zhihao Wu, Chao Huang, Lunke Fei, and Yong Xu.
\newblock Highly confident local structure based consensus graph learning for incomplete multi-view clustering.
\newblock In \emph{CVPR}, pages 15712--15721, 2023.

\bibitem[Xie et~al.(2016)Xie, Girshick, and Farhadi]{xie2016unsupervised}
Junyuan Xie, Ross Girshick, and Ali Farhadi.
\newblock Unsupervised deep embedding for clustering analysis.
\newblock In \emph{ICML}, pages 478--487, 2016.

\bibitem[Xie et~al.(2020)Xie, Lin, Qu, Li, Zhang, Ma, Wen, and Tao]{xie2020joint}
Yuan Xie, Bingqian Lin, Yanyun Qu, Cuihua Li, Wensheng Zhang, Lizhuang Ma, Yonggang Wen, and Dacheng Tao.
\newblock Joint deep multi-view learning for image clustering.
\newblock \emph{TKDE}, 33\penalty0 (11):\penalty0 3594--3606, 2020.

\bibitem[Xu et~al.(2019)Xu, Guan, Zhao, Wu, Niu, and Ling]{xu2019adversarial}
Cai Xu, Ziyu Guan, Wei Zhao, Hongchang Wu, Yunfei Niu, and Beilei Ling.
\newblock Adversarial incomplete multi-view clustering.
\newblock In \emph{IJCAI}, pages 3933--3939, 2019.

\bibitem[Xu et~al.(2022)Xu, Li, Ren, Peng, Mo, Shi, and Zhu]{xu2022deep}
Jie Xu, Chao Li, Yazhou Ren, Liang Peng, Yujie Mo, Xiaoshuang Shi, and Xiaofeng Zhu.
\newblock Deep incomplete multi-view clustering via mining cluster complementarity.
\newblock In \emph{AAAI}, pages 8761--8769, 2022.

\bibitem[Xu et~al.(2023{\natexlab{a}})Xu, Ren, Shi, Shen, and Zhu]{xu2023untie}
Jie Xu, Yazhou Ren, Xiaoshuang Shi, Heng~Tao Shen, and Xiaofeng Zhu.
\newblock {UNTIE}: Clustering analysis with disentanglement in multi-view information fusion.
\newblock \emph{Information Fusion}, 100:\penalty0 101937, 2023{\natexlab{a}}.

\bibitem[Xu et~al.(2023{\natexlab{b}})Xu, Ren, Tang, Yang, Pan, Yang, Pu, Yu, and He]{9839616}
Jie Xu, Yazhou Ren, Huayi Tang, Zhimeng Yang, Lili Pan, Yang Yang, Xiaorong Pu, Philip~S. Yu, and Lifang He.
\newblock Self-supervised discriminative feature learning for deep multi-view clustering.
\newblock \emph{TKDE}, 35\penalty0 (7):\penalty0 7470--7482, 2023{\natexlab{b}}.

\bibitem[Yan et~al.(2023)Yan, Zhang, Lv, Tang, Yue, Liao, and Lin]{yan2023gcfagg}
Weiqing Yan, Yuanyang Zhang, Chenlei Lv, Chang Tang, Guanghui Yue, Liang Liao, and Weisi Lin.
\newblock {GCFAgg}: Global and cross-view feature aggregation for multi-view clustering.
\newblock In \emph{CVPR}, pages 19863--19872, 2023.

\bibitem[Yang et~al.(2022)Yang, Li, Hu, Bai, Lv, and Peng]{yang2022robust}
Mouxing Yang, Yunfan Li, Peng Hu, Jinfeng Bai, Jiancheng Lv, and Xi Peng.
\newblock Robust multi-view clustering with incomplete information.
\newblock \emph{TPAMI}, 45\penalty0 (1):\penalty0 1055--1069, 2022.

\bibitem[Yang et~al.(2021)Yang, Liang, Yan, Li, and Xie]{yang2020uniform}
Zuyuan Yang, Naiyao Liang, Wei Yan, Zhenni Li, and Shengli Xie.
\newblock Uniform distribution non-negative matrix factorization for multiview clustering.
\newblock \emph{TCYB}, pages 3249--3262, 2021.

\bibitem[Ye et~al.(2018)Ye, Liu, and Yin]{ye2018multi}
Yongkai Ye, Xinwang Liu, and Jianping Yin.
\newblock Multi-view clustering with noisy views.
\newblock In \emph{Proceedings of the International Conference on Computer Science and Artificial Intelligence}, pages 339--344, 2018.

\bibitem[Zhan et~al.(2018)Zhan, Nie, Wang, and Yang]{zhan2018multiview}
Kun Zhan, Feiping Nie, Jing Wang, and Yi Yang.
\newblock Multiview consensus graph clustering.
\newblock \emph{TIP}, 28\penalty0 (3):\penalty0 1261--1270, 2018.

\bibitem[Zhan et~al.(2019)Zhan, Niu, Chen, Nie, Zhang, and Yang]{Zhan8052206}
Kun Zhan, Chaoxi Niu, Changlu Chen, Feiping Nie, Changqing Zhang, and Yi Yang.
\newblock Graph structure fusion for multiview clustering.
\newblock \emph{TKDE}, 31\penalty0 (10):\penalty0 1984--1993, 2019.

\bibitem[Zhang et~al.(2017)Zhang, Hu, Fu, Zhu, and Cao]{zhang2017latent}
Changqing Zhang, Qinghua Hu, Huazhu Fu, Pengfei Zhu, and Xiaochun Cao.
\newblock Latent multi-view subspace clustering.
\newblock In \emph{CVPR}, pages 4279--4287, 2017.

\bibitem[Zhang et~al.(2018)Zhang, Fu, Hu, Cao, Xie, Tao, and Xu]{zhang2018generalized}
Changqing Zhang, Huazhu Fu, Qinghua Hu, Xiaochun Cao, Yuan Xie, Dacheng Tao, and Dong Xu.
\newblock Generalized latent multi-view subspace clustering.
\newblock \emph{TPAMI}, 42\penalty0 (1):\penalty0 86--99, 2018.

\bibitem[Zhang et~al.(2020)Zhang, Cui, Han, Zhou, Fu, and Hu]{zhang2020deep}
Changqing Zhang, Yajie Cui, Zongbo Han, Joey~Tianyi Zhou, Huazhu Fu, and Qinghua Hu.
\newblock Deep partial multi-view learning.
\newblock \emph{TPAMI}, 44\penalty0 (5):\penalty0 2402--2415, 2020.

\bibitem[Zhou and Shen(2020)]{zhou2020end}
Runwu Zhou and Yi-Dong Shen.
\newblock End-to-end adversarial-attention network for multi-modal clustering.
\newblock In \emph{CVPR}, pages 14619--14628, 2020.

\bibitem[Zhou et~al.(2019)Zhou, Zhang, Peng, Bhaskar, and Yang]{zhou2019dual}
Tao Zhou, Changqing Zhang, Xi Peng, Harish Bhaskar, and Jie Yang.
\newblock Dual shared-specific multiview subspace clustering.
\newblock \emph{TCYB}, 50\penalty0 (8):\penalty0 3517--3530, 2019.

\end{thebibliography}
}

\iftrue

\newpage
\onecolumn

\setcounter{figure}{0}
\setcounter{table}{0}
\setcounter{theorem}{0}
\setcounter{definition}{0}
\setcounter{equation}{0}

\section*{Appendix A: Framework, Related Work, and Notations}\label{rl}

\begin{figure}[!ht]
\centering
\includegraphics[height=1.9in]{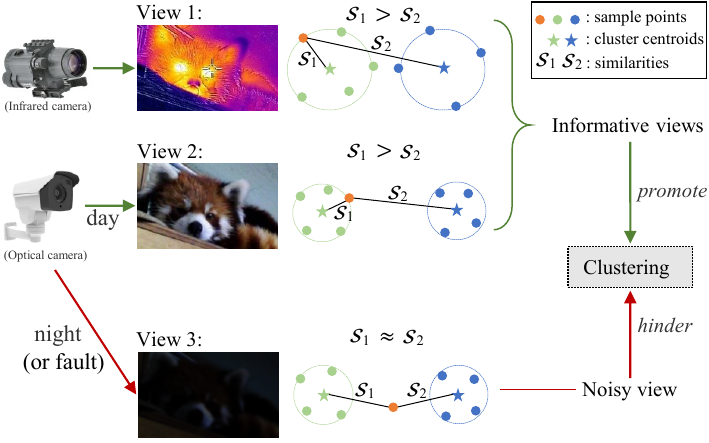}
\caption{Illustration of the noisy-view drawback (NVD).
The informative views have distinct representation similarities, which can promote clustering due to their consistency and complementarity.
However, the noisy views have indistinct representation similarities. For instance, the views extracted from faulty or inapplicable sensors will bring noisy information and hinder clustering, making it be of practical significance to investigate the noise robustness.
}\label{nvp}
\end{figure}

\begin{figure*}[!ht]
\centering
\includegraphics[width=5.8in]{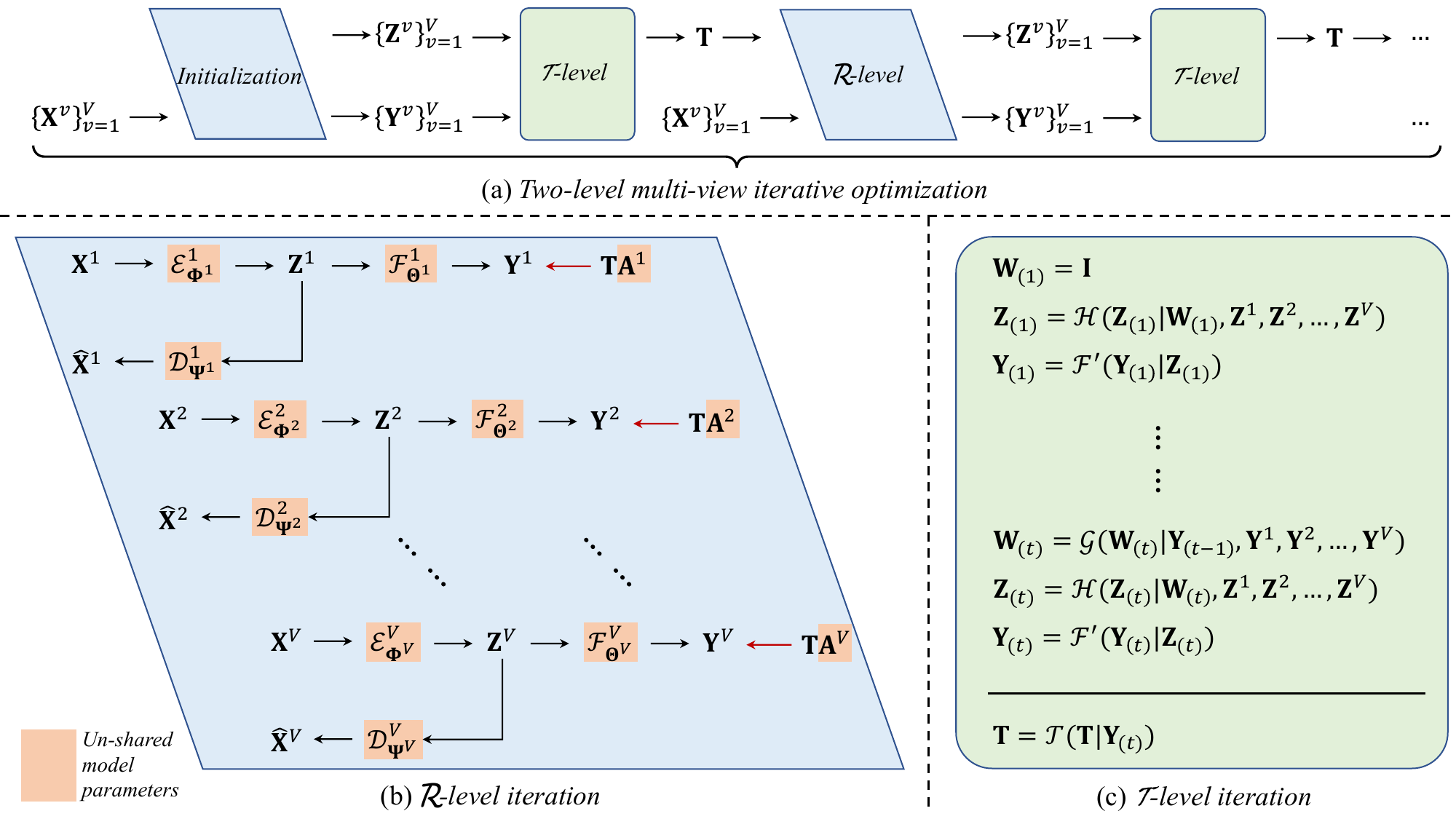}
\caption{The frame diagram of our MVCAN.
Specifically,
(a) MVCAN utilizes the two-level multi-view iterative optimization strategy to train the model for clustering multi-view data $\{\mathbf{X}^v\}_{v=1}^V$,
where
(b) $\mathcal{R}$-level iteration adjusts $\mathbf{T}$ for corresponding to each $\mathbf{Y}^v$ by $\mathbf{A}^v$ and updates the decoupled model with un-shared network parameters for $V$ views, to obtain their individual representations $\{\mathbf{Z}^v\}_{v=1}^V$ and clustering soft labels $\{\mathbf{Y}^v\}_{v=1}^V$;
(c) $\mathcal{T}$-level iteration is established to infer the robust learning target $\mathbf{T}$ based on the already learned $\{\mathbf{Z}^v\}_{v=1}^V$ and $\{\mathbf{Y}^v\}_{v=1}^V$ (including the iterations of scaling matrix $\mathbf{W}_{(t)}$, scaled representation $\mathbf{Z}_{(t)}$, and robust soft labels $\mathbf{Y}_{(t)}$). \emph{Note:}~$\mathbf{W}_{(t)}$ plays the role of scaling the values of different dimensions of $\mathbf{Z}_{(t)}$, so we name it the scaling matrix. $\mathbf{Y}_{(t)}$ finally predicts the clustering results.
}\label{framework}
\end{figure*}

\begin{table*}[!ht]
  \centering
  \caption{Notations and their descriptions in this paper.} \label{tab: Notations}
    \begin{tabular}{|c|l|}
    \hline
    \multicolumn{1}{|c|}{Notations} & \multicolumn{1}{c|}{Descriptions} \\ \hline
    $N$  & the number of samples or the data size \\ \hline
    $V$  & the number of views in the multi-view dataset \\ \hline
    $K$  & the number of clusters in the multi-view dataset \\ \hline
    $i$,$j$,$k$,$v$,$t$,$e$,$m^*$  & the index notations \\ \hline
    $\mathbf{X}$  & the data for a single-view method, $\mathbf{x}_i \in \mathbf{X}$ \\ \hline
    $\mathbf{Z}$  & the learned representation in a single-view method, $\mathbf{z}_i \in \mathbf{Z}$ \\ \hline
    $\mathbf{Y}$  & the learned soft labels in a single-view method, $y_{ij} \in \mathbf{Y}$ \\ \hline
    $\mathbf{T}$  & the learning target, $\mathbf{t}_i \in \mathbf{T}$, $t_{ij} \in \mathbf{T}$ \\ \hline
    $\mathbf{X}^v$  & the $v$-th view's data for a multi-view method, $\mathbf{x}^v_i \in \mathbf{X}^v$ \\ \hline
    $\mathbf{Z}^v$  & the $v$-th view's representations learned by a multi-view method, $\mathbf{z}^v_i \in \mathbf{Z}^v$\\ \hline
    $\mathbf{Y}^v$  & the $v$-th view's soft labels in a multi-view method, $\mathbf{y}^v_i \in \mathbf{Y}^v$, $y^v_{ij} \in \mathbf{Y}^v$ \\ \hline
    $\Hat{\mathbf{X}}^v$  & the $v$-th view's reconstructed data for a multi-view method \\ \hline
    $D_v$         & the dimensionality of $\mathbf{X}^v$ and $\Hat{\mathbf{X}}^v$ \\ \hline
    $d_v$         & the dimensionality of $\mathbf{Z}^v$ \\ \hline
    $\mathcal{E}_{\mathbf{\Phi}}$ & the encoder with parameter $\mathbf{\Phi}$ for a single-view method \\ \hline
    $\bm{\upmu}_j$ & the $j$-th cluster centroid in a single-view method \\ \hline
    $\mathcal{E}^v_{\mathbf{\Phi}^v}$ & the $v$-th view's encoder with parameter $\mathbf{\Phi}^v$ for a multi-view method \\ \hline
    $\mathcal{D}^v_{\mathbf{\Psi}^v}$ & the $v$-th view's decoder with parameter $\mathbf{\Psi}^v$ for a multi-view method \\ \hline
    $\bm{\upmu}_j^v$ & the $j$-th cluster centroid in the $v$-th view for a multi-view method \\ \hline
    $\mathcal{F}_{\mathbf{\Theta}}$ & the fusion module $\mathcal{F}$ with the parameter $\mathbf{\Theta}$ shared for multiple views \\ \hline
    $\mathcal{F}^v_{\mathbf{\Theta}^v}$ & the $v$-th view's clustering module $\mathcal{F}^v$ with the parameter $\mathbf{\Theta}^v$ in MVCAN \\ \hline
    $\mathcal{D}(\mathbf{a},\mathbf{b})$ & the squared Euclidean distance between representations $\mathbf{a}$ and $\mathbf{b}$ \\ \hline
    $\mathcal{S}(\mathbf{a},\mathbf{b})$ & the similarity between representations $\mathbf{a}$ and $\mathbf{b}$  \\ \hline
    $\varepsilon$ & a sufficiently small value in Definitions \\ \hline
    $\delta$      & a threshold in Theorems \\ \hline
    $\mathbf{L}$     &  the ground-truth label matrix \\ \hline
    $\mathbf{A}$     &  the matching matrix to calculate the clustering accuracy \\ \hline
    $\mathbf{\check{Y}}$ & the transformed prediction label matrix,  $\mathbf{\check{y}}_i \in  \mathbf{\check{Y}}$, $\check{y}_{ij} \in \mathbf{\check{Y}}$ \\ \hline
    $\mathbf{A}^v$     &  the matching matrix in the Condition 2 of MVCAN \\ \hline
    $\mathbf{I}$, $\mathbf{I}_K$, $\mathbf{I}^v$     &  the unit matrices \\ \hline
    $\mathbf{W}_{(t)}$     &  the scaling matrix in MVCAN, $w^v_{(t)} \in \mathbf{W}_{(t)}$ \\ \hline
    $\mathbf{Z}_{(t)}$     &  the scaled representation in MVCAN, $\mathbf{z}_{i(t)} \in \mathbf{Z}_{(t)}$ \\ \hline
    $\mathbf{Y}_{(t)}$     &  the robust soft labels in MVCAN, $y_{ij(t)} \in \mathbf{Y}_{(t)}$ \\ \hline
    $\mathbf{c}_{j(t)}$    &  the $j$-th cluster centroid of $\mathbf{Z}_{(t)}$ in the $t$-th iteration \\ \hline
    $\mathcal{L}_r^v$  & the representation learning objective of the $v$-th view in MVCAN \\ \hline
    $\mathcal{L}_c^v$  & the clustering objective of the $v$-th view in MVCAN \\ \hline
    $\mathcal{L}_K$  & the clustering objective of $K$-means \\ \hline
    $\mathcal{L}^v$  & the loss function to train the deep model of the $v$-th view in the $\mathcal{R}$-level iteration of MVCAN \\ \hline
    $\lambda$     &  the hyper-parameter to achieve the trade-off between $\mathcal{L}_r^v$ and $\mathcal{L}_c^v$ \\ \hline
    $T_1$     &  the iteration number in the $\mathcal{T}$-level iterative optimization \\ \hline
    $T_2$     &  the iteration number in the $\mathcal{R}$-level iterative optimization \\ \hline
    $E$       &  the number of training epochs \\ \hline
    $\arg \max_j y_{ij(t)}$ & the final cluster assignment for the $i$-th sample based on the robust soft label $y_{ij(t)} \in \mathbf{Y}_{(t)}$ \\ \hline
    \end{tabular}
\end{table*}

In the literature, existing multi-view clustering (MVC) methods could be divided into two groups, \ie, traditional methods and deep methods. In this paper, we briefly introduce traditional MVC methods and focus on deep MVC methods as follows.

Traditional MVC methods learn the representations of multi-view data for clustering by leveraging the classical machine learning technologies, such as graph MVC \cite{wen2018incomplete,zhan2018multiview,peng2019comic}, subspace MVC \cite{cao2015diversity,li2019reciprocal,zhou2019dual}, kernel MVC \cite{tzortzis2012kernel,liu2018late,liu2021one}, and matrix factorization MVC \cite{wang2018multiview,yang2020uniform}. Typically, many traditional MVC methods are limited by their representation capability of shallow models such that they usually perform clustering tasks with the limited feature forms and data scales \cite{wang2013multi,wang2016iterative,nie2016parameter,zhang2017latent}.

During past years, many deep MVC methods have been proposed \cite{xu2019adversarial,Huang0ZL020,yang2022robust}.
The mainstream techniques used in deep MVC methods belong to self-supervised learning.
Deep autoencoder \cite{baldi2012autoencoders} is the most poplar model used for deep MVC \cite{xu2019adversarial,lin2021completer,xu2022deep}, which conducts data reconstruction processes and can create basic self-supervision signals for unsupervised clustering tasks.
Some deep MVC methods are based on the aforementioned machine learning technologies, such as deep subspace MVC \cite{abavisani2018deep} and deep graph MVC \cite{fan2020one2multi}. These methods usually leverage the representation learning capability of deep autoencoder networks as well as the data mining capability of traditional technologies.
Additionally, many work investigate various deep learning technologies to develop deep MVC methods. For example, some work \cite{xu2019adversarial,li2019deep,zhou2020end,wang2021generative} combine GAN~\cite{goodfellow2014generative} with clustering objective for multiple views. Some work \cite{lin2021completer,trostenMVC} introduce the contrastive learning \cite{tian2019contrastive,lin2021completer} to learn the consistency of multi-view representations for clustering.
Establishing clustering pseudo labels is an important approach to achieve end-to-end MVC.
In research work, one of the most popular learning paradigms in deep MVC is based on the single-view clustering (SVC) method with self-training, entitled deep embedded clustering (DEC \cite{xie2016unsupervised}), \eg, \cite{xu2019adversarial,xie2020joint,fan2020one2multi,wen2020dimc,9839616,wang2021generative}, which establishes pseudo labels and then produces self-supervised signals to learn the clustering-oriented representations.
Compared with traditional methods, deep MVC methods usually have superior representation capability and scalability.

In practical multi-view scenarios, sample features extracted from some views could be noisy and might have harmful information for MVC.
Although some efforts take the view quality into account and propose weighting strategies in the fusion of multiple views \cite{nie2018multiview,ye2018multi,wen2020dimc,trostenMVC,wang2021generative}, the low-quality views could be treated as a special case of noise and both the noisy views and the low-quality views will make it difficult to uncover the effective cluster patterns for most existing MVC methods.
Because the model network parameters in fusion methods \cite{wang2016iterative,zhang2017latent,zhou2019dual,wang2019multi,liu2021one,wen2020dimc,trostenMVC,wang2021generative} usually are shared for multiple views, and many methods hope to obtain consistent clustering predictions for different views \cite{nie2016parameter,wang2018multiview,Zhan8052206,zhou2020end,tang2022deep,xu2022deep,nie2018multiview,ye2018multi,9839616},
the noisy-view drawback (NVD) is easy to make the learning of all views degenerate, and thus losing the useful consistency and complementarity among multiple views.
The work most similar to ours are the precursor deep MVC studies \cite{xu2022deep,9839616}, which construct a consistent self-supervised clustering objective for multiple views and design a weighted concatenation of multi-view features, respectively.
As NVD might cause that MVC is not necessarily better than SVC and limit the application of MVC methods, it is a meaningful but challenging goal to address the NVD and we require ongoing solutions.
In this paper, we modify the multi-view clustering objective to allow for inconsistent clustering results across different views in parameter-decoupled models, and propose an iterative approach to learn the robust self-supervised target for addressing NVD.

\section*{Appendix B: Theoretical Analysis and Proofs}
\begin{definition}
Denoting $\mathcal{D}(\mathbf{a},\mathbf{b})=\left\| \mathbf{a} - \mathbf{b} \right\|_2^2$ as squared Euclidean distance between the representations $\mathbf{a}$ and $\mathbf{b}$,
\begin{equation}
    \mathcal{S}(\mathbf{a},\mathbf{b}):= \frac{1}{1+\mathcal{D}(\mathbf{a},\mathbf{b})} \in (0, 1]
    \nonumber
\end{equation}
is defined as the representation similarity between $\mathbf{a}$ and $\mathbf{b}$.
Formally, $y_{ij}^v \in (0, 1]$ holds given Eq.~(1).
\end{definition}
Letting $\bm{\upmu}_a^v$ and $\bm{\upmu}_b^v$ denote the $a$-th and the $b$-th cluster centroids in the representation space of $\mathbf{Z}^v$, $\mathbf{z}_i^v \in \mathbf{Z}^v$, we have $\mathcal{S}(\mathbf{z}_i^v,\bm{\upmu}_a^v) > \mathcal{S}(\mathbf{z}_i^v,\bm{\upmu}_b^v) \Rightarrow y_{ia}^v > y_{ib}^v$ and $\mathcal{S}(\mathbf{z}_i^v,\bm{\upmu}_a^v) < \mathcal{S}(\mathbf{z}_i^v,\bm{\upmu}_b^v) \Rightarrow y_{ia}^v < y_{ib}^v$,
because
\begin{equation}
\begin{aligned}
y_{ij}^v &= \frac{(1+\lVert \mathbf{z}_i^v-\bm{\upmu}_j^v\rVert^2_2)^{-1}}
    {\sum_{j=1}^K(1+\lVert \mathbf{z}_i^v-\bm{\upmu}_{j}^v\rVert^2_2)^{-1}} \\
&= \frac{\mathcal{S}(\mathbf{z}_i^v,\bm{\upmu}_j^v)}{\sum_{j=1}^{K} \mathcal{S}(\mathbf{z}_i^v,\bm{\upmu}_j^v)} \\
&\propto \mathcal{S}(\mathbf{z}_i^v,\bm{\upmu}_j^v).
\end{aligned}
\nonumber
\end{equation}

\begin{definition} ($\varepsilon$, $\mathbf{z}$, $\bm{\upmu}$ - Noisy-view)
For $\forall \mathbf{z}_i^v \in \mathbf{Z}^v$, it belongs to the noisy view if $\exists \bm{\upmu}_a^v, \bm{\upmu}_b^v$, and $\varepsilon > 0$ such that $|\mathcal{D}(\mathbf{z}_i^v,\bm{\upmu}_a^v) - \mathcal{D}(\mathbf{z}_i^v,\bm{\upmu}_b^v) |< \varepsilon$, $\mathcal{S}(\mathbf{z}_i^v,\bm{\upmu}_a^v) \approx \mathcal{S}(\mathbf{z}_i^v,\bm{\upmu}_b^v)$, and $y_{ia}^v \approx y_{ib}^v$, where $\varepsilon$ is a sufficiently small value.
Otherwise, $\mathbf{z}_i^v$ is the informative view.
\end{definition}

\begin{theorem}\label{the:theorm01}
Denoting $\mathbf{\check{Y}} = \mathbf{L} \mathbf{A}$, where $\mathbf{A} \in \{0,1\}^{K\times K}$ makes $\mathbf{\check{Y}}$ maximally match the learning target $\mathbf{T}$. Then, the clustering accuracy can be calculated as $ACC = \frac{1}{N} \left(N - \frac{1}{2}\| \mathbf{\check{Y}} - \mathbf{T} \|_F^2\right) = 1 - \frac{1}{2N}\| \mathbf{\check{Y}} - \mathbf{T} \|_F^2$. In Eq.~(3), if $\mathbf{\Theta}$ is shared by multiple views and their soft labels $\{\mathbf{Y}^v\}_{v=1}^V$ have consistent learning target $\mathbf{T}$, we have
\begin{equation}\label{acca}
ACC \leq 1- \frac{1}{2N} \left (\left(\mathop{ \max_{1\leq m \leq V}}\| \mathbf{\check{Y}} - \mathcal{F}_{\mathbf{\Theta}}(\mathbf{Y}^m|\{\mathbf{Z}^v\}_{v=1}^V) \|_F^2 \right) - \| \mathbf{T} - \mathcal{F}_{\mathbf{\Theta}}(\mathbf{Y}^v|\{\mathbf{Z}^v\}_{v=1}^V) \|_F^2 \right).
\nonumber
\end{equation}
\end{theorem}
\begin{proof}
Motivated by the proof in \cite{nie2018multiview}, we first consider the following equation:
\begin{equation}
\begin{aligned}
&\| \mathbf{\check{Y}} - \mathcal{F}_{\mathbf{\Theta}}(\mathbf{Y}^v|\{\mathbf{Z}^v\}_{v=1}^V) \|_F^2 - \| \mathbf{T} - \mathcal{F}_{\mathbf{\Theta}}(\mathbf{Y}^v|\{\mathbf{Z}^v\}_{v=1}^V) \|_F^2 \\
=& \| \mathbf{\check{Y}} \|_F^2 - \| \mathbf{T} \|_F^2 + \| \mathcal{F}_{\mathbf{\Theta}}(\mathbf{Y}^v|\{\mathbf{Z}^v\}_{v=1}^V) \|_F^2 - \| \mathcal{F}_{\mathbf{\Theta}}(\mathbf{Y}^v|\{\mathbf{Z}^v\}_{v=1}^V) \|_F^2 - 
2Tr\left( \left( \mathbf{\check{Y}} - \mathbf{T} \right)^T \mathcal{F}_{\mathbf{\Theta}}(\mathbf{Y}^v|\{\mathbf{Z}^v\}_{v=1}^V) \right) \\
=& \left(N - \| \mathbf{T} \|_F^2 \right) + \left(\| \mathcal{F}_{\mathbf{\Theta}}(\mathbf{Y}^v|\{\mathbf{Z}^v\}_{v=1}^V) \|_F^2 - \| \mathcal{F}_{\mathbf{\Theta}}(\mathbf{Y}^v|\{\mathbf{Z}^v\}_{v=1}^V) \|_F^2\right) - 
2Tr\left( \left( \mathbf{\check{Y}} - \mathbf{T} \right)^T \mathcal{F}_{\mathbf{\Theta}}(\mathbf{Y}^v|\{\mathbf{Z}^v\}_{v=1}^V) \right) \\
=&\left(N - \| \mathbf{T} \|_F^2 \right) + 2Tr\left( \left(\mathbf{T} - \mathbf{\check{Y}} \right)^T \mathcal{F}_{\mathbf{\Theta}}(\mathbf{Y}^v|\{\mathbf{Z}^v\}_{v=1}^V) \right) \\
=&\left(N - \| \mathbf{T} \|_F^2 \right) + 2Tr\left( \mathcal{F}_{\mathbf{\Theta}}(\mathbf{Y}^v|\{\mathbf{Z}^v\}_{v=1}^V) \left(\mathbf{T} - \mathbf{\check{Y}} \right)^T \right),
\nonumber
\end{aligned}
\end{equation}
where the term of matrix trace could satisfy the inequality:
\begin{equation}
\begin{aligned}
&2Tr\left( \mathcal{F}_{\mathbf{\Theta}}(\mathbf{Y}^v|\{\mathbf{Z}^v\}_{v=1}^V) \left(\mathbf{T} - \mathbf{\check{Y}} \right)^T \right) \\
=& 2 \sum_{i=1}^N \mathbf{y}^v_i (\mathbf{t}_i^T - \mathbf{\check{y}}_i^T) \\
=& 2 \sum_{i=1}^N \mathbf{y}^v_i \mathbf{t}_i^T - \mathbf{y}^v_i \mathbf{\check{y}}_i^T\\
=& 2 \sum_{\mathbf{t}_i \neq \mathbf{\check{y}}_i} \left( \sum_{k=1}^K y^v_{ik}t_{ik}  \right) - \left( \sum_{k=1}^K y^v_{ik}\check{y}_{ik}  \right)\\
\leq& 2 \sum_{\mathbf{t}_i \neq \mathbf{\check{y}}_i} \left | \left( \sum_{k=1}^K y^v_{ik}t_{ik}  \right) - \left( \sum_{k=1}^K y^v_{ik}\check{y}_{ik}  \right) \right|\\
\leq& 2 \sum_{\mathbf{t}_i \neq \mathbf{\check{y}}_i} |1 - 0|\\
=& 2 \sum_{\mathbf{t}_i \neq \mathbf{\check{y}}_i} 1\\
=& \| \mathbf{\check{Y}} - \mathbf{T} \|_F^2.
\nonumber
\end{aligned}
\end{equation}
Then, we have
\begin{equation}
\begin{aligned}
\| \mathbf{\check{Y}} - \mathbf{T} \|_F^2 
&\geq \| \mathbf{\check{Y}} - \mathcal{F}_{\mathbf{\Theta}}(\mathbf{Y}^v|\{\mathbf{Z}^v\}_{v=1}^V) \|_F^2 - \| \mathbf{T} - \mathcal{F}_{\mathbf{\Theta}}(\mathbf{Y}^v|\{\mathbf{Z}^v\}_{v=1}^V) \|_F^2 
- N + \| \mathbf{T} \|_F^2 \\
&\geq  \left(\mathop{ \max_{1\leq m \leq V}}\| \mathbf{\check{Y}} - \mathcal{F}_{\mathbf{\Theta}}(\mathbf{Y}^m|\{\mathbf{Z}^v\}_{v=1}^V) \|_F^2 \right) - \| \mathbf{T} - \mathcal{F}_{\mathbf{\Theta}}(\mathbf{Y}^v|\{\mathbf{Z}^v\}_{v=1}^V) \|_F^2 - N + \left( \mathop{ \max_{\mathbf{T}}} \| \mathbf{T} \|_F^2 \right)\\
&\geq  \left(\mathop{ \max_{1\leq m \leq V}}\| \mathbf{\check{Y}} - \mathcal{F}_{\mathbf{\Theta}}(\mathbf{Y}^m|\{\mathbf{Z}^v\}_{v=1}^V) \|_F^2 \right) - \| \mathbf{T} - \mathcal{F}_{\mathbf{\Theta}}(\mathbf{Y}^v|\{\mathbf{Z}^v\}_{v=1}^V) \|_F^2.\\
\nonumber
\end{aligned}
\end{equation}
Furthermore, the clustering accuracy becomes
\begin{equation}
\begin{aligned}
ACC &= \frac{1}{N} \left(N - \frac{1}{2}\| \mathbf{\check{Y}} - \mathbf{T} \|_F^2\right) = 1 - \frac{1}{2N}\| \mathbf{\check{Y}} - \mathbf{T} \|_F^2 \\
&\leq 1- \frac{1}{2N} \left(\left(\mathop{ \max_{1\leq m \leq V}}\| \mathbf{\check{Y}} - \mathcal{F}_{\mathbf{\Theta}}(\mathbf{Y}^m|\{\mathbf{Z}^v\}_{v=1}^V) \|_F^2 \right) - \| \mathbf{T} - \mathcal{F}_{\mathbf{\Theta}}(\mathbf{Y}^v|\{\mathbf{Z}^v\}_{v=1}^V) \|_F^2 \right),
\nonumber
\end{aligned}
\end{equation}
which completes this proof.
\end{proof}

\textbf{A specific example for illustrating Theorem \ref{the:theorm01}}.
Given the DEC framework as a specific example, we denote $\mathbf{Z}^m$ and $\mathbf{Z}^n$, respectively, as the representations of the informative view and the noisy/low-quality view. The shared parameters are the $K$ cluster centroids $\mathbf{U} = [\bm{\upmu}_1, \bm{\upmu}_2,\dots, \bm{\upmu}_K]$, which are the common cluster centroids for the all views’ representations. Based on $\mathbf{Z}^m$ and $\mathbf{U}$ as formulated in Eq.~(1), the informative view's representation $\mathbf{Z}^m$ has clear cluster structures and thus its clustering loss $\mathcal{L}_c^m$ is easy to minimize. However, based on $\mathbf{Z}^n$ and the same $\mathbf{U}$ as formulated in Eq.~(1), the noisy/low-quality view's representation $\mathbf{Z}^n$ has unclear cluster structures and thus its clustering loss $\mathcal{L}_c^n$ is hard to minimize. On the one hand, since the training objectives of multiple views have the same learning target $\mathbf{T}$, the large $\mathcal{L}_c^n$ will dominate the small $\mathcal{L}_c^m$ if their optimizations are not decoupled. On the other hand, the minimization of the noisy view’s clustering loss $\mathcal{L}_c^n$ will change the shared $K$ cluster centroids $\mathbf{U}$, causing that $\mathbf{U}$ is not suitable for the informative view’s representations $\mathbf{Z}^m$. As a result, the NVD of low-quality or noisy views destroys the method's robustness.

\begin{theorem}\label{the:kkkk}
Denoting $\mathcal{L}_{K}$ as the $K$-means objective, $\mathcal{L}_{K}(\mathbf{Z}_{(t)})$ is equivalent to 
punishing different scaling factors on $\{\mathcal{L}_{K}(\mathbf{Z}^v)\}_{v=1}^V$ under the consistency constraint of multiple views' cluster centroids.
\end{theorem}
\begin{proof}
For clarity, we may omit the iteration symbol $(t)$ of $\mathbf{Z}_{(t)}$, $\mathbf{W}_{(t)}$, and $\mathbf{c}_{j(t)}$ in next proofs. Then, the $K$-means objective on $\mathbf{Z}$ (\ie, $\mathbf{Z}_{(t)}$) can be formulated as:
\begin{equation}\label{eq:com}
\begin{aligned}
    \mathcal{L}_{K}(\mathbf{Z}) = &\mathop{\min_{\mathbf{c}_j}}\sum_{i=1}^N\sum_{j=1}^K \left\| {\mathbf{z}_{i}-\mathbf{c}_j} \right\|_2^2\\
    =&\mathop{\min_{\bm{\upmu}^v_j}}\sum_{i=1}^N\sum_{j=1}^K \sum_{v=1}^V \left\| w^v\mathbf{z}_i^v-w^v\bm{\upmu}_j^v \right\|_2 ^2, s.t.~ \mathbf{c}_j=\begin{bmatrix} \bm{\upmu}^1_j & \bm{\upmu}^2_j & \dots & \bm{\upmu}^V_j \end{bmatrix} \mathbf{W}\\
    =&\sum_{v=1}^V (w^v)^2 \mathop{\min_{\bm{\upmu}^v_j}}\sum_{i=1}^N\sum_{j=1}^K \left\|\mathbf{z}_i^v-\bm{\upmu}_j^v \right\|_2 ^2, s.t.~ \mathbf{c}_j=\begin{bmatrix} \bm{\upmu}^1_j & \bm{\upmu}^2_j & \dots & \bm{\upmu}^V_j \end{bmatrix} \mathbf{W}\\
    =&\sum_{v=1}^V (w^v)^2 \mathcal{L}_{K}(\mathbf{Z}^v), s.t.~ \mathbf{c}_j=\begin{bmatrix} \bm{\upmu}^1_j & \bm{\upmu}^2_j & \dots & \bm{\upmu}^V_j \end{bmatrix} \mathbf{W},
\nonumber
\end{aligned}
\end{equation}
where $\mathbf{c}_j$ and $\bm{\upmu}^v_j$ denote the $j$-th cluster centroid of $\mathbf{Z}$ and $\mathbf{Z}^v$, respectively.
First, we can find that $(w^v)^2\mathcal{L}_{K}(\mathbf{Z}^v)$ punishes different scaling factors on different views.
Second, if the $K$-means objective is conducted on $\mathbf{Z}^v$ individually, the obtained cluster centroids for the same cluster of two views might have different samples, \eg, $\bm{\upmu}^a_j = \sum_{i\in \Omega^a} \mathbf{z}^a_i / |\Omega^a|$ and $\bm{\upmu}^b_j = \sum_{i\in \Omega^b} \mathbf{z}^b_i / |\Omega^b|$ where $\Omega^a \neq \Omega^b$.
However, the constraint condition $\mathbf{c}_j=\begin{bmatrix} \bm{\upmu}^1_j & \bm{\upmu}^2_j & \dots & \bm{\upmu}^V_j \end{bmatrix} \mathbf{W}$ makes $\Omega^a = \Omega^b = \Omega^c$ such that $w^a\bm{\upmu}^a_j = \sum_{i\in \Omega^a} w^a\mathbf{z}^a_i /|\Omega^a|$, $w^b\bm{\upmu}^b_j = \sum_{i\in \Omega^b} w^b\mathbf{z}^b_i /|\Omega^b|$, and $\mathbf{c}_j = \sum_{i\in \Omega^c} \mathbf{z}_i /|\Omega^c|$ for $\forall a,b \in \{1,2,\dots,V\}$. Therefore, the constraint condition could be treated as the consistency constraint of multiple views' cluster centroids, which guarantees that the cluster centroids $\mathbf{c}_j, \bm{\upmu}^1_j, \bm{\upmu}^2_j, \dots, \bm{\upmu}^V_j$ are obtained from the consistent samples among multiple views.
\end{proof}

\begin{theorem}\label{the:tt2}
(\textbf{Consistency})
If a sample representation is informative in multiple views and has the same cluster assignments in these views, its cluster assignment in $\mathbf{Y}_{(t)}$ is the same as that in these views.
\end{theorem}

\begin{proof}
For clarity, we take two views as example, $\mathbf{z}_i = \begin{bmatrix} w^1\mathbf{z}_i^{1} &w^2\mathbf{z}_i^{2} \end{bmatrix} \in \mathbf{Z}$ and we let $\mathbf{c}_j = \begin{bmatrix} w^1\bm{\upmu}_j^{1} &w^2\bm{\upmu}_j^{2} \end{bmatrix}$ denote the ideal cluster centroid in the scaled representation space of $\mathbf{Z}$, where $\bm{\upmu}_j^{1}$ and $\bm{\upmu}_j^{2}$ are the cluster centroids in the 1-st view and the 2-nd view of the $j$-th cluster, respectively. For any two clusters, \eg, $\mathbf{c}_1$ and $\mathbf{c}_2$, we have
\begin{equation}
\begin{aligned}
\mathcal{D}(\mathbf{z}_i,\mathbf{c}_1)&=\left\| \begin{bmatrix} w^1\mathbf{z}_i^{1} &w^2\mathbf{z}_i^{2} \end{bmatrix} - \begin{bmatrix} w^1\bm{\upmu}_1^{1} & w^2\bm{\upmu}_1^{2} \end{bmatrix} \right\|_2^2 \\
&= \left\| w^1\mathbf{z}_i^{1} - w^1\bm{\upmu}_1^{1} \right\|_2^2 + \left\| w^2\mathbf{z}_i^{2} - w^2\bm{\upmu}_1^{2} \right\|_2^2\\
&= (w^1)^2\left\| \mathbf{z}_i^{1} - \bm{\upmu}_1^{1} \right\|_2^2 + (w^2)^2\left\| \mathbf{z}_i^{2} - \bm{\upmu}_1^{2} \right\|_2^2\\
&= (w^1)^2 \mathcal{D}(\mathbf{z}_i^{1},\bm{\upmu}_1^{1}) + (w^2)^2 \mathcal{D}(\mathbf{z}_i^{2},\bm{\upmu}_1^{2}).
\nonumber
\end{aligned}
\end{equation}
Similarly, $\mathcal{D}(\mathbf{z}_i,\mathbf{c}_2)=(w^1)^2 \mathcal{D}(\mathbf{z}_i^{1},\bm{\upmu}_2^{1}) + (w^2)^2 \mathcal{D}(\mathbf{z}_i^{2},\bm{\upmu}_2^{2})$. Furthermore, we have
\begin{equation}
\begin{aligned}
\mathcal{D}(\mathbf{z}_i,\mathbf{c}_1) - \mathcal{D}(\mathbf{z}_i,\mathbf{c}_2) &= (w^1)^2 \mathcal{D}(\mathbf{z}_i^{1},\bm{\upmu}_1^{1}) + (w^2)^2 \mathcal{D}(\mathbf{z}_i^{2},\bm{\upmu}_1^{2})
                                                                                - (w^1)^2 \mathcal{D}(\mathbf{z}_i^{1},\bm{\upmu}_2^{1}) - (w^2)^2 \mathcal{D}(\mathbf{z}_i^{2},\bm{\upmu}_2^{2})\\
              &= (w^1)^2 \left(\mathcal{D}(\mathbf{z}_i^{1},\bm{\upmu}_1^{1}) - \mathcal{D}(\mathbf{z}_i^{1},\bm{\upmu}_2^{1})\right) + (w^2)^2 \left(\mathcal{D}(\mathbf{z}_i^{2},\bm{\upmu}_1^{2}) - \mathcal{D}(\mathbf{z}_i^{2},\bm{\upmu}_2^{2})\right).
\nonumber
\end{aligned}
\end{equation}
If $y_{i1}^1 > y_{i2}^1$ and $y_{i1}^2 > y_{i2}^2$, the sample is informative in both views and has the same cluster assignment, \ie, $\mathcal{D}(\mathbf{z}_i^{1},\bm{\upmu}_1^{1}) < \mathcal{D}(\mathbf{z}_i^{1},\bm{\upmu}_2^{1})$ and $\mathcal{D}(\mathbf{z}_i^{2},\bm{\upmu}_1^{2}) < \mathcal{D}(\mathbf{z}_i^{2},\bm{\upmu}_2^{2})$. Then, the following inequality holds:
\begin{equation}
\begin{aligned}
\mathcal{D}(\mathbf{z}_i,\mathbf{c}_1) - \mathcal{D}(\mathbf{z}_i,\mathbf{c}_2) = (w^1)^2 \left(\mathcal{D}(\mathbf{z}_i^{1},\bm{\upmu}_1^{1}) - \mathcal{D}(\mathbf{z}_i^{1},\bm{\upmu}_2^{1})\right) + (w^2)^2 \left(\mathcal{D}(\mathbf{z}_i^{2},\bm{\upmu}_1^{2}) - \mathcal{D}(\mathbf{z}_i^{2},\bm{\upmu}_2^{2})\right) < 0,
\nonumber
\end{aligned}
\end{equation}
which indicates that $y_{i1} > y_{i2}$.
Similarly, if $y_{i1}^1 < y_{i2}^1$ and $y_{i1}^2 < y_{i2}^2$, we have $\mathcal{D}(\mathbf{z}_i^{1},\bm{\upmu}_1^{1}) > \mathcal{D}(\mathbf{z}_i^{1},\bm{\upmu}_2^{1})$, $\mathcal{D}(\mathbf{z}_i^{2},\bm{\upmu}_1^{2}) > \mathcal{D}(\mathbf{z}_i^{2},\bm{\upmu}_2^{2})$, and $\mathcal{D}(\mathbf{z}_i,\mathbf{c}_1) - \mathcal{D}(\mathbf{z}_i,\mathbf{c}_2) > 0$ which indicates that $y_{i1} < y_{i2}$.
As a result, the sample has the same cluster assignment in $\mathbf{Y}_{(t)}$ as that in the informative views, \ie, $\mathbf{Y}^1$ and $\mathbf{Y}^2$.
\end{proof}

\begin{theorem}\label{the:tt3}
(\textbf{Complementarity})
If a sample representation is informative in multiple views where it has different cluster assignments, we have two cases according to the differences of similarity among the clusters.\\
Case 1: if the differences of similarity among the clusters are equal, its cluster assignment in $\mathbf{Y}_{(t)}$ is the same as that in the informative view with the largest scaling factor.\\
Case 2: if the differences of similarity among the clusters are not equal, its cluster assignment in $\mathbf{Y}_{(t)}$ is more likely to be the same as that in the informative view with the largest scaling factor.
\end{theorem}

\begin{proof}
Letting $\Delta^{v} = |\mathcal{D}(\mathbf{z}_i^{v},\bm{\upmu}_1^{v}) - \mathcal{D}(\mathbf{z}_i^{v},\bm{\upmu}_2^{v})|$, and $\Gamma^v = |\mathcal{S}(\mathbf{z}_i^{v},\bm{\upmu}_1^{v}) - \mathcal{S}(\mathbf{z}_i^{v},\bm{\upmu}_2^{v})|$ denote the difference of similarity, we have $\Delta^{1}=\Delta^{2} \Rightarrow |\mathcal{S}(\mathbf{z}_i^{1},\bm{\upmu}_1^{1}) - \mathcal{S}(\mathbf{z}_i^{1},\bm{\upmu}_2^{1})| = |\mathcal{S}(\mathbf{z}_i^{2},\bm{\upmu}_1^{2}) - \mathcal{S}(\mathbf{z}_i^{2},\bm{\upmu}_2^{2})|$.

\textbf{Case 1}: if $(\mathcal{D}(\mathbf{z}_i^{1},\bm{\upmu}_1^{1}) - \mathcal{D}(\mathbf{z}_i^{1},\bm{\upmu}_2^{1}))(\mathcal{D}(\mathbf{z}_i^{2},\bm{\upmu}_1^{2}) - \mathcal{D}(\mathbf{z}_i^{2},\bm{\upmu}_2^{2})) < 0$ and $\Delta^{1}=\Delta^{2}$, the sample is informative in both views but it has different cluster assignments (\ie, $\arg \max_j {y}_{ij}^1 \neq \arg \max_j {y}_{ij}^2$), and the differences of similarity are equal $\Gamma^1 = \Gamma^2$. In this case, we obtain a boundary condition:
\begin{equation}
\mathcal{D}(\mathbf{z}_i^{1},\bm{\upmu}_1^{1}) - \mathcal{D}(\mathbf{z}_i^{1},\bm{\upmu}_2^{1}) = \mathcal{D}(\mathbf{z}_i^{2},\bm{\upmu}_2^{2}) - \mathcal{D}(\mathbf{z}_i^{2},\bm{\upmu}_1^{2}),
\nonumber
\end{equation}
and thus $\mathcal{D}(\mathbf{z}_i,\mathbf{c}_1) - \mathcal{D}(\mathbf{z}_i,\mathbf{c}_2)$ in Theorem \ref{the:tt2} becomes
\begin{equation}
\begin{aligned}
\mathcal{D}(\mathbf{z}_i,\mathbf{c}_1) - \mathcal{D}(\mathbf{z}_i,\mathbf{c}_2) &= (w^1)^2 \left(\mathcal{D}(\mathbf{z}_i^{1},\bm{\upmu}_1^{1}) - \mathcal{D}(\mathbf{z}_i^{1},\bm{\upmu}_2^{1})\right) + (w^2)^2 \left(\mathcal{D}(\mathbf{z}_i^{2},\bm{\upmu}_1^{2}) - \mathcal{D}(\mathbf{z}_i^{2},\bm{\upmu}_2^{2})\right)\\
              &= ((w^1)^2 - (w^2)^2)(\mathcal{D}(\mathbf{z}_i^{1},\bm{\upmu}_1^{1}) - \mathcal{D}(\mathbf{z}_i^{1},\bm{\upmu}_2^{1})).
\nonumber
\end{aligned}
\end{equation}
If $w^1 > w^2$, we further have
\begin{equation}
\begin{aligned}
y_{i1}^1 > y_{i2}^1 &\Rightarrow \mathcal{D}(\mathbf{z}_i^{1},\bm{\upmu}_1^{1}) < \mathcal{D}(\mathbf{z}_i^{1},\bm{\upmu}_2^{1}) \\
&\Rightarrow \mathcal{D}(\mathbf{z}_i,\mathbf{c}_1) < \mathcal{D}(\mathbf{z}_i,\mathbf{c}_2) \\
&\Rightarrow y_{i1} > y_{i2}
\nonumber
\end{aligned}
\end{equation}
and $y_{i1}^1 < y_{i2}^1 \Rightarrow y_{i1} < y_{i2}$.

If $w^1 < w^2$, we similarly have
\begin{equation}
\begin{aligned}
y_{i1}^2 > y_{i2}^2 &\Rightarrow \mathcal{D}(\mathbf{z}_i^{2},\bm{\upmu}_1^{2}) < \mathcal{D}(\mathbf{z}_i^{2},\bm{\upmu}_2^{2}) \\
&\Rightarrow' \mathcal{D}(\mathbf{z}_i^{1},\bm{\upmu}_1^{1}) > \mathcal{D}(\mathbf{z}_i^{1},\bm{\upmu}_2^{1})\\
&\Rightarrow \mathcal{D}(\mathbf{z}_i,\mathbf{c}_1) < \mathcal{D}(\mathbf{z}_i,\mathbf{c}_2)\\
&\Rightarrow y_{i1} > y_{i2}
\nonumber
\end{aligned}
\end{equation}
and $y_{i1}^2 < y_{i2}^2 \Rightarrow y_{i1} < y_{i2}$. $\Rightarrow'$ holds because $(\mathcal{D}(\mathbf{z}_i^{1},\bm{\upmu}_1^{1}) - \mathcal{D}(\mathbf{z}_i^{1},\bm{\upmu}_2^{1}))(\mathcal{D}(\mathbf{z}_i^{2},\bm{\upmu}_1^{2}) - \mathcal{D}(\mathbf{z}_i^{2},\bm{\upmu}_2^{2})) < 0$.
In conclusion, the sample has the same cluster assignment in $\mathbf{Y}_{(t)}$ as that the view with the largest scaling factor.

\textbf{Case 2}: if $(\mathcal{D}(\mathbf{z}_i^{1},\bm{\upmu}_1^{1}) - \mathcal{D}(\mathbf{z}_i^{1},\bm{\upmu}_2^{1}))(\mathcal{D}(\mathbf{z}_i^{2},\bm{\upmu}_1^{2}) - \mathcal{D}(\mathbf{z}_i^{2},\bm{\upmu}_2^{2})) < 0$ and $\Delta^{1}\neq\Delta^{2}$, the sample is informative in both views but it has different cluster assignments (\ie, $\arg \max_j {y}_{ij}^1 \neq \arg \max_j {y}_{ij}^2$), and the differences of similarity are not equal $\Gamma^1 \neq \Gamma^2$.

In this case, if $y_{i1}^1 < y_{i2}^1$ and $y_{i1}^2 > y_{i2}^2$, \ie, $\mathcal{D}(\mathbf{z}_i^{1},\bm{\upmu}_1^{1}) > \mathcal{D}(\mathbf{z}_i^{1},\bm{\upmu}_2^{1})$ and $\mathcal{D}(\mathbf{z}_i^{2},\bm{\upmu}_1^{2}) < \mathcal{D}(\mathbf{z}_i^{2},\bm{\upmu}_2^{2})$, there is a threshold $\delta$ as $\mathcal{D}(\mathbf{z}_i^{1},\bm{\upmu}_1^{1}) \neq \mathcal{D}(\mathbf{z}_i^{1},\bm{\upmu}_2^{1})$ and $\mathcal{D}(\mathbf{z}_i^{2},\bm{\upmu}_1^{2}) \neq \mathcal{D}(\mathbf{z}_i^{2},\bm{\upmu}_2^{2})$:
\begin{equation}
\begin{aligned}
\delta = \frac{\Delta^{2}}{\Delta^{1}} = \frac{\mathcal{D}(\mathbf{z}_i^{2},\bm{\upmu}_2^{2}) - \mathcal{D}(\mathbf{z}_i^{2},\bm{\upmu}_1^{2})}{\mathcal{D}(\mathbf{z}_i^{1},\bm{\upmu}_1^{1}) - \mathcal{D}(\mathbf{z}_i^{1},\bm{\upmu}_2^{1})}.
\nonumber
\end{aligned}
\end{equation}
When $w^1$ is large so that $\frac{(w^1)^2}{(w^2)^2} > \delta$, we have
\begin{equation}
\begin{aligned}
&\frac{(w^1)^2}{(w^2)^2} > \frac{\mathcal{D}(\mathbf{z}_i^{2},\bm{\upmu}_2^{2}) - \mathcal{D}(\mathbf{z}_i^{2},\bm{\upmu}_1^{2})}{\mathcal{D}(\mathbf{z}_i^{1},\bm{\upmu}_1^{1}) - \mathcal{D}(\mathbf{z}_i^{1},\bm{\upmu}_2^{1})}\\
\Rightarrow &(w^1)^2 \left(\mathcal{D}(\mathbf{z}_i^{1},\bm{\upmu}_1^{1}) - \mathcal{D}(\mathbf{z}_i^{1},\bm{\upmu}_2^{1})\right) + (w^2)^2 \left(\mathcal{D}(\mathbf{z}_i^{2},\bm{\upmu}_1^{2}) - \mathcal{D}(\mathbf{z}_i^{2},\bm{\upmu}_2^{2})\right) > 0\\
\Rightarrow &\mathcal{D}(\mathbf{z}_i,\mathbf{c}_1) - \mathcal{D}(\mathbf{z}_i,\mathbf{c}_2) > 0\\
\Rightarrow & y_{i1} < y_{i2},
\nonumber
\end{aligned}
\end{equation}
which indicates $y_{i1}^1 < y_{i2}^1 \Rightarrow y_{i1} < y_{i2}$, \ie, the sample has the same cluster assignment in $\mathbf{Y}_{(t)}$ as that in the view with the scaling factor $w^1$.\\
When $w^2$ is large so that $\frac{(w^1)^2}{(w^2)^2} < \delta$, we have
\begin{equation}
\begin{aligned}
&\frac{(w^1)^2}{(w^2)^2} < \frac{\mathcal{D}(\mathbf{z}_i^{2},\bm{\upmu}_2^{2}) - \mathcal{D}(\mathbf{z}_i^{2},\bm{\upmu}_1^{2})}{\mathcal{D}(\mathbf{z}_i^{1},\bm{\upmu}_1^{1}) - \mathcal{D}(\mathbf{z}_i^{1},\bm{\upmu}_2^{1})}\\
\Rightarrow &(w^1)^2 \left(\mathcal{D}(\mathbf{z}_i^{1},\bm{\upmu}_1^{1}) - \mathcal{D}(\mathbf{z}_i^{1},\bm{\upmu}_2^{1})\right) + (w^2)^2 \left(\mathcal{D}(\mathbf{z}_i^{2},\bm{\upmu}_1^{2}) - \mathcal{D}(\mathbf{z}_i^{2},\bm{\upmu}_2^{2})\right) < 0\\
\Rightarrow &\mathcal{D}(\mathbf{z}_i,\mathbf{c}_1) - \mathcal{D}(\mathbf{z}_i,\mathbf{c}_2) < 0\\
\Rightarrow & y_{i1} > y_{i2},
\nonumber
\end{aligned}
\end{equation}
which indicates $y_{i1}^2 > y_{i2}^2 \Rightarrow y_{i1} > y_{i2}$, \ie, the sample has the same cluster assignment in $\mathbf{Y}_{(t)}$ as that in the view with the scaling factor $w^2$.\\
Similarly, if $y_{i1}^1 > y_{i2}^1$ and $y_{i1}^2 < y_{i2}^2$, \ie, $\mathcal{D}(\mathbf{z}_i^{1},\bm{\upmu}_1^{1}) < \mathcal{D}(\mathbf{z}_i^{1},\bm{\upmu}_2^{1})$ and $\mathcal{D}(\mathbf{z}_i^{2},\bm{\upmu}_1^{2}) > \mathcal{D}(\mathbf{z}_i^{2},\bm{\upmu}_2^{2})$, we have the same conclusions.
Therefore, if we put a large scaling factor on the $v$-th view, the sample is more likely to be have the same cluster assignment as that in the $v$-th view.
\end{proof}

\begin{theorem}\label{the:tt4}
(\textbf{Complementarity $\&$ Noise robustness})\\
Case 1: if a sample representation is informative in some views and is noisy in other views, its cluster assignment in $\mathbf{Y}_{(t)}$ is the same as that in the informative views.\\
Case 2: if a sample representation is noisy in all views, its cluster assignment in $\mathbf{Y}_{(t)}$ is the same as the common cluster assignments existing in these views.
\end{theorem}

\begin{proof}
\textbf{Case 1}: if a sample is informative in some views and is noisy in other views, we can treat all noisy views as one noisy view and all informative views as one informative view, \eg,

$\mathbf{z}_i^1 =\begin{bmatrix} w^a\mathbf{z}_i^{a} &w^b\mathbf{z}_i^{b} \end{bmatrix}$ denotes the noisy view and $\mathbf{z}_i^2 =\begin{bmatrix} w^c\mathbf{z}_i^{c} &w^d\mathbf{z}_i^{d} \end{bmatrix}$ denotes the informative view, where $\mathbf{z}_i^{a}$ and $\mathbf{z}_i^{b}$ are noisy while $\mathbf{z}_i^{c}$ and $\mathbf{z}_i^{d}$ are informative.
Letting $w^v=1$ in Theorem \ref{the:tt2}, $\mathcal{D}(\mathbf{z}_i,\mathbf{c}_1) - \mathcal{D}(\mathbf{z}_i,\mathbf{c}_2)$ becomes
\begin{equation}
\begin{aligned}
\mathcal{D}(\mathbf{z}_i,\mathbf{c}_1) - \mathcal{D}(\mathbf{z}_i,\mathbf{c}_2) &= \mathcal{D}(\mathbf{z}_i^{1},\bm{\upmu}_1^{1}) + \mathcal{D}(\mathbf{z}_i^{2},\bm{\upmu}_1^{2}) - \mathcal{D}(\mathbf{z}_i^{1},\bm{\upmu}_2^{1}) - \mathcal{D}(\mathbf{z}_i^{2},\bm{\upmu}_2^{2})\\
              &= \left(\mathcal{D}(\mathbf{z}_i^{1},\bm{\upmu}_1^{1}) - \mathcal{D}(\mathbf{z}_i^{1},\bm{\upmu}_2^{1})\right) + \left(\mathcal{D}(\mathbf{z}_i^{2},\bm{\upmu}_1^{2}) - \mathcal{D}(\mathbf{z}_i^{2},\bm{\upmu}_2^{2})\right).
\nonumber
\end{aligned}
\end{equation}
If $|\mathcal{D}(\mathbf{z}_i^{1},\bm{\upmu}_1^{1}) - \mathcal{D}(\mathbf{z}_i^{1},\bm{\upmu}_2^{1})| < \varepsilon$ and $\mathcal{D}(\mathbf{z}_i^{2},\bm{\upmu}_1^{2}) \neq \mathcal{D}(\mathbf{z}_i^{2},\bm{\upmu}_2^{2})$, \ie, the sample is noisy in the 1-st view but informative in the 2-nd view. As $\varepsilon$ is a sufficiently small value, we have
\begin{equation}
\begin{aligned}
\mathcal{D}(\mathbf{z}_i,\mathbf{c}_1) - \mathcal{D}(\mathbf{z}_i,\mathbf{c}_2) \approx \mathcal{D}(\mathbf{z}_i^{2},\bm{\upmu}_1^{2}) - \mathcal{D}(\mathbf{z}_i^{2},\bm{\upmu}_2^{2}).
\nonumber
\end{aligned}
\end{equation}
According to the definitions, we obtain the conclusion:
\begin{equation}
\begin{aligned}
y_{i1}^2 > y_{i2}^2 &\Rightarrow \mathcal{D}(\mathbf{z}_i^{2},\bm{\upmu}_1^{2}) < \mathcal{D}(\mathbf{z}_i^{2},\bm{\upmu}_2^{2}) \\
&\Rightarrow \mathcal{D}(\mathbf{z}_i,\mathbf{c}_1) < \mathcal{D}(\mathbf{z}_i,\mathbf{c}_2) \\
&\Rightarrow \mathcal{S}(\mathbf{z}_i,\mathbf{c}_1) > \mathcal{S}(\mathbf{z}_i,\mathbf{c}_2) \\
&\Rightarrow y_{i1} > y_{i2}
\nonumber
\end{aligned}
\end{equation}
and $y_{i1}^2 < y_{i2}^2 \Rightarrow y_{i1} < y_{i2}$, which indicates that the sample has the same cluster assignment in $\mathbf{Y}_{(t)}$ as that in $\mathbf{Y}^{2}$ of the informative views.

Similarly, if $\mathcal{D}(\mathbf{z}_i^{1},\bm{\upmu}_1^{1}) \neq \mathcal{D}(\mathbf{z}_i^{1},\bm{\upmu}_2^{1})$ and $|\mathcal{D}(\mathbf{z}_i^{2},\bm{\upmu}_1^{2}) - \mathcal{D}(\mathbf{z}_i^{2},\bm{\upmu}_2^{2})| < \varepsilon$, the sample has the same cluster assignment in $\mathbf{Y}_{(t)}$ as that in $\mathbf{Y}^{1}$ of the informative views.

\textbf{Case 2}: For any three clusters, \eg, $\bm{\upmu}_1^{1}$, $\bm{\upmu}_2^{1}$, and $\bm{\upmu}_3^{1}$ in the 1-st view, $\bm{\upmu}_1^{2}$, $\bm{\upmu}_2^{2}$, and $\bm{\upmu}_3^{2}$ in the 2-nd view,
if a sample is noisy in all views, there is no harm in supposing that $\mathcal{D}(\mathbf{z}_i^{1},\bm{\upmu}_1^{1}) \approx \mathcal{D}(\mathbf{z}_i^{1},\bm{\upmu}_2^{1}) < \mathcal{D}(\mathbf{z}_i^{1},\bm{\upmu}_3^{1})$ and
$\mathcal{D}(\mathbf{z}_i^{2},\bm{\upmu}_1^{2}) > \mathcal{D}(\mathbf{z}_i^{2},\bm{\upmu}_2^{2}) \approx \mathcal{D}(\mathbf{z}_i^{2},\bm{\upmu}_3^{2})$, \ie, $y_{i1}^1 \approx y_{i2}^1 > y_{i3}^1$ and $y_{i1}^2 < y_{i2}^2 \approx y_{i3}^2$.
In this situation, the sample is noisy in two views and the common cluster assignments are $y_{i2}^1$ and $y_{i2}^2$ for these two views.

First, we can consider this situation $\mathcal{D}(\mathbf{z}_i^{1},\bm{\upmu}_1^{1}) \approx \mathcal{D}(\mathbf{z}_i^{1},\bm{\upmu}_2^{1})$ and $\mathcal{D}(\mathbf{z}_i^{2},\bm{\upmu}_1^{2}) > \mathcal{D}(\mathbf{z}_i^{2},\bm{\upmu}_2^{2})$.
According to Theorem \ref{the:tt2}, we have
\begin{equation}
\begin{aligned}
\mathcal{D}(\mathbf{z}_i,\mathbf{c}_1) - \mathcal{D}(\mathbf{z}_i,\mathbf{c}_2) 
&= (w^1)^2 \left(\mathcal{D}(\mathbf{z}_i^{1},\bm{\upmu}_1^{1}) - \mathcal{D}(\mathbf{z}_i^{1},\bm{\upmu}_2^{1})\right) + (w^2)^2 \left(\mathcal{D}(\mathbf{z}_i^{2},\bm{\upmu}_1^{2}) - \mathcal{D}(\mathbf{z}_i^{2},\bm{\upmu}_2^{2})\right) \\
&\approx (w^2)^2 \left(\mathcal{D}(\mathbf{z}_i^{2},\bm{\upmu}_1^{2}) - \mathcal{D}(\mathbf{z}_i^{2},\bm{\upmu}_2^{2})\right) > 0,
\nonumber
\end{aligned}
\end{equation}
which indicates that $y_{i1} < y_{i2}$.

Then, we can consider this situation $\mathcal{D}(\mathbf{z}_i^{1},\bm{\upmu}_2^{1}) < \mathcal{D}(\mathbf{z}_i^{1},\bm{\upmu}_3^{1})$ and $\mathcal{D}(\mathbf{z}_i^{2},\bm{\upmu}_2^{2}) \approx \mathcal{D}(\mathbf{z}_i^{2},\bm{\upmu}_3^{2})$.
According to Theorem \ref{the:tt2}, we have
\begin{equation}
\begin{aligned}
\mathcal{D}(\mathbf{z}_i,\mathbf{c}_2) - \mathcal{D}(\mathbf{z}_i,\mathbf{c}_3) 
&= (w^1)^2 \left(\mathcal{D}(\mathbf{z}_i^{1},\bm{\upmu}_2^{1}) - \mathcal{D}(\mathbf{z}_i^{1},\bm{\upmu}_3^{1})\right) + (w^2)^2 \left(\mathcal{D}(\mathbf{z}_i^{2},\bm{\upmu}_2^{2}) - \mathcal{D}(\mathbf{z}_i^{2},\bm{\upmu}_3^{2})\right) \\
&\approx (w^1)^2 \left(\mathcal{D}(\mathbf{z}_i^{1},\bm{\upmu}_2^{1}) - \mathcal{D}(\mathbf{z}_i^{1},\bm{\upmu}_3^{1})\right) < 0,
\nonumber
\end{aligned}
\end{equation}
which indicates that $y_{i2} > y_{i3}$.

Therefore, $y_{i1}^1 \approx y_{i2}^1 > y_{i3}^1, y_{i1}^2 < y_{i2}^2 \approx y_{i3}^2 \Rightarrow y_{i2} > y_{i1}, y_{i2} > y_{i3}$. Although the sample is noisy in the two views,
the scaled representation $\mathbf{z}_i$ is informative and its cluster assignment in $\mathbf{Y}_{(t)}$ is consistent with $y_{i2}^1 \in \mathbf{Y}^{1}$ and $y_{i2}^2 \in \mathbf{Y}^{2}$. Indeed, $y_{i2}^1$ and $y_{i2}^2$ commonly denote the 2-nd cluster, and they are defined as the common cluster assignments existing in these two views.
\end{proof}
The proofs in the above theorems can be easily extended to the situation of multiple views.\\

\textbf{Complexity analysis}.~In $\mathcal{T}$-level iteration, $K$-means will be adopted to calculate the cluster centroids for $\mathbf{Z}_{(t)} \in \mathbb{R}^{N \times {\sum_{v=1}^V d_v}}$, which has the complexity of $O(NK{\sum_{v=1}^V d_v})$. Eqs.~(2, 6, 7, and 8) only involve the computation process with the complexity of $O(N)$. In $\mathcal{R}$-level iteration, $\mathop{\min_{\mathbf{A}^v}}{\| \mathbf{T} \mathbf{A}^v - \mathbf{Y}^v \|_F^2}$ needs the computation of Hungarian algorithm with the complexity of $O(K^3)$. The loss function of Eq.~(11) of all views is with the complexity of $O(N{\sum_{v=1}^V d_v} + VNK)$. Both $\mathcal{T}$- and $\mathcal{R}$-level iterations can converge, and the total complexity is $O(E/T_2(T_1 (NK{\sum_{v=1}^V d_v}) + VK^3 + VNK))$ with respect to iterations $T_1$ and $T_2$, epoch $E$, and data size $N$. The complexity to train deep autoencoders is also linear to $N$.

\section*{Appendix C: Setting Details and More Experimental Results}

\begin{table}[!ht]
  \caption{Information of datasets}\label{sample-table}
  \centering
  \resizebox{0.5\linewidth}{!}{
  \begin{tabular}{llrrc}
    \toprule
    Name          & ~~~$V$      & $N$~~~        & $K$ & Features \\
    \midrule
    BDGP          &  ~~~2       & 2,500         & 5   & 1750/79 \\
    DIGIT         &  ~~~2       & 5,000         & 10  & (32$\times$32$\times$1) $\times$ 2 \\
    COIL          &  ~~~3       & 720           & 10  & (32$\times$32$\times$1) $\times$ 3 \\
    Amazon        &  ~~~3       & 4,790         & 10  & (32$\times$32$\times$3) $\times$ 3 \\
    NoisyBDGP     &  ~~~3       & 2,500         & 5   & 1750/79/10 \\
    NoisyDIGIT    &  ~~~3       & 5,000         & 10  & (32$\times$32$\times$1) $\times$ 3 \\
    NoisyCOIL     &  ~~~4       & 720           & 10  & (32$\times$32$\times$1) $\times$ 4 \\
    NoisyAmazon   &  ~~~4       & 4,790         & 10  & (32$\times$32$\times$3) $\times$ 4 \\
    DHA           &  ~~~2       & 483           & 23  & 110/6144 \\
    RGB-D         &  ~~~2       & 1,449         & 13  & 2048/300 \\
    Caltech       &  ~~~6       & 1,400         & 7   & 48/40/254/1984/512/928 \\
    YoutubeVideo  &  ~~~3       & 101,499       & 31  & 512/647/838 \\
    \bottomrule
  \end{tabular}
  }
\end{table}

\textbf{Datasets}.~Our experiments are conducted on 8 public datasets including 4 noise-simulated ones, as listed in Table~\ref{sample-table}. 

Firstly, four normal datasets include BDGP \cite{cai2012joint}, DIGIT \cite{peng2019comic}, COIL \cite{nene1996columbia}, and Amazon \cite{saenko2010adapting}. Different views of these datasets often are informative views that have consistent semantic category information, and thus their hidden cluster structures are relatively easy to be discovered.
Concretely,
BDGP is a drosophila embryos dataset where each sample contains a visual view and a textual view.
DIGIT (\ie, MNIST-USPS \cite{peng2019comic}) is a dataset of handwritten arabic numbers in which each sample has two views (paired by two digits with the same label) coming from MNIST and USPS, respectively.
COIL and Amazon (\ie, Multi-COIL-10 and Multi-Amazon \cite{xu2023untie}) are also two artificially constructed multi-view datasets, which are formed by following the manner of constructing MNIST-USPS \cite{peng2019comic}.
COIL is a multi-view dataset and the multiple views of one object are captured by the cameras from different visual angles.
Amazon consists of the commodity images, such as bicycles, bags, and earphones, where the raw samples are single-view RGB images and their augmented data are obtained by rotation, horizontal flip, and color filter to further construct the multi-view dataset.

Secondly, we construct 4 noise-simulated datasets (namely NoisyBDGP, NoisyDIGIT, NoisyCOIL, and NoisyAmazon) on the four normal datasets to test the noise robustness of methods in extreme scenarios.
These datasets have both informative views and noisy views, which is helpful to understand the stability of comparison methods under extreme noise interference.
For each dataset, we randomly sample noise to build an additional view and obtain NoisyBDGP, NoisyDIGIT, NoisyCOIL, and NoisyAmazon, respectively.
To reduce the computational burden, all image samples are scaled to $32 \times 32 \times 1$ (DIGIT, COIL, NoisyDIGIT, and NoisyCOIL) or $32 \times 32 \times 3$ (Amazon and NoisyAmazon) pixels.

Thirdly, we also conduct experiments on four real-world datasets including DHA \cite{lin2012human}, RGB-D \cite{zhou2020end}, Caltech \cite{fei2004learning}, and YoutubeVideo \cite{madani2012using}.
These datasets usually have non-negligible noisy views that is from real-world multi-view environments, so that it is relatively difficult to explore their hidden cluster information.
Specifically,
DHA is a multimodal dataset that leverages RGB and depth features to capture human actions.
RGB-D contains different indoor scenes described by image features and textual descriptions.
Caltech is a popular multi-view datasets which has six views (Gabor, WM, CENTRIST, HOG, GIST, and LBP) to represent image samples.
YoutubeVideo is a large-scale dataset containing video contents with cuboids histogram, HOG, and MISC features.\\

\textbf{Implementation}.~For a fair comparison, we adopt the same network structures to implement the models for all deep MVC methods like previous work \cite{xie2016unsupervised,guo2017improved,xie2020joint}.
Concretely, for the vector input (BDGP, NoisyBDGP, DHA, RGB-D, Caltech, and YoutubeVideo), the autoencoder adopts the fully connected network with the architecture of $\mathbf{X}^v-500-500-2000-\mathbf{Z}^v-2000-500-500-\Hat{\mathbf{X}}^v$.
For the image input (DIGIT, COIL, Amazon, NoisyDIGIT, NoisyCOIL, and NoisyAmazon), the autoencoder adopts the convolutional neural network whose architecture is $\mathbf{X}^v-32-64-64-\mathbf{Z}^v-64-32-32-\Hat{\mathbf{X}}^v$ with the kernel size of $4 \times 4$, the stride of 2, and the padding of 1.
The activation function in networks is ReLU \cite{glorot2011deep}.
For the proposed MVCAN, the dimensionality of all $\{\mathbf{Z}^v\}_{v=1}^V$ is set to 10.
Adam \cite{kingma2014adam} with a learning rate of 0.0001 is adopted for mini-batch optimization with the batch size of 256.
To make the optimization of multiple views be away from mutual interference, the deep model structure achieves the decoupling of network parameters for multiple views and we respectively construct $V$ decoupled Adam optimizers for $V$ different views.
In $\mathcal{T}$-level iterative optimization, $T_1 = 2$.
In $\mathcal{R}$-level iterative optimization, $T_2 = 100$.
MVCAN is implemented by PyTorch~\cite{paszke2019pytorch}, and more implementation details can be found in our code (\url{https://github.com/SubmissionsIn/MVCAN}) together with datasets and trained models.
To implement the clustering and representation learning optimization objectives of MVCAN, we adopt the mean squared error (MSE) in PyTorch to compute the optimization loss of $\mathcal{L}_c^v$ and $\mathcal{L}_r^v$.

For the comparison methods, we tune the parameters for all baselines as suggested in their published work.
Specially,
since the original DMJC \cite{xie2020joint} conducts the single-view image clustering by exploring multiple views from the same samples, we adopt the proposed variant of DMJC-T and the multi-view data replaces the input of the single-view data in experiments.
For the incomplete multi-view clustering methods (\ie, DIMC-net \cite{wen2020dimc}, GP-MVC \cite{wang2021generative}, DIMVC \cite{xu2022deep}, DSIMVC \cite{tang2022deepi}, and CPSPAN \cite{jin2023deep}), their missing rates are set to 0.
All experiments are run on a same device with NVIDIA GeForce RTX 3090 GPUs (24.0GB caches) and 11th Gen Intel(R) Core(TM) i5-11600KF @ 3.90GHz (64.0GB RAM).\\

\textbf{More experimental results}.~We provide more experimental results which can not be shown in the paper due to its space limitation. The clustering effectiveness is evaluated by three widely used metrics, \ie, clustering accuracy (ACC), normalized mutual information (NMI), and adjusted rand index (ARI).

\begin{table*}[!ht]\caption{Importance of two conditions in clustering objective on four normal datasets.}\label{tab:table5}
\small
\centering
\renewcommand\tabcolsep{10.0pt}
\resizebox{\linewidth}{!}{
\begin{threeparttable}
    \begin{tabular}{l|cc|ccc|ccc|ccc|ccc}
    \toprule
    &\multicolumn{2}{c|}{Conditions} &\multicolumn{3}{c|}{BDGP} &\multicolumn{3}{c|}{DIGIT} &\multicolumn{3}{c|}{COIL} &\multicolumn{3}{c}{Amazon} \cr
    \hline
    &$\mathbf{\Theta}^v$ &$\mathbf{A}^v$  & ACC & NMI & ARI & ACC & NMI & ARI & ACC & NMI & ARI & ACC & NMI & ARI\\
    \hline
    (i)&$\checkmark$&\quad                      &0.973 &0.921 &0.935 &0.986 &0.980 &0.982 &0.942 &0.924 &0.886 &0.795 &0.810 &0.714 \\
    (ii)&\quad&$\checkmark$                     &0.660 &0.477 &0.451 &0.840 &0.735 &0.690 &0.913 &0.868 &0.825 &0.685 &0.696 &0.536 \\
    (iii)&$\checkmark$&$\checkmark$             &0.984 &0.953 &0.961 &0.995 &0.988 &0.989 &0.996 &0.991 &0.990 &0.826 &0.867 &0.779 \\
    \bottomrule
    \end{tabular}
\end{threeparttable}
}
\end{table*}

\begin{table*}[!ht]\caption{Importance of two conditions in clustering objective on four noise-simulated datasets.}\label{tab:table6}
\small
\centering
\renewcommand\tabcolsep{10.0pt}
\resizebox{\linewidth}{!}{
\begin{threeparttable}
    \begin{tabular}{l|cc|ccc|ccc|ccc|ccc}
    \toprule
    &\multicolumn{2}{c|}{Conditions} &\multicolumn{3}{c|}{NoisyBDGP} &\multicolumn{3}{c|}{NoisyDIGIT} &\multicolumn{3}{c|}{NoisyCOIL} &\multicolumn{3}{c}{NoisyAmazon} \cr
    \hline
    &$\mathbf{\Theta}^v$ &$\mathbf{A}^v$  & ACC & NMI & ARI & ACC & NMI & ARI & ACC & NMI & ARI & ACC & NMI & ARI\\
    \hline
    (i)&$\checkmark$&\quad                      &0.977 &0.946 &0.949 &0.902 &0.933 &0.884 &0.890 &0.912 &0.851 &0.667 &0.747 &0.621 \\
    (ii)&\quad&$\checkmark$                     &0.602 &0.339 &0.305 &0.611 &0.599 &0.436 &0.806 &0.832 &0.719 &0.523 &0.514 &0.390 \\
    (iii)&$\checkmark$&$\checkmark$             &0.980 &0.951 &0.957 &0.990 &0.984 &0.986 &0.992 &0.988 &0.987 &0.728 &0.732 &0.614 \\
    \bottomrule
    \end{tabular}
\end{threeparttable}
}
\end{table*}

\begin{table*}[!ht]\caption{Importance of two loss components in optimization on four normal datasets.}\label{tab:table3}
\small
\centering
\renewcommand\tabcolsep{10.0pt}
\resizebox{\linewidth}{!}{
\begin{threeparttable}
    \begin{tabular}{l|cc|ccc|ccc|ccc|ccc}
    \toprule
    &\multicolumn{2}{c|}{Components} &\multicolumn{3}{c|}{BDGP} &\multicolumn{3}{c|}{DIGIT} &\multicolumn{3}{c|}{COIL} &\multicolumn{3}{c}{Amazon} \cr
    \hline
    &$\mathcal{L}_{r}^v$ &$\mathcal{L}_{c}^v$  & ACC & NMI & ARI & ACC & NMI & ARI & ACC & NMI & ARI & ACC & NMI & ARI\\
    \hline
    (a)&\quad&\quad                      &0.643 &0.522 &0.245 &0.768 &0.723 &0.635 &0.839 &0.927 &0.838 &0.441 &0.373 &0.268 \\
    (b)&$\checkmark$&\quad               &0.948 &0.842 &0.874 &0.787 &0.747 &0.658 &0.826 &0.888 &0.787 &0.498 &0.452 &0.322 \\
    (c)&\quad&$\checkmark$               &0.798 &0.709 &0.656 &0.870 &0.943 &0.867 &0.876 &0.957 &0.876 &0.724 &0.817 &0.678 \\
    (d)&$\checkmark$&$\checkmark$        &0.984 &0.953 &0.961 &0.995 &0.988 &0.989 &0.996 &0.991 &0.990 &0.826 &0.867 &0.779 \\
    \bottomrule
    \end{tabular}
\end{threeparttable}
}
\end{table*}

\begin{table*}[!ht]\caption{Importance of two loss components in optimization on four noise-simulated datasets.}\label{tab:table4}
\small
\centering
\renewcommand\tabcolsep{10.0pt}
\resizebox{\linewidth}{!}{
\begin{threeparttable}
    \begin{tabular}{l|cc|ccc|ccc|ccc|ccc}
    \toprule
    &\multicolumn{2}{c|}{Components} &\multicolumn{3}{c|}{NoisyBDGP} &\multicolumn{3}{c|}{NoisyDIGIT} &\multicolumn{3}{c|}{NoisyCOIL} &\multicolumn{3}{c}{NoisyAmazon} \cr
    \hline
    &$\mathcal{L}_{r}^v$ &$\mathcal{L}_{c}^v$  & ACC & NMI & ARI & ACC & NMI & ARI & ACC & NMI & ARI & ACC & NMI & ARI\\
    \hline
    (a)&\quad&\quad                      &0.499 &0.315 &0.253 &0.741 &0.705 &0.611 &0.828 &0.875 &0.785 &0.432 &0.371 &0.269 \\
    (b)&$\checkmark$&\quad               &0.944 &0.839 &0.864 &0.769 &0.748 &0.651 &0.813 &0.867 &0.780 &0.487 &0.455 &0.341 \\
    (c)&\quad&$\checkmark$               &0.726 &0.574 &0.381 &0.591 &0.702 &0.504 &0.686 &0.782 &0.612 &0.423 &0.423 &0.258 \\
    (d)&$\checkmark$&$\checkmark$        &0.980 &0.951 &0.957 &0.990 &0.984 &0.986 &0.992 &0.988 &0.987 &0.728 &0.732 &0.614 \\
    \bottomrule
    \end{tabular}
\end{threeparttable}
}
\end{table*}

\begin{figure*}[!ht]
\centering
  \begin{subfigure}{0.33\linewidth}
    \includegraphics[width=\linewidth]{ACC_T1.png}
    \caption{ACC $vs.$ $T_1$}
  \end{subfigure}
  \begin{subfigure}{0.33\linewidth}
    \includegraphics[width=\linewidth]{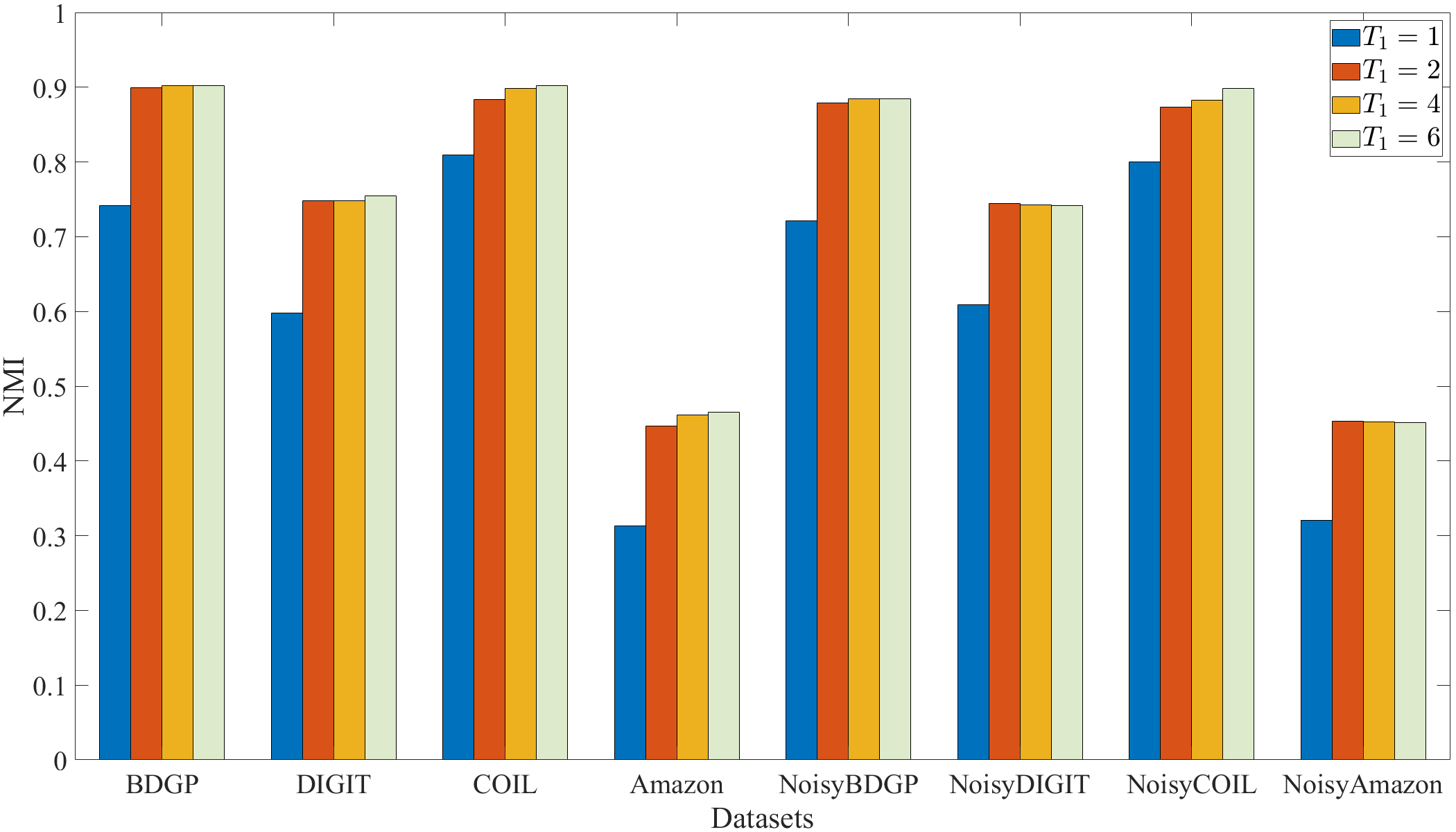}
    \caption{NMI $vs.$ $T_1$}
  \end{subfigure}
  \begin{subfigure}{0.33\linewidth}
    \includegraphics[width=\linewidth]{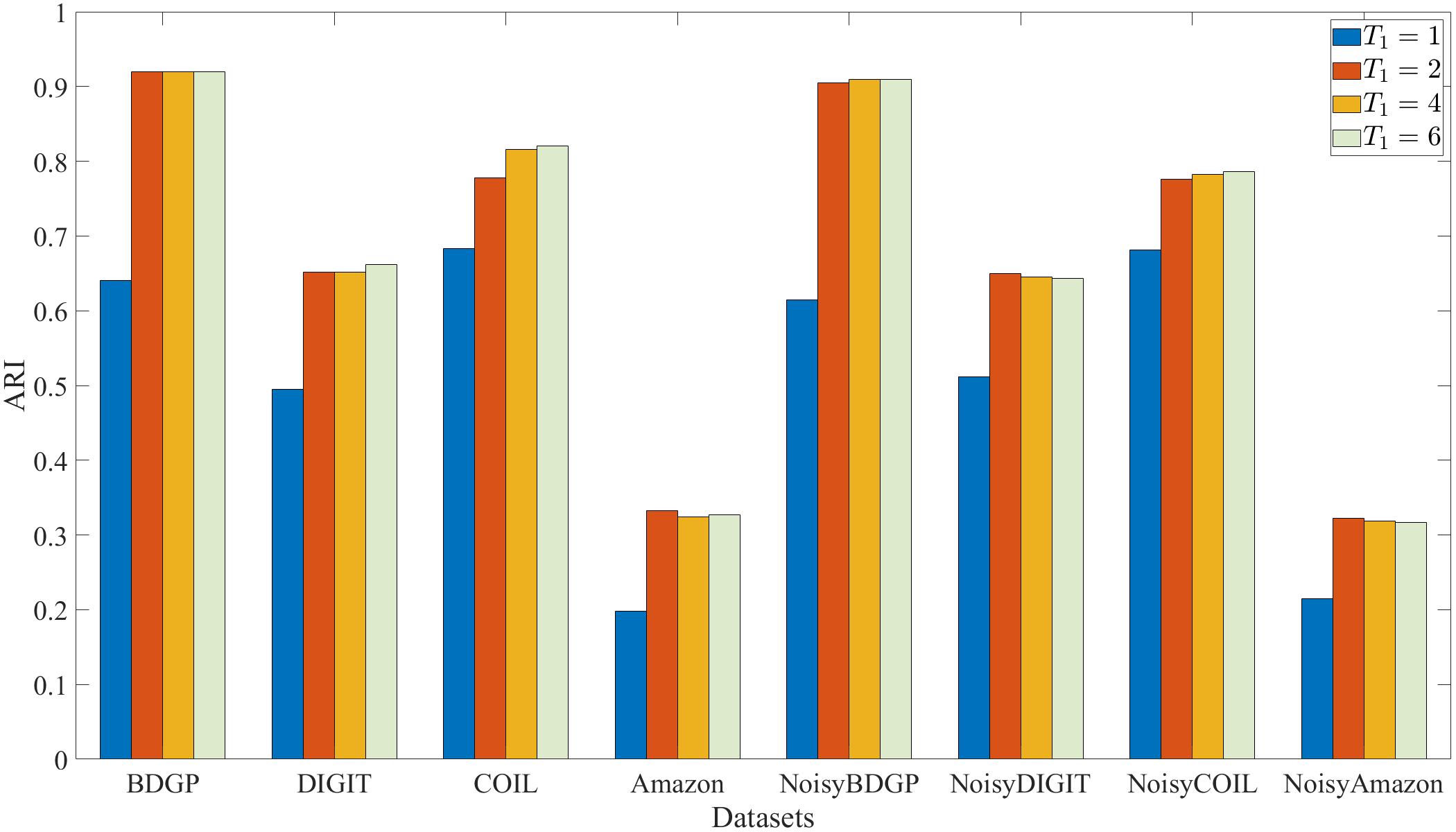}
    \caption{ARI $vs.$ $T_1$}
  \end{subfigure}
  \begin{subfigure}{0.33\linewidth}
    \includegraphics[width=\linewidth]{ACC_T2.png}
    \caption{ACC $vs.$ $T_2$}
  \end{subfigure}
  \begin{subfigure}{0.33\linewidth}
    \includegraphics[width=\linewidth]{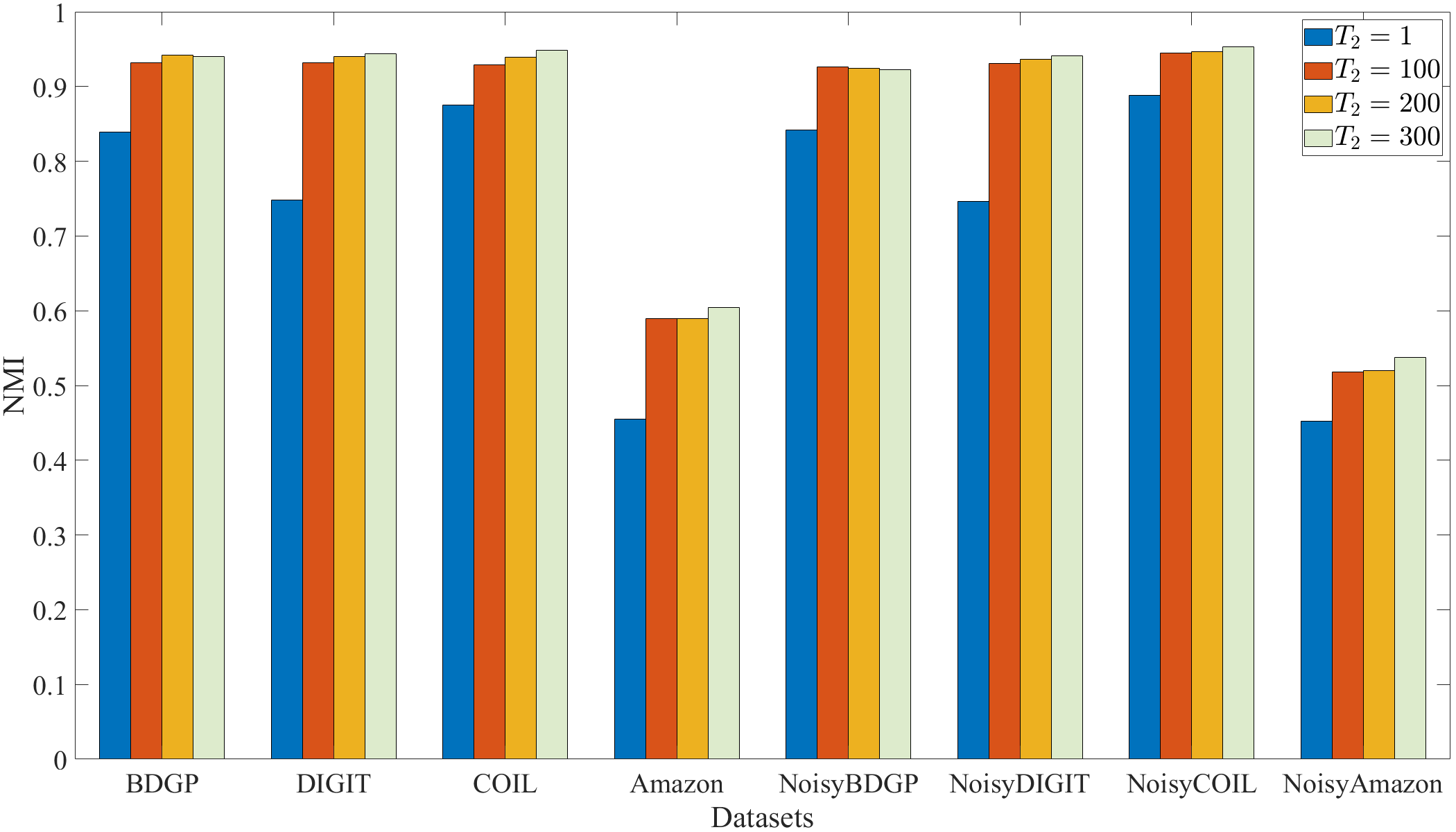}
    \caption{NMI $vs.$ $T_2$}
  \end{subfigure}
  \begin{subfigure}{0.33\linewidth}
    \includegraphics[width=\linewidth]{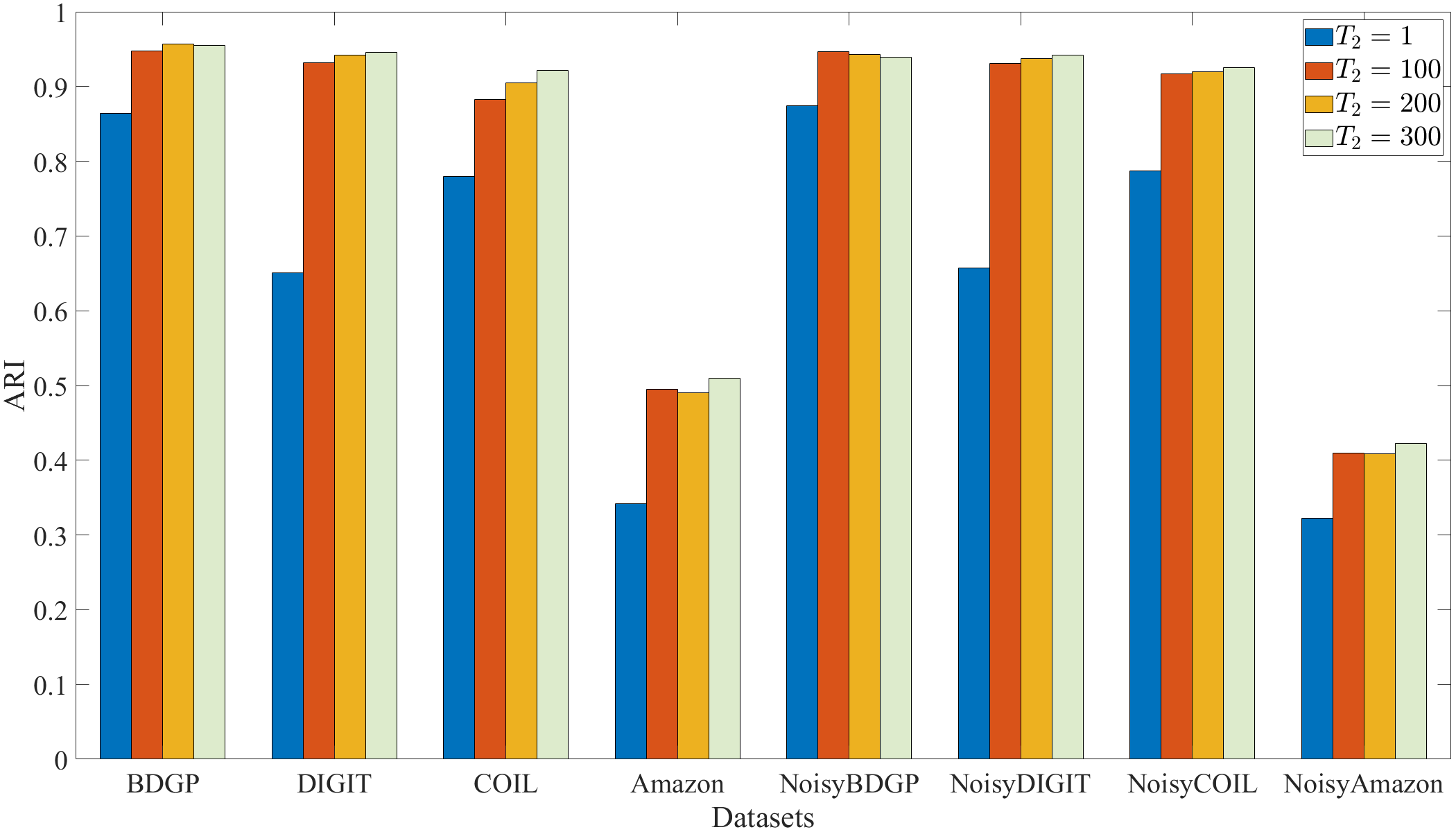}
    \caption{ARI $vs.$ $T_2$}
  \end{subfigure}
\caption{Different training iterations in $\mathcal{T}$-level (a-c) and $\mathcal{R}$-level (d-f) of the proposed two-level multi-view iterative optimization.}\label{BS0}
\end{figure*}

\begin{figure*}[!ht]
\centering
  \begin{subfigure}{0.24\linewidth}
    \includegraphics[width=\linewidth]{BDGP_plot.png}
    \caption{BDGP}
  \end{subfigure}
  \begin{subfigure}{0.24\linewidth}
    \includegraphics[width=\linewidth]{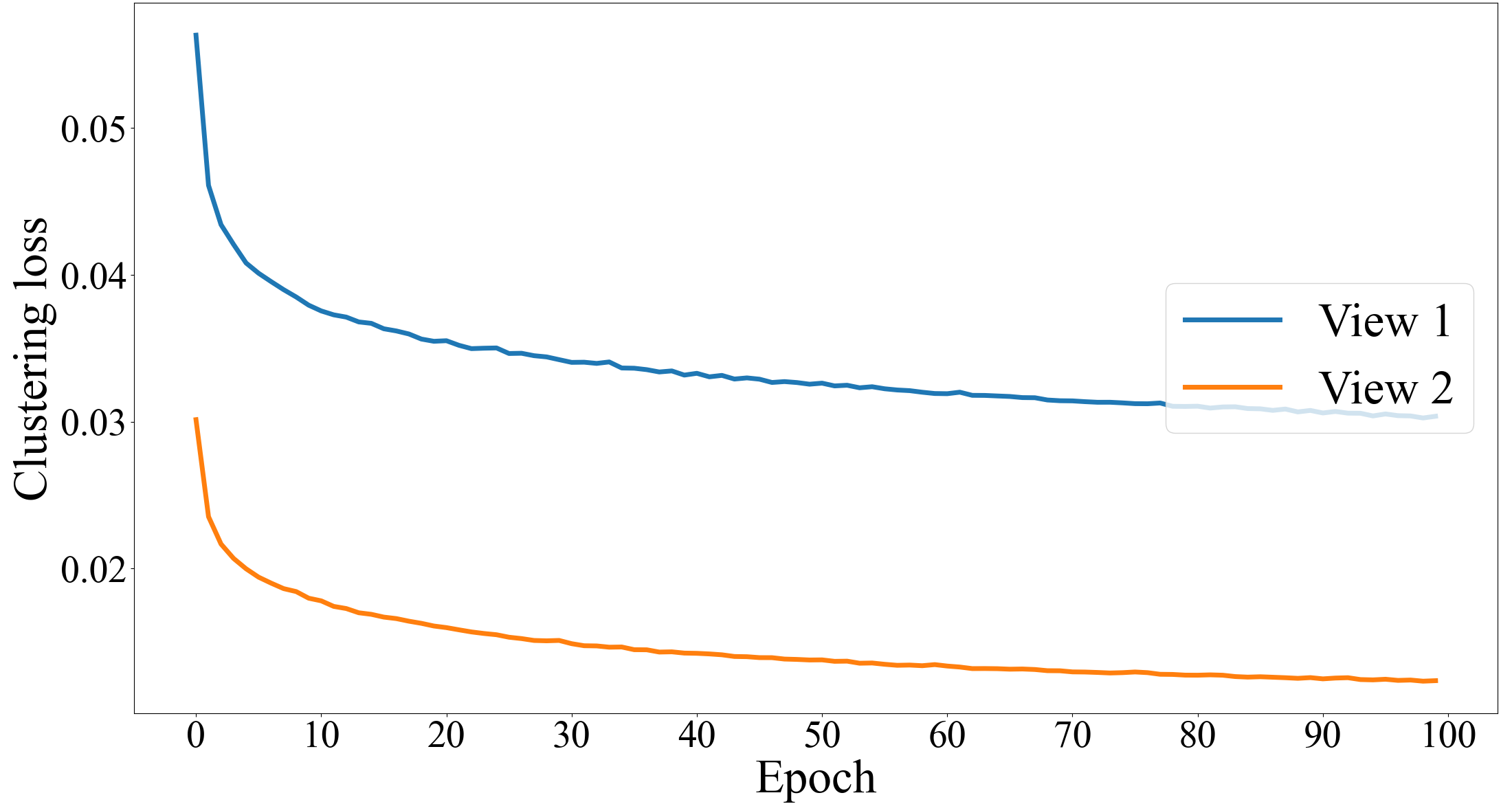}
    \caption{DIGIT}
  \end{subfigure}
  \begin{subfigure}{0.24\linewidth}
    \includegraphics[width=\linewidth]{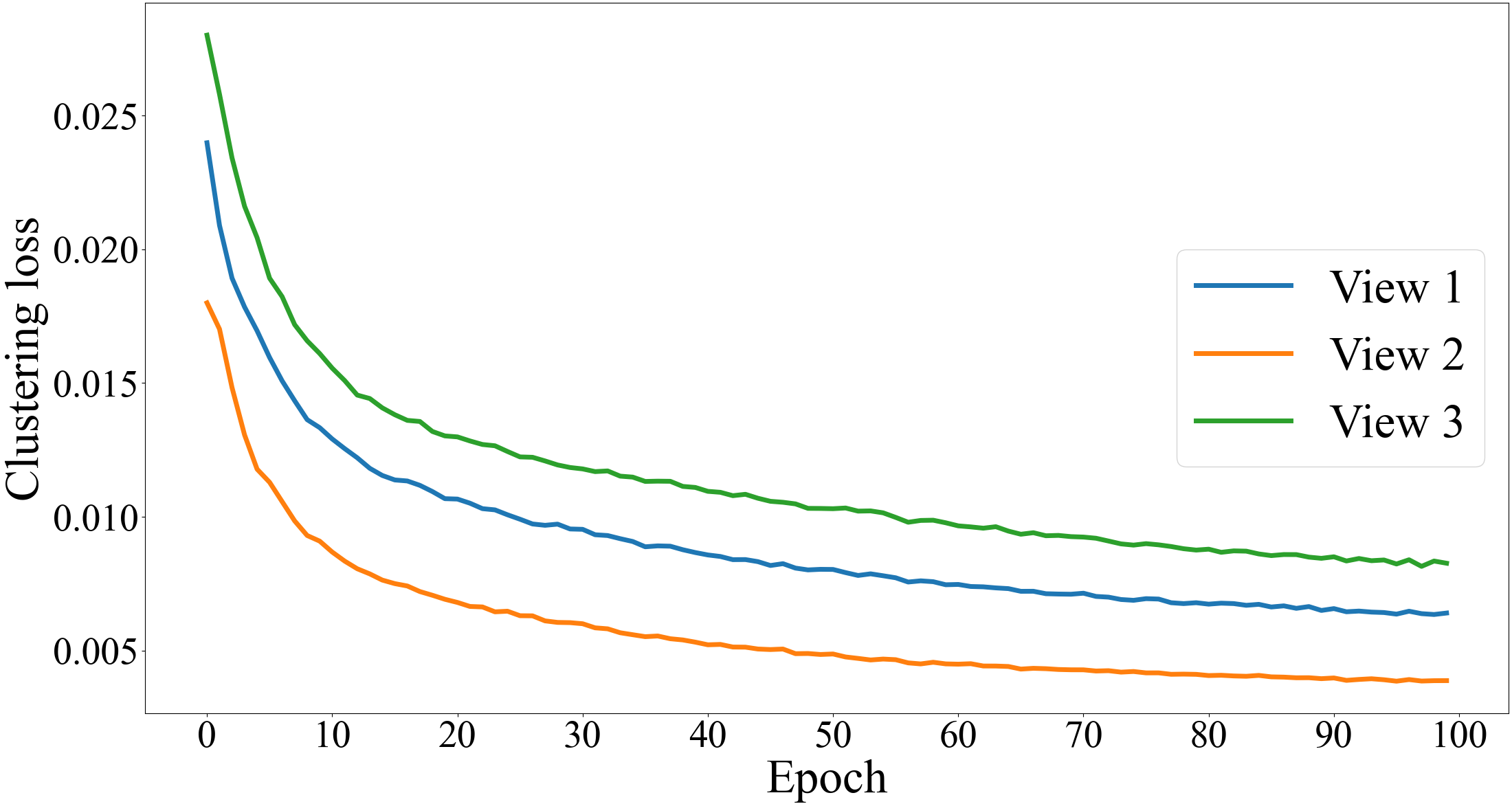}
    \caption{COIL}
  \end{subfigure}
  \begin{subfigure}{0.24\linewidth}
    \includegraphics[width=\linewidth]{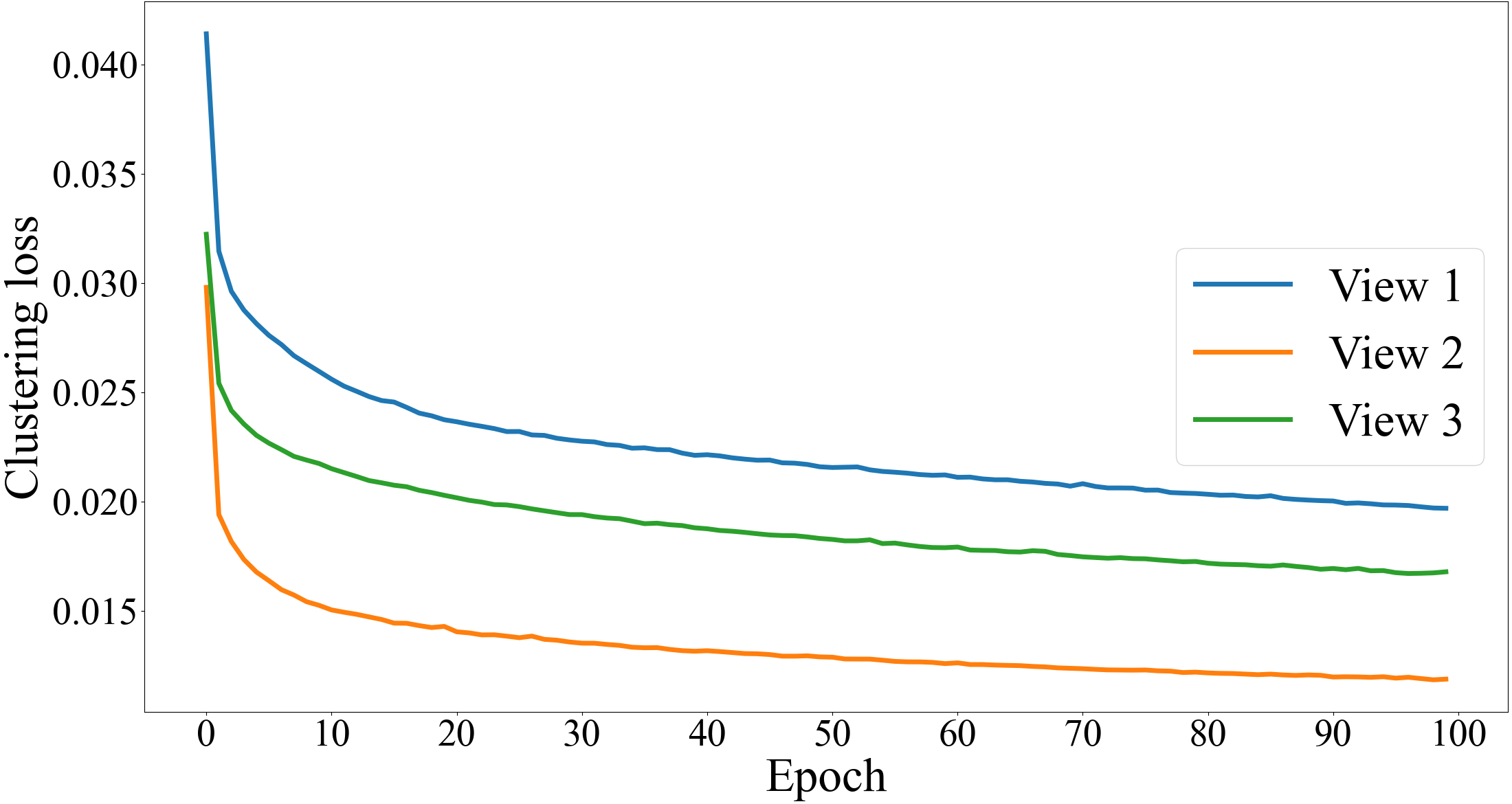}
    \caption{Amazon}
  \end{subfigure}
  \begin{subfigure}{0.24\linewidth}
    \includegraphics[width=\linewidth]{BDGP-N_plot.png}
    \caption{NoisyBDGP}
  \end{subfigure}
  \begin{subfigure}{0.24\linewidth}
    \includegraphics[width=\linewidth]{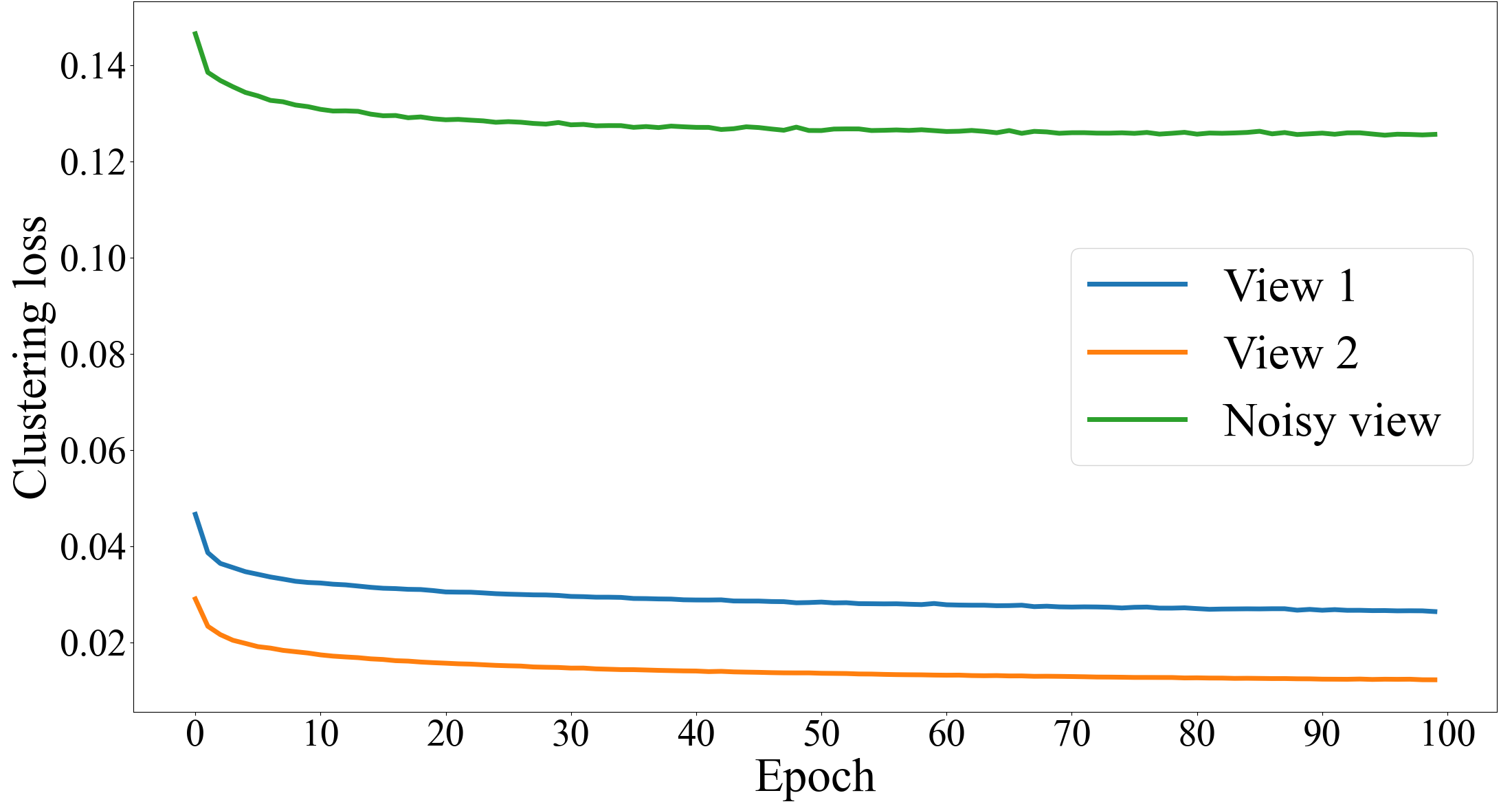}
    \caption{NoisyDIGIT}
  \end{subfigure}
  \begin{subfigure}{0.24\linewidth}
    \includegraphics[width=\linewidth]{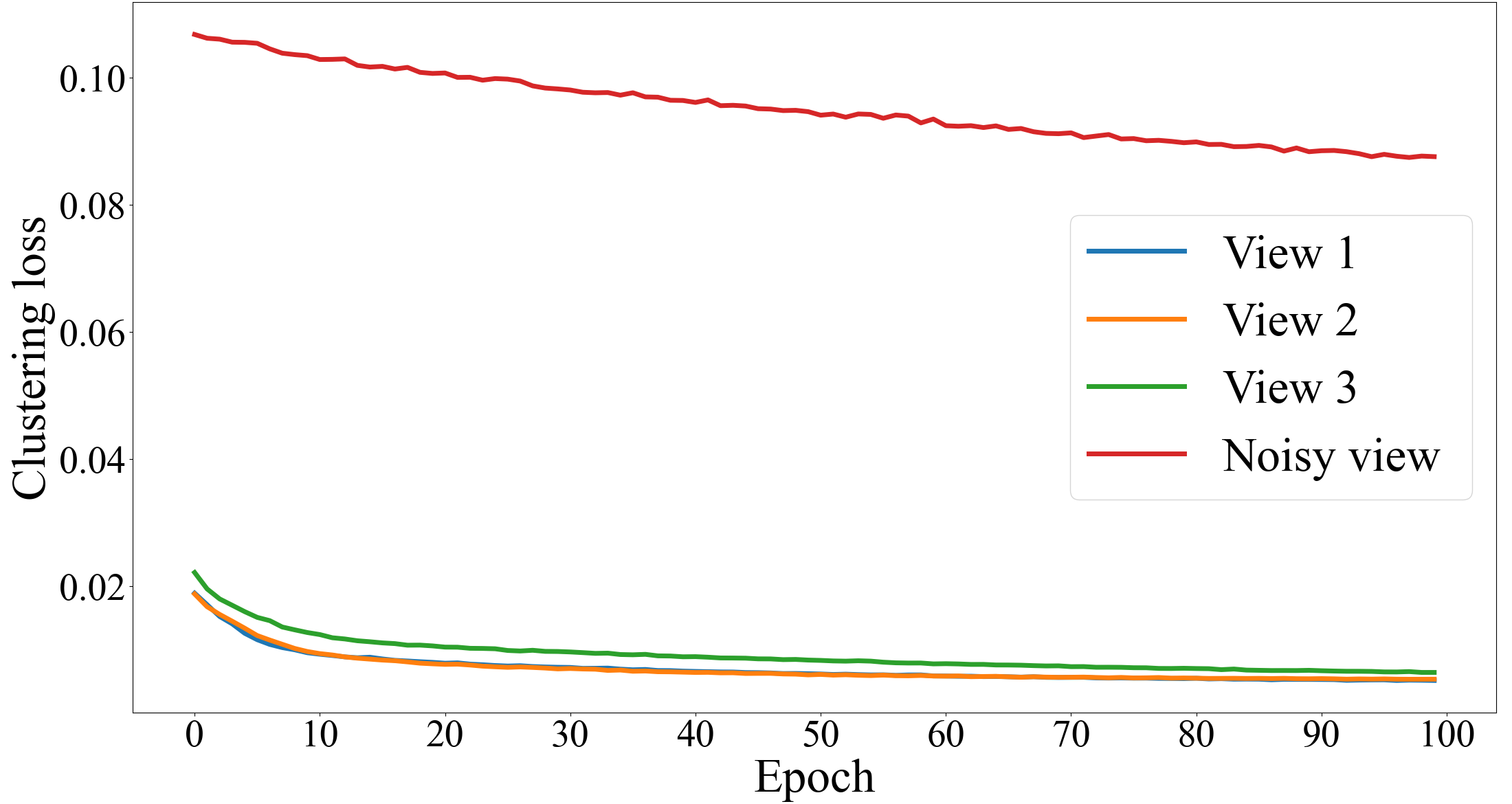}
    \caption{NoisyCOIL}
  \end{subfigure}
  \begin{subfigure}{0.24\linewidth}
    \includegraphics[width=\linewidth]{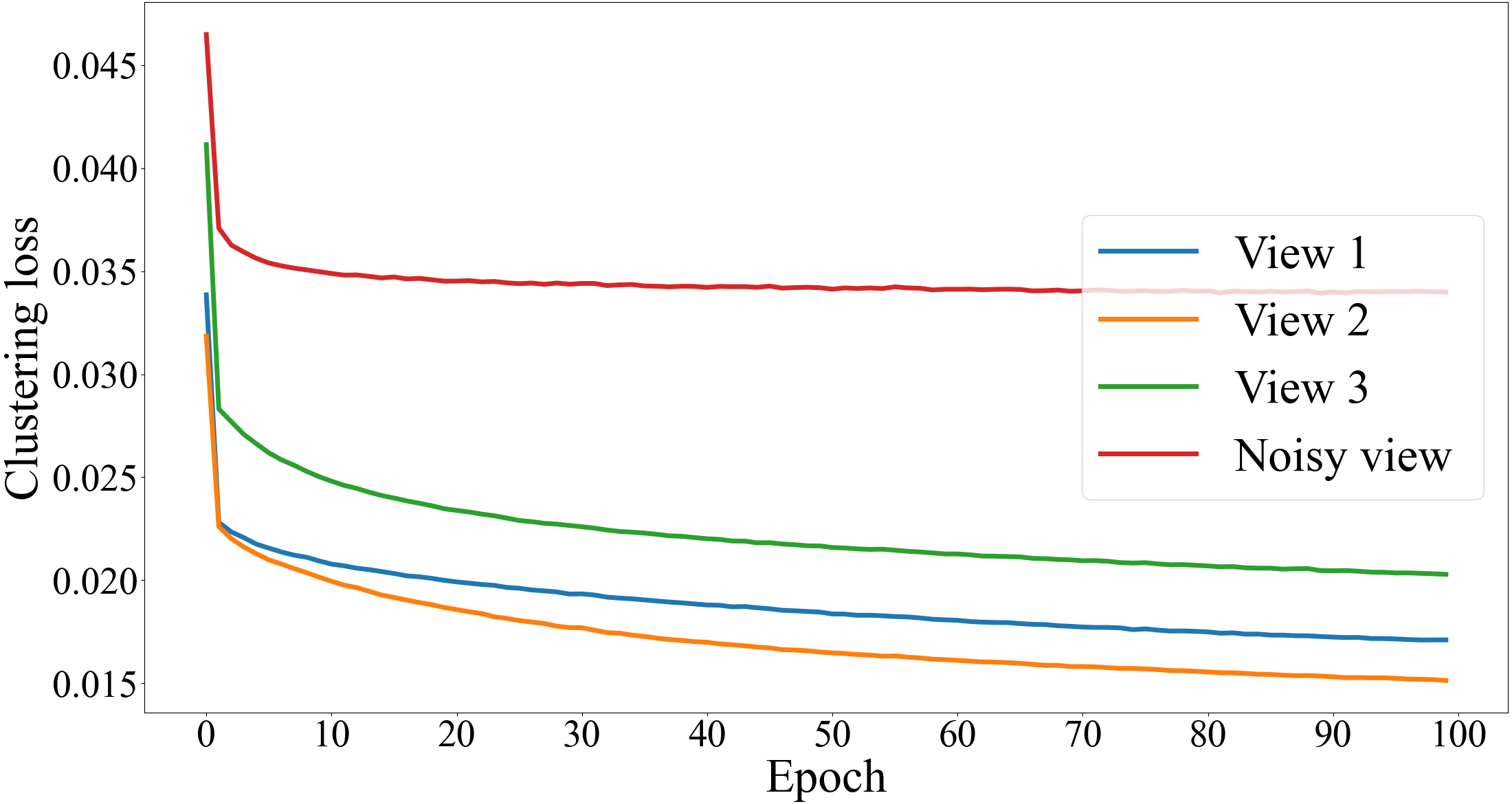}
    \caption{NoisyAmazon}
  \end{subfigure}
\caption{Loss $vs.$ Epoch on four normal datasets (a-d) and on four noise-simulated datasets (e-h).}\label{lossplot0}
\end{figure*}

\begin{figure*}[!ht]
\centering
  \begin{subfigure}{0.33\linewidth}
    \includegraphics[width=\linewidth]{ACC.png}
    \caption{ACC $vs.$ $\lambda$}
  \end{subfigure}
  \begin{subfigure}{0.33\linewidth}
    \includegraphics[width=\linewidth]{NMI.png}
    \caption{NMI $vs.$ $\lambda$}
  \end{subfigure}
  \begin{subfigure}{0.33\linewidth}
    \includegraphics[width=\linewidth]{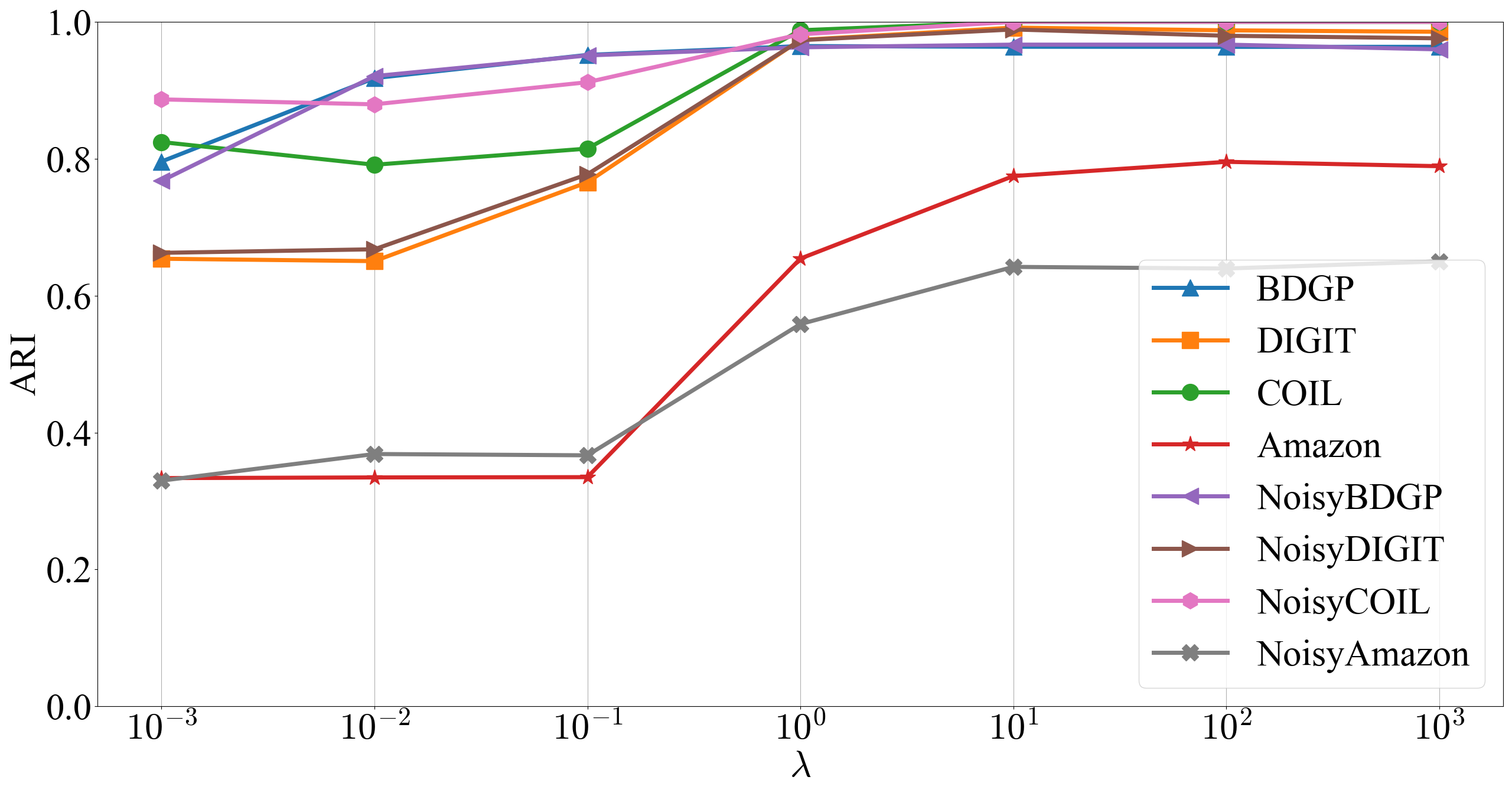}
    \caption{ARI $vs.$ $\lambda$}
  \end{subfigure}
\caption{Clustering effectiveness $vs.$ $\lambda$ on four normal datasets and on four noise-simulated datasets.}\label{trade0}
\end{figure*}

\begin{figure*}[!ht]
\centering
\begin{subfigure}{0.15\linewidth}
\includegraphics[width=\linewidth]{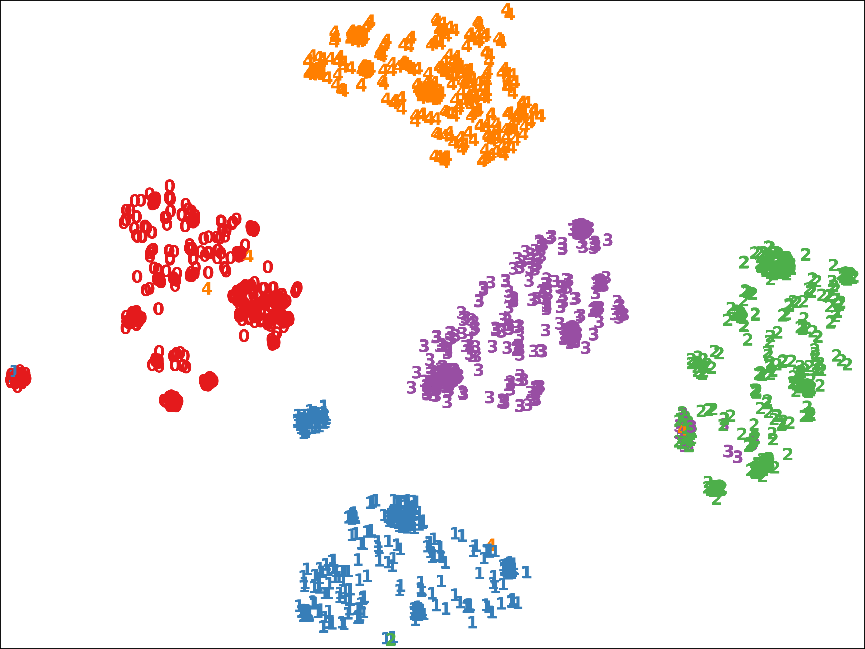}
\caption{$K=5$}
\end{subfigure}
\begin{subfigure}{0.15\linewidth}
\includegraphics[width=\linewidth]{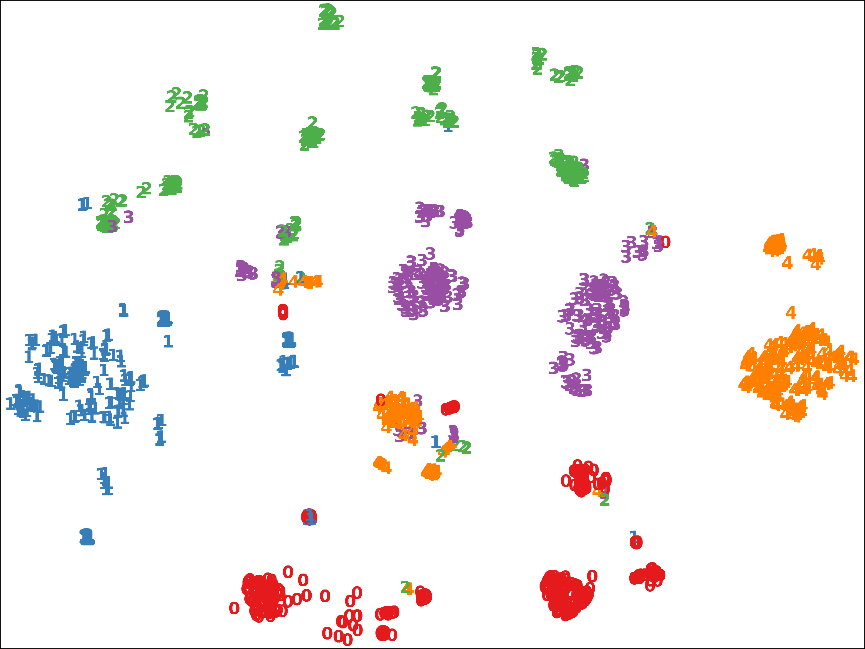}
\caption{$K=10$}
\end{subfigure}
\begin{subfigure}{0.15\linewidth}
\includegraphics[width=\linewidth]{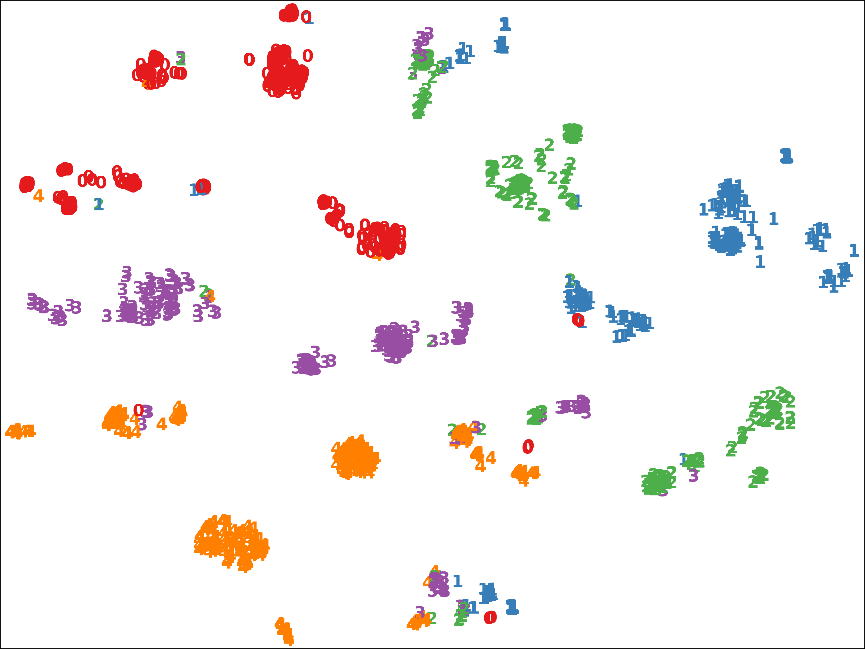}
\caption{$K=15$}
\end{subfigure}
\begin{subfigure}{0.15\linewidth}
\includegraphics[width=\linewidth]{DIGIT_K_5.png}
\caption{$K=5$}
\end{subfigure}
\begin{subfigure}{0.15\linewidth}
\includegraphics[width=\linewidth]{DIGIT_K_10.png}
\caption{$K=10$}
\end{subfigure}
\begin{subfigure}{0.15\linewidth}
\includegraphics[width=\linewidth]{DIGIT_K_15.png}
\caption{$K=15$}
\end{subfigure}
\begin{subfigure}{0.15\linewidth}
\includegraphics[width=\linewidth]{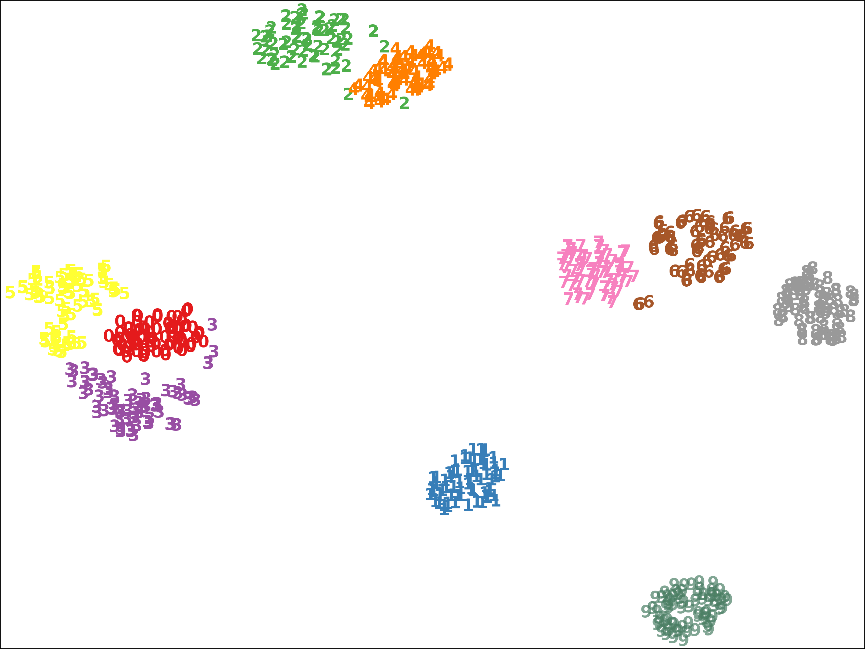}
\caption{$K=5$}
\end{subfigure}
\begin{subfigure}{0.15\linewidth}
\includegraphics[width=\linewidth]{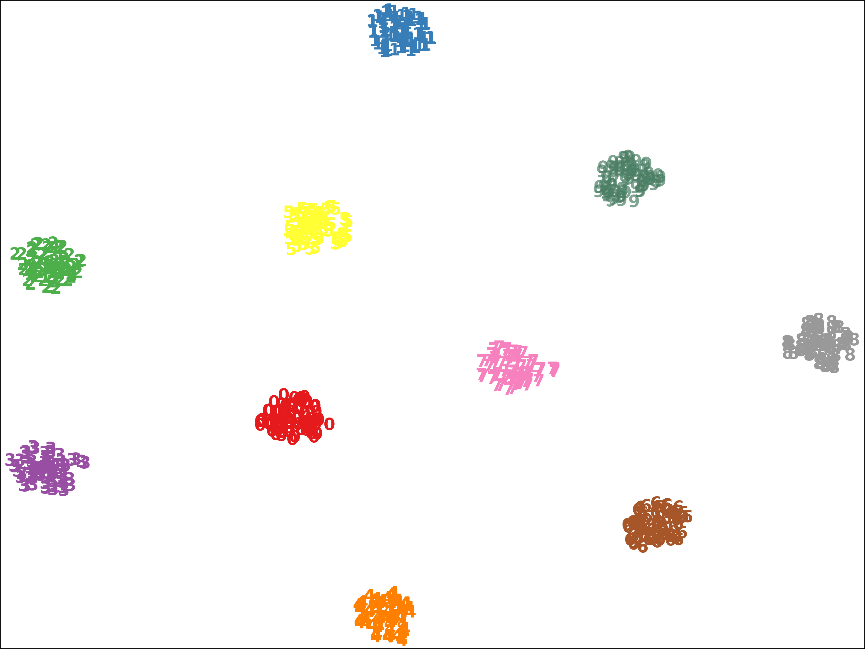}
\caption{$K=10$}
\end{subfigure}
\begin{subfigure}{0.15\linewidth}
\includegraphics[width=\linewidth]{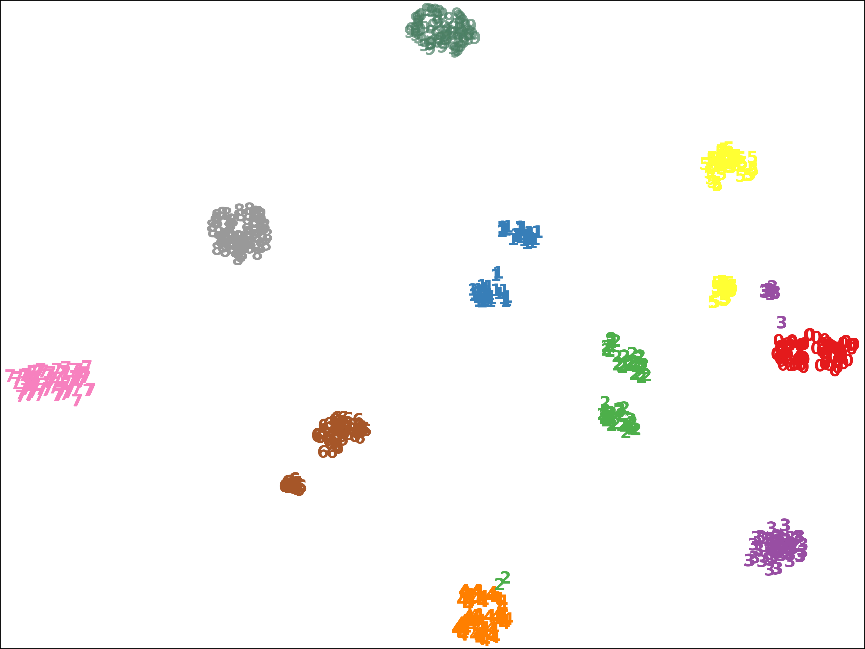}
\caption{$K=15$}
\end{subfigure}
\begin{subfigure}{0.15\linewidth}
\includegraphics[width=\linewidth]{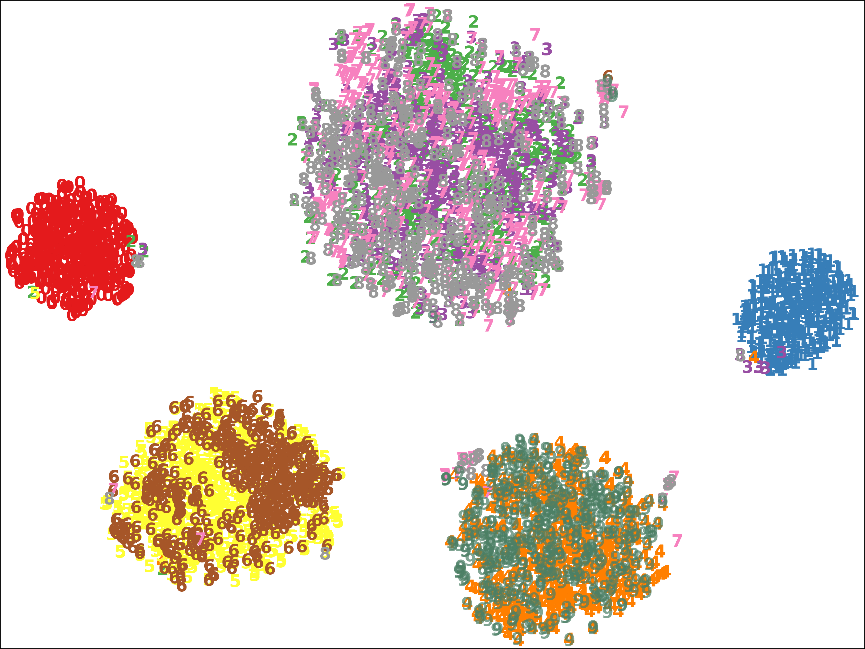}
\caption{$K=5$}
\end{subfigure}
\begin{subfigure}{0.15\linewidth}
\includegraphics[width=\linewidth]{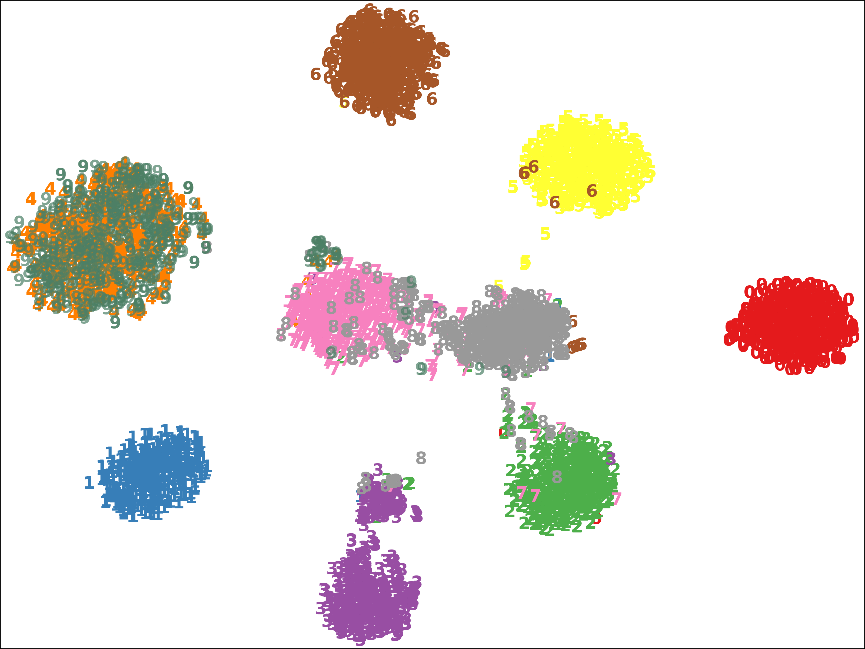}
\caption{$K=10$}
\end{subfigure}
\begin{subfigure}{0.15\linewidth}
\includegraphics[width=\linewidth]{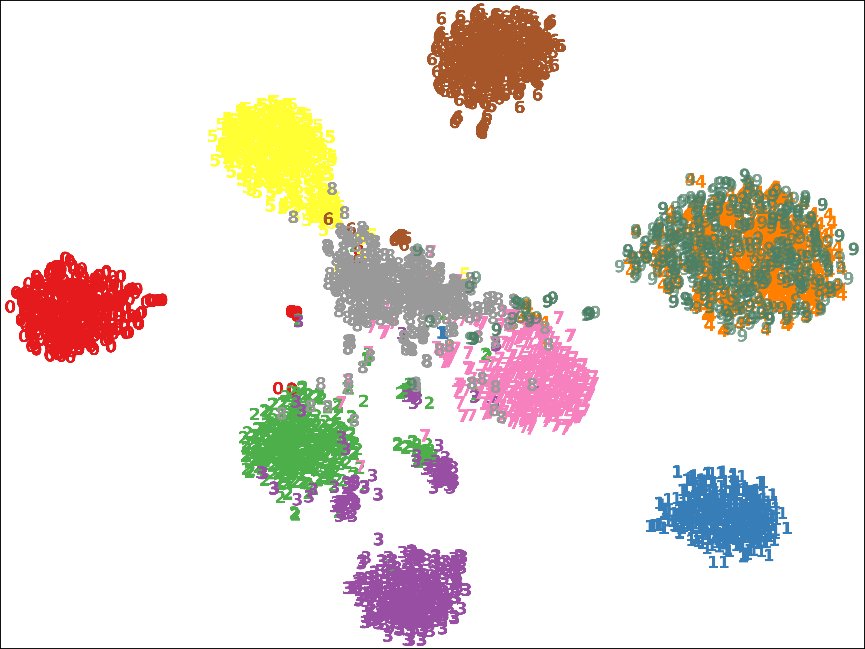}
\caption{$K=15$}
\end{subfigure}
\begin{subfigure}{0.15\linewidth}
\includegraphics[width=\linewidth]{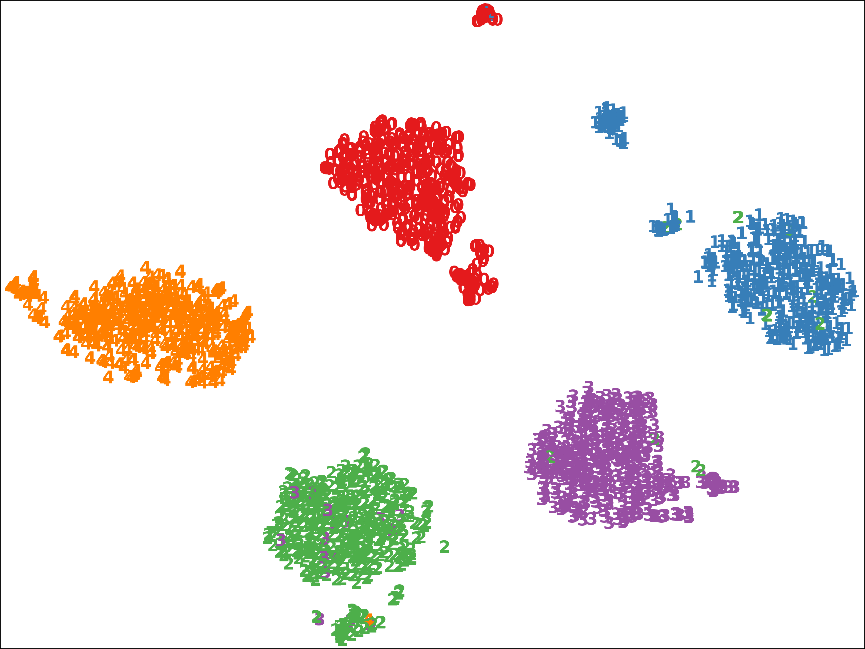}
\caption{$K=5$}
\end{subfigure}
\begin{subfigure}{0.15\linewidth}
\includegraphics[width=\linewidth]{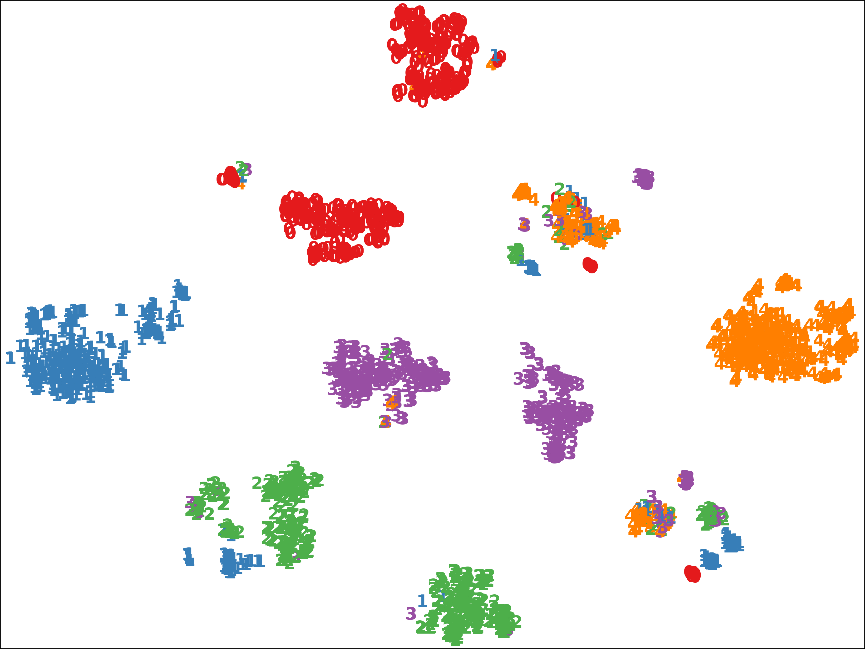}
\caption{$K=10$}
\end{subfigure}
\begin{subfigure}{0.15\linewidth}
\includegraphics[width=\linewidth]{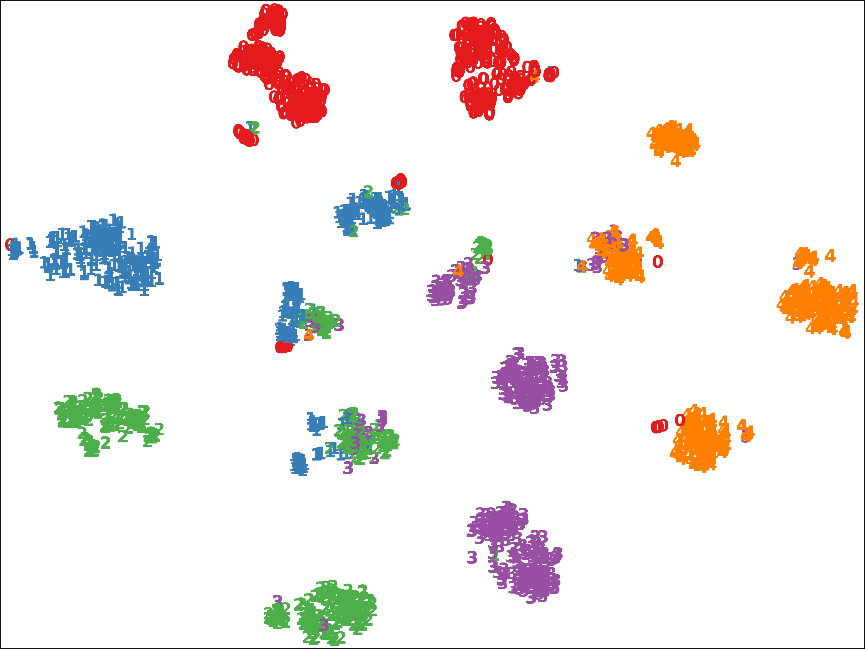}
\caption{$K=15$}
\end{subfigure}
\begin{subfigure}{0.15\linewidth}
\includegraphics[width=\linewidth]{DIGIT-N_K_5.png}
\caption{$K=5$}
\end{subfigure}
\begin{subfigure}{0.15\linewidth}
\includegraphics[width=\linewidth]{DIGIT-N_K_10.png}
\caption{$K=10$}
\end{subfigure}
\begin{subfigure}{0.15\linewidth}
\includegraphics[width=\linewidth]{DIGIT-N_K_15.png}
\caption{$K=15$}
\end{subfigure}
\begin{subfigure}{0.15\linewidth}
\includegraphics[width=\linewidth]{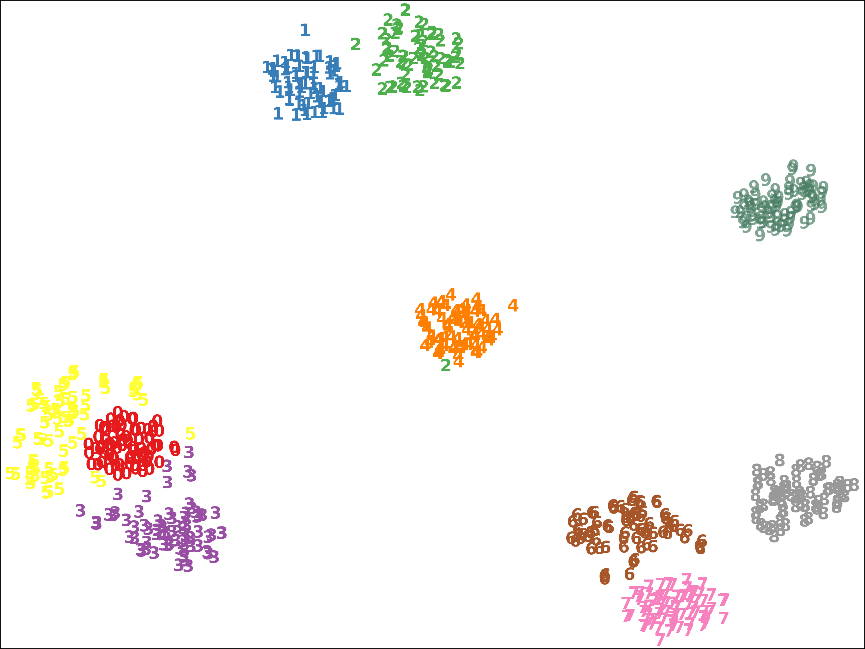}
\caption{$K=5$}
\end{subfigure}
\begin{subfigure}{0.15\linewidth}
\includegraphics[width=\linewidth]{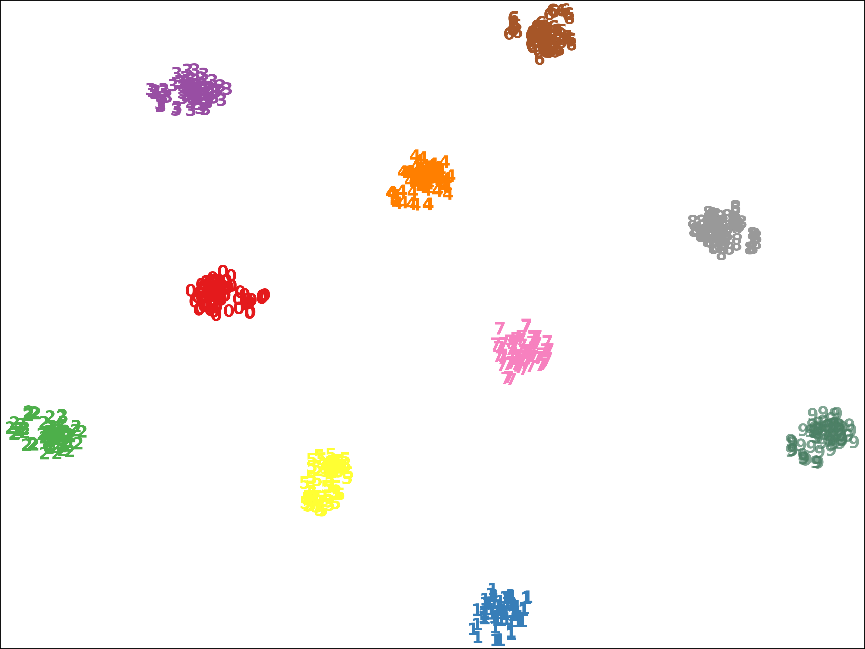}
\caption{$K=10$}
\end{subfigure}
\begin{subfigure}{0.15\linewidth}
\includegraphics[width=\linewidth]{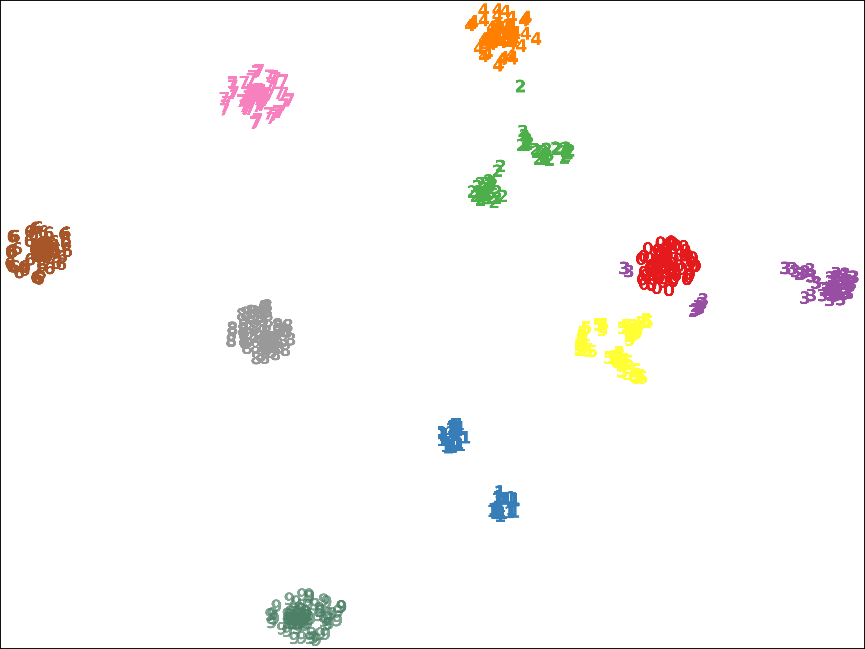}
\caption{$K=15$}
\end{subfigure}
\begin{subfigure}{0.15\linewidth}
\includegraphics[width=\linewidth]{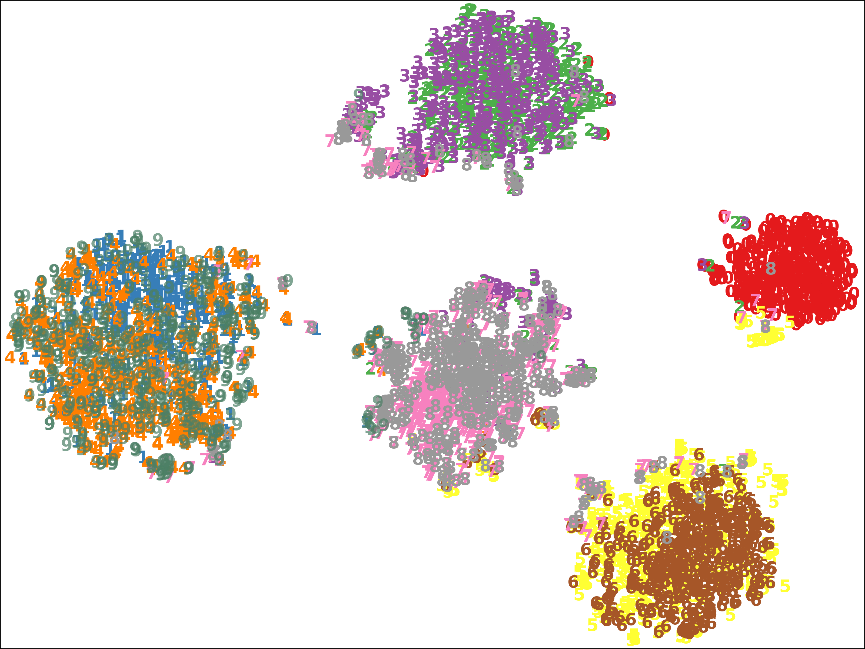}
\caption{$K=5$}
\end{subfigure}
\begin{subfigure}{0.15\linewidth}
\includegraphics[width=\linewidth]{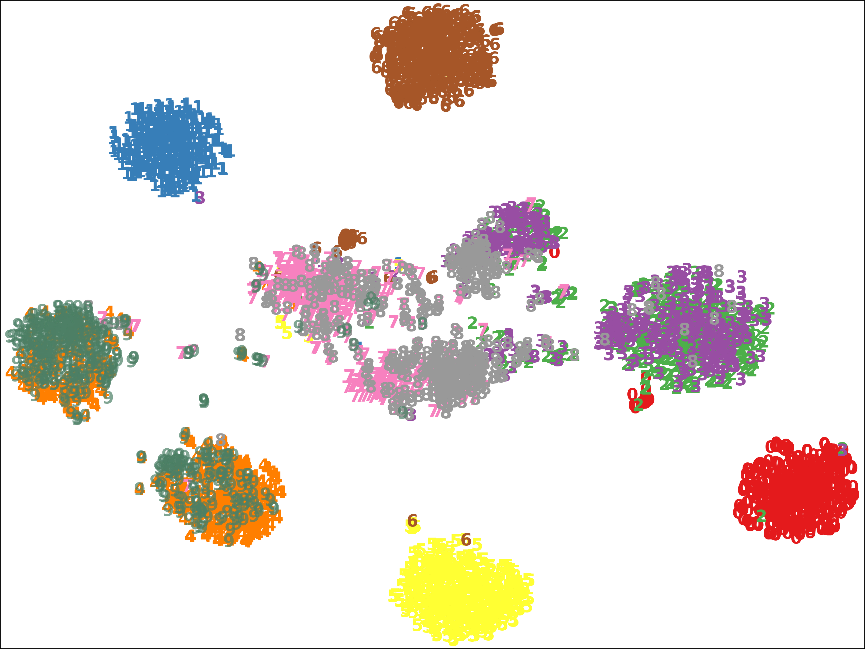}
\caption{$K=10$}
\end{subfigure}
\begin{subfigure}{0.15\linewidth}
\includegraphics[width=\linewidth]{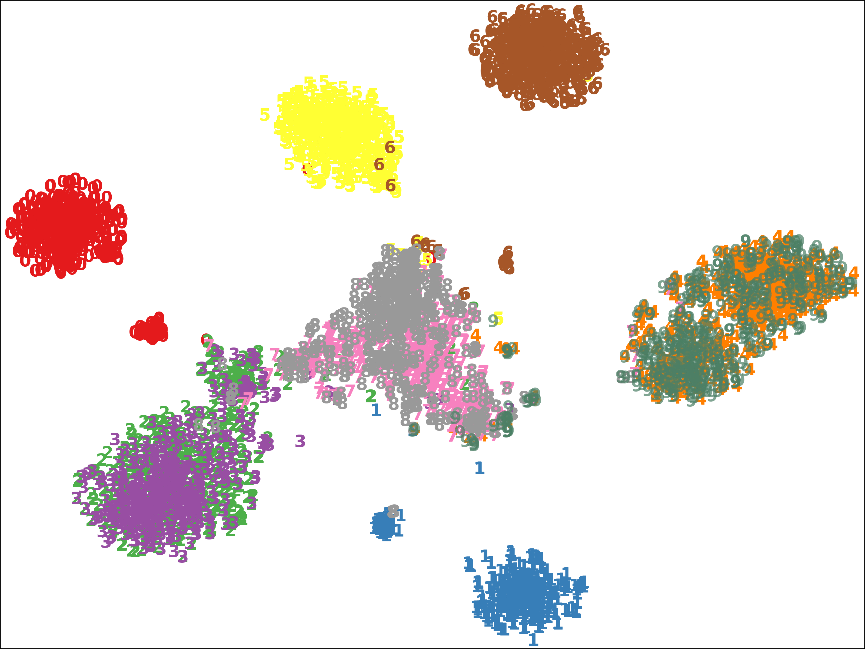}
\caption{$K=15$}
\end{subfigure}
\caption{Visualization of the representations learned with different prior knowledge of cluster numbers on four normal and four noise-simulated datasets. BDGP (a-c), DIGIT (d-f), COIL (g-i), Amazon (j-l), NoisyBDGP (m-o), NoisyDIGIT (p-r), NoisyCOIL (s-u), NoisyAmazon (v-x).}\label{tsne}
\end{figure*}

\fi

\end{document}